\documentclass{article}
\usepackage[utf8]{inputenc}

 \usepackage[preprint]{neurips_2026}

\usepackage[utf8]{inputenc} 
\usepackage[T1]{fontenc}    
\usepackage{hyperref}       
\usepackage{url}            
\usepackage{booktabs}       
\usepackage{amsfonts}       
\usepackage{nicefrac}       
\usepackage{microtype}      
\usepackage{xcolor}         
\usepackage{graphicx}
\usepackage{subcaption}
\usepackage{wrapfig}
\usepackage{lipsum}
\usepackage{amsthm}
\usepackage{amsmath}
\usepackage{enumitem}
\newtheorem{assumption}{Assumption}
\newtheorem{proposition}{Proposition}

\newtheorem{theorem}{Theorem}

\usepackage{tabularx}
\usepackage{makecell}
\usepackage{tocloft}
\usepackage{titletoc}
\usepackage{soul}
\usepackage{pifont}
\usepackage{multirow}
\usepackage{hyperref}
\usepackage{algorithm}
\usepackage{algpseudocode}
\usepackage{booktabs}
\usepackage{array}
\usepackage{longtable}
\usepackage{makecell}
\usepackage[most]{tcolorbox}

\newcommand{\yt}[2]{\href{https://www.youtube.com/channel/#1}{\texttt{#2}}}
\usepackage{todonotes}
\definecolor{ocr}{HTML}{00C8FF}
\definecolor{ocr}{HTML}{009900}
\definecolor{leeColor}{rgb}{0.6, 0.2, 1.0}
\definecolor{delvinColor}{rgb}{1.0, 0.6, 0.4}
\definecolor{nafisColor}{rgb}{0.5, 0.2, 0.8}
\definecolor{mahjabinColor}{rgb}{0.2, 0.6, 0.9}

\title{One Model Is Not a Crowd: 
Multi-LLM and Aspect-Conditioned Diverse Comment Generation}

\author{%
  Nafis Irtiza Tripto$^1$\thanks{Work completed while the author was a PhD student at Pennsylvania State University.} \quad
  Delvin Ce Zhang$^2$ \quad
  Mahjabin Nahar$^1$ \quad
  Dongwon Lee$^1$ \\
  $^1$College of Information Sciences and Technology, Pennsylvania State University, PA, USA \\
  $^2$University of Sheffield, Sheffield, UK \\
  \texttt{\{nit5154, mahjabin.n, dongwon\}@psu.edu}, \texttt{delvin.ce.zhang@sheffield.ac.uk}
}

\begin{document}

\maketitle

\begin{abstract}
Human communication on the internet is shaped by diverse perspectives, most visibly expressed in online comment spaces. As large language model (LLM)- based AI agents begin to inhabit these spaces, a key question arises: whether synthetic comment threads can capture the diversity inherent in human discourse. This concern is increasingly important, as the growing presence of homogenized AI-generated content risks reducing diversity over time, potentially leading to model collapse and degrading the richness of digital communication. Inspired by the plurality of human crowds and the aspect-driven nature of discourse, we hypothesize that comment diversity is better approximated by combining multiple LLMs with aspect-conditioned generation. We formalize and evaluate this approach using models from different providers and introduce a framework that characterizes diversity across semantic, linguistic, and socio-pragmatic features along three axes: dispersion, coverage, and alignment. Using this framework, we conduct a large-scale study on over 2 million YouTube comments across multiple domains. Our results reveal that multi-LLM and aspect-conditioned generation better align with human comment distributions and such data remains viable under pretraining-style curation and is effective for downstream  tasks. Yet, human diversity remains unmatched. Overall, our findings provide a practical foundation for generating more diverse and socially grounded discourse in AI-mediated environments.

\end{abstract}

\section{Introduction}
Human communication is fundamentally pluralistic. The diversity of opinions from individuals is a defining feature of human discourse \citep{mcluhan1994understanding}. Specially, with the rise of the 
internet as the new “digital 
campfire” \citep{arnd2015sherry}, this pluralism has become visible on a remarkable scale. It is more evident in online comment spaces, where people respond in ways that range from supportive to oppositional, humorous to aggressive, sarcastic to insightful, or entirely tangential. This diversity, integral to online communication, plays a crucial role in increasing engagement and credibility \citep{hong2018will}. 

Simultaneously, LLM-based AI agents 
are becoming participants in this online discourse. Multi-agent simulation frameworks \citep{yang2024oasis, zhang2025socioverse} and emerging AI-driven social platforms, such as Moltbook \citep{moltbook2026}, demonstrate that LLM-powered agents can simulate large-scale social interactions \citep{holtz2026anatomy}. Social media platforms, such as Meta, are actively exploring the integration of AI “users” into real platforms \citep{rollingstone_meta_ai_users_2023}. 
Thus, the possibility of AI-generated comments in online discourse  is not speculative; it is imminent.

\begin{figure}[h]
    \centering
    \includegraphics[width=0.93\linewidth]{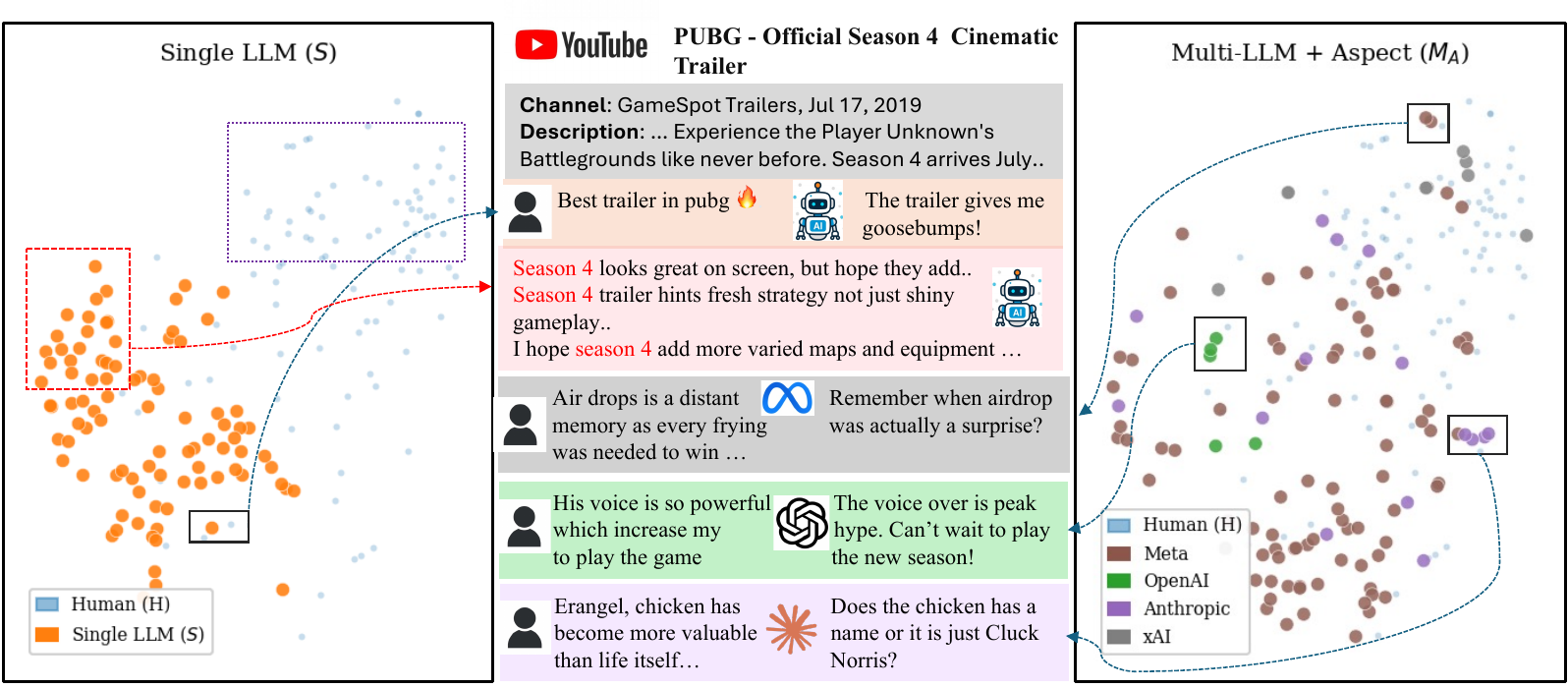}
    
    \caption{We visualize comments from a YouTube video (\textit{PUBG Season 4 Cinematic Trailer}) using 2D t-SNE \citep{van2008tsne} projections of text embeddings, comparing single-LLM generations (\textbf{S}) and multi-LLM aspect-conditioned generations ($M_A$) against human ($H$) comments (blue dots). While single-LLM outputs capture some common themes (e.g., trailer quality), they miss several semantic regions (\textcolor{violet}{violet rectangle}) and exhibit repetition (\textcolor{red}{red rectangle}). In contrast, ($M_A$) generation covers both dominant patterns (e.g., gameplay features, voice quality) and niche topics (e.g., “\textit{winner winner chicken dinner}” PUBG reference), yielding a broader and more structured distribution.
    }
    \label{fig:introduction_figure}
\end{figure}

However, this raises a fundamental question: {\bf Can LLM-generated comments truly capture the diversity that defines human discussion?} 
LLMs have a well-documented tendency to converge, producing averaged responses \citep{guo2025benchmarking, kirk2024understanding, omahony2024attributing}. While LLM diversity has been studied in persona generation \citep{ge2024scaling}, storytelling \citep{Chakrabarty2024art, yang-etal-2022-re3}, 
and dialogue settings \citep{jiangartificial, zhang2025noveltybench}, diversity in social media ecosystems, where multiple entities interact around shared content, as humans do, remains underexplored.

This problem is particularly urgent for two reasons. {\em First}, foundation models have already consumed much of the publicly available 
internet text \citep{abdalla2025future, liu2025datasets, qin2025scaling}, making social media one of the fastest-growing and  most consequential sources of new data \citep{longpre2025bridging, villalobos2024position}. If synthetic comments flood these platforms, they will inevitably circle back as future training signal, compounding any diversity failures over successive model generations \citep{gerstgrasser2024model, seddik2024how,  shumailov2024ai}. 
{\em Second}, as AI agents populate these platforms, the diversity (or lack thereof) in their outputs will shape future online discourse, with homogenized content risking bias amplification, suppression of minority perspectives,  reduced digital richness, and future implications for model scaling
\citep{anderson2024homogenization,fang2024bias,gerstgrasser2024model,kearney2025echoes}.

In particular, we hypothesize that {\em a single LLM is structurally ill-equipped to capture the full range of human pluralism}. This is motivated by two observations. First, human communication is inherently diverse, reflecting the plurality of individuals in a crowd \citep{littlejohn2010theories, wilson2002anthropology}. Second, humans rarely comment without intent, as comments are shaped by underlying aspects such as topic, stance, emotion, humor, or critique \citep{jill2025ways, kramer2021feel, matook2022user, northup2022personality}. 
To capture these effects, we propose two complementary mechanisms. {\em First}, we employ \textbf{model swarms}, ensembles of LLMs from different providers \citep{feng2025one, feng2025model}, each reflecting distinct training distributions, to approximate variability across individuals. {\em Second}, we introduce \textbf{aspect-conditioned generation}, which guides comment production along broad social dimensions rather than generating responses blindly. Figure~\ref{fig:introduction_figure} shows a representative example.

We organize our investigation around two core research questions:

 \textbf{RQ1:} {\em Can LLMs replicate human-level diversity in social media comment spaces?} \\
 \textbf{RQ2:} {\em Can synthetic comments serve as high-quality data for pre-training or fine-tuning LLMs?}

To answer these questions, we focus on YouTube as a content-centric platform where users primarily react to shared stimuli rather than pre-existing social relationships, providing a controlled setting for discourse analysis. We conduct a large-scale empirical study spanning three domains, 46 channels, nearly 7500 videos, 35 LLMs across nine providers, and over 2 million human comments. To move beyond narrow definitions of diversity, we introduce a \textbf{comprehensive evaluation framework} that measures diversity through \textbf{dispersion}, \textbf{coverage}, and \textbf{alignment} across \textit{semantic}, \textit{linguistic}, and \textit{socio-pragmatic} feature spaces. Our findings show that human diversity in YouTube comment discourse remains difficult to  fully replicate, even with model swarms, but combining multiple LLMs with aspect-conditioned generation substantially improves diversity. We further show that such synthetic data remains robust under pre-training curation and is effective for fine-tuning and instruction-tuning, particularly in low-resource settings. Overall, our work provides a unified framework for generating and evaluating socially grounded synthetic comment discourse and highlights the importance of preserving diversity as AI systems become increasingly integrated into online spaces.

\section{Related work}
We situate our work within three related areas: LLMs in social media environments, diversity in LLM generation, and 
their role in scaling language models.
\paragraph{LLM-based agents in social media environments.}
LLMs have demonstrated strong capabilities in persona simulation \citep{ge2024scaling, park2024generative} and role-playing \citep{chen2024persona, tseng2024two, zhou2025personaeval}, motivating their use as synthetic users in social media settings \citep{cau2025language, gao2025does, yang2024oasis, zhang2025socioverse}. Prior work has combined LLM agents with agent-based modeling to study online interactions \citep{cau2025language}, while scalable simulators such as OASIS \citep{yang2024oasis} and SocioVerse \citep{zhang2025socioverse} model dynamic environments. More recently, fully synthetic platforms like Moltbook demonstrate that LLM agents can emulate social interactions without human participation \citep{holtz2026anatomy}. However, whether such agents can replicate the behavioral and content diversity of human discourse remains largely unexplored.
\paragraph{Diversity issues in LLM.}
Diversity in LLM generation has become a central concern \citep{guo-etal-2024-curious, guo2025benchmarking, lu2024creativity_index_paper}, particularly due to model collapse \citep{gerstgrasser2024model, shumailov2024ai}, where outputs grow increasingly homogeneous. Prior work addresses this through training interventions \citep{chung2025modifying, ismayilzada-etal-2025-creative, zhou2024sotopia}, decoding strategies \citep{minh2025turning, lanchantin2025diversepreferenceoptimization}, and prompting methods \citep{Holtzman2020The, yang-etal-2022-re3, zhang2025verbalized}. However, most studies evaluate diversity in controlled settings such as persona generation \citep{ge2024scaling}, storytelling \citep{yang-etal-2022-re3,Chakrabarty2024art, zhang2025verbalized}, or multi-turn dialogue \citep{zhang2025noveltybench, jiangartificial}, focusing on variation across individual responses. In contrast, social media discourse is inherently collective, where diversity emerges across a set of responses. Moreover, existing evaluations \citep{jiangartificial, guo-etal-2024-curious, zhang2025noveltybench, zhang2025verbalized} often rely on limited metrics, primarily capturing semantic or lexical variation only.

\paragraph{The role of synthetic data in scaling language models.}
Recent work highlights the growing role of synthetic data in scaling language models \citep{kang-etal-2025-demystifying, qin2025scaling}, especially as high-quality human-authored data becomes saturated \citep{abdalla2025future, liu2025datasets}. While LLMs can generate large volumes of data \citep{11080380}, concerns such as the curse of recursion suggest potential degradation and model collapse if diversity and quality are not preserved \citep{shumailov2023curse, seddik2024how, chen2024diversity}. This has led to efforts in constructing high-quality synthetic datasets \citep{benallal2024cosmopedia, finephrase}. At the same time, social media is emerging as a major source of organic text \citep{longpre2025bridging, villalobos2024position}. However, it remains unclear whether LLM-generated comments, increasingly present in online ecosystems \citep{cau2025language, yang2024oasis, zhang2025socioverse}, can serve as effective data for scaling language models.

Therefore, our work addresses these gaps by providing a unified evaluation of synthetic social media comments,  analyzing their collective diversity and utility for scaling language models.

\section{The pluralism problem: why one model is not enough}
\label{sec_mathematical_proof}

Let $v$ denote a content item (e.g., a video) and  $\mathcal{C}$ the space of 
possible comments, with a feature mapping $\phi:\mathcal{C}\to\mathbb{R}^d$.
A comment ($C\in\mathcal{C}$) is modeled as a random variable  $C\sim P(\cdot\mid v)$.

\textbf{Human comments:}
Human-generated comments for video $v$ are drawn from an \emph{unknown} reference
distribution $P_H(\cdot\mid v)$.  Let $H = \{h_1,\dots,h_n\}$ be an observed sample.

\textbf{Single-LLM comments:}
A single language model $M$ induces $C\sim P_M(\cdot\mid v)$.

\textbf{Multi-LLM mixture:}
Given $K$ models $\{M_k\}_{k=1}^K$ with weights $\boldsymbol\pi = (\pi_1,\dots,\pi_K)$,
$\pi_k\ge 0$, $\sum_k \pi_k=1$, define the \emph{model mixture}:
$  P_{\mathrm{mix}}(\cdot\mid v)
    \;=\;
    \sum_{k=1}^K \pi_k\, P_{M_k}(\cdot\mid v).$

\textbf{Aspect conditioning:}
Let $\mathcal{A}=\{a_1,\dots,a_A\}$ be a finite set of discourse aspects for a content $v$.
Given aspect weights $\boldsymbol\lambda = (\lambda_1,\dots,\lambda_A)$,
$\lambda_a\ge 0$, $\sum_a \lambda_a=1$, an \emph{aspect-conditioned} model
$M_k$ generates from
$  P_{M_k,\mathcal{A}}(\cdot\mid v)
    \;=\;
    \sum_{a\in\mathcal{A}} \lambda_a\, P_{M_k}(\cdot\mid v,a). $
The full \emph{multi-LLM, multi-aspect} distribution combines both mechanisms, which is a mixture over $K\!\times\!A$ component distributions.
\begin{equation*}
    P_{\mathrm{multi}}(\cdot\mid v)
    \;=\;
    \sum_{k=1}^K \sum_{a\in\mathcal{A}} \omega_{k,a}\, P_{M_k}(\cdot\mid v,a),
    \qquad \omega_{k,a} = \pi_k\lambda_a,
    \label{eq:pmulti}
\end{equation*}

\textbf{Generalised kernel diversity:}
Let $\kappa:\mathcal{C}\times\mathcal{C}\to[0,1]$ be a symmetric, bounded
\emph{dissimilarity kernel}.  We define the \emph{expected diversity} of a
distribution $P$ as $  \delta_\kappa(P)
    \;=\;
    \mathbb{E}_{C,\,C'\,\overset{\mathrm{iid}}{\sim}\, P}\!\left[\kappa(C,C')\right].$
Given a finite sample $\{c_1,\dots,c_n\}\overset{\mathrm{iid}}{\sim} P$,
the estimator
$
    \widehat\delta_\kappa
    \;=\;
    \frac{1}{n(n-1)}\sum_{i\neq j}\kappa(c_i,c_j) $
is unbiased for $\delta_\kappa(P)$ by a standard U-statistic
argument~\citep{hoeffding1992class}.
While our study considers a multi-axis set of metrics for diversity, we illustrate our formulation using semantic dispersion.
We define it via clipped cosine dissimilarity between embeddings: $ \kappa_{\cos}(c,c')
    \;=\;
    1 - \max\!\left(0,\,\cos(\phi(c),\phi(c'))\right).$, following prior works \citep{cann2023using, guo2025benchmarking, lu2024creativity_index_paper, zhang2025verbalized}, and write $\delta(\cdot)\equiv\delta_{\kappa_{\cos}}(\cdot)$ for brevity.
The same formulation extends to other metrics. Diversity gains from model mixtures then depend on whether LLMs occupy distinct regions in $\phi(\mathcal{C})$,
which we formalize as a single, testable Assumption \ref{ass:separation} of \textbf{Model Separation}.

\begin{assumption}[]
\label{ass:separation}
There exist constants $\bar\rho\in[0,1)$ and $\epsilon>0$ such that for all
model pairs $k\neq\ell$:
\begin{enumerate}[label=(\roman*)]
    \item \emph{Low cross-model similarity:}
    $\;\mathbb{E}_{C\sim P_{M_k},\,C'\sim P_{M_\ell}}
        \!\left[\max(0,\cos(\phi(C),\phi(C')))\right] \;\le\; \bar\rho.$
    \item \emph{Higher within-model similarity:}
    $\;\mathbb{E}_{C,C'\sim P_{M_k}}
        \!\left[\max(0,\cos(\phi(C),\phi(C')))\right] \;\ge\; \bar\rho + \epsilon.$
\end{enumerate}
\end{assumption}

We support Assumption~\ref{ass:separation} through a proof-of-concept study (Appendix \ref{subsec_empirical_validation}), where multiple models generate responses to identical prompts across a subset of videos and generic queries, consistently showing lower cross-model similarity and higher within-model similarity. This observation is also aligned with prior work \citep{jiangartificial, west2025base, zhang2024forcing}, which motivates the use of heterogeneous model swarms \citep{feng2025one, feng2025model}.
Building on this assumption, we derive a mixture-based formulation of diversity
and prove two key results: Theorem~\ref{thm:multi_llm} shows that multi-LLM mixtures improve
expected diversity over the mixture-weighted average of their component models,
while Theorem~\ref{thm:full} extends this intuition to multi-LLM, aspect-conditioned mixtures
under additional aspect-separation assumptions. These results provide theoretical
motivation for using model heterogeneity and aspect conditioning as structural
sources of diversity, with full details in Appendix~\ref{appendix_sec_theory}.

\section{From human crowds to model swarms: dataset and  generation}
\label{sec_method}
In this section, we describe the construction of our dataset, generation pipeline, and feature extraction (Figure \ref{fig:methodology} highlights the overview of our methodology).

\begin{figure}[h]
    \centering
    \includegraphics[width=0.9\linewidth]{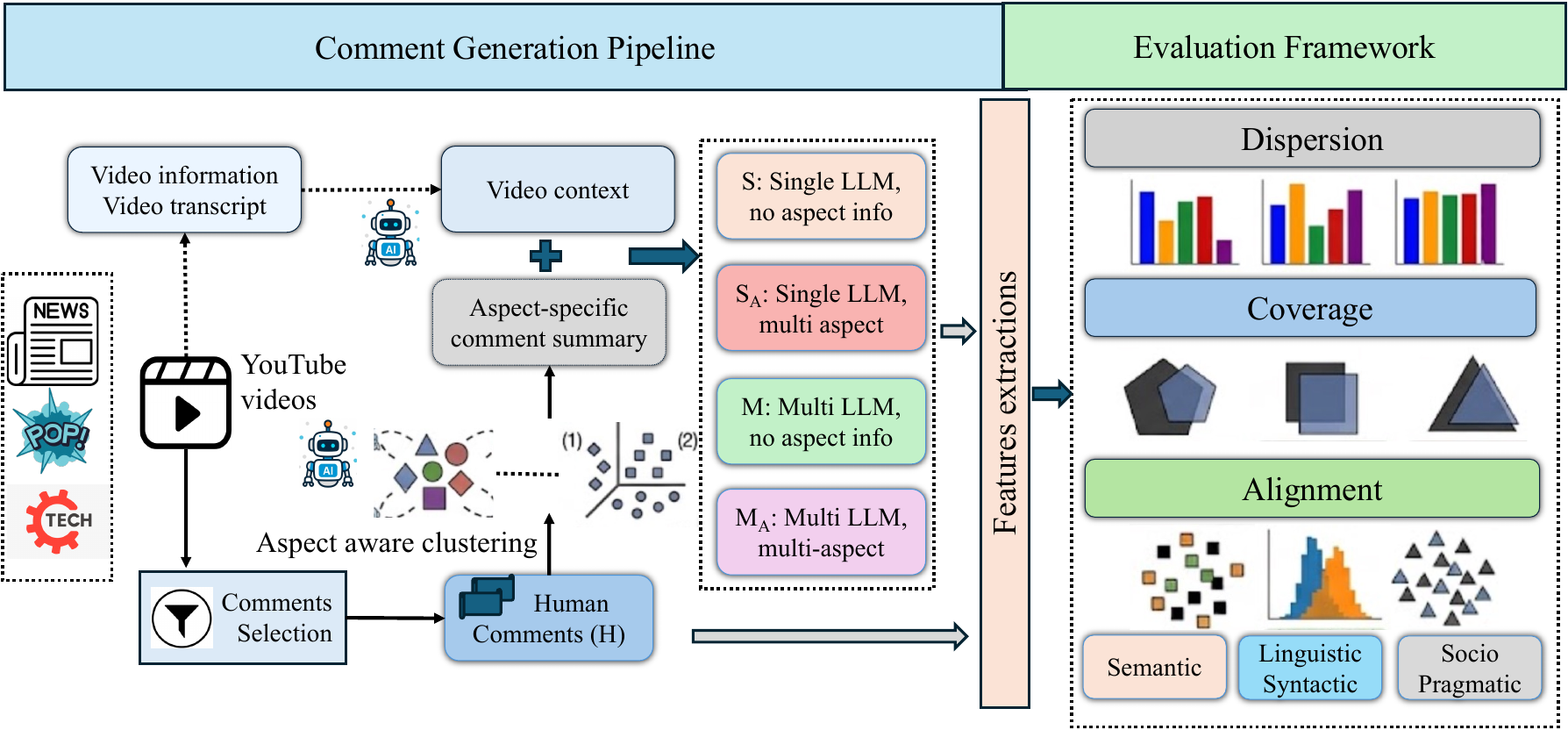}
    \caption{Overview of our framework. We collect YouTube videos across multiple domains, construct video context from metadata and transcripts, and derive aspect information from human comments via clustering. These are used to guide comment generation under different settings, followed by feature extraction and a comprehensive evaluation across \textit{dispersion}, \textit{coverage}, and \textit{alignment}.}
    \label{fig:methodology}
\end{figure}

\paragraph{Human comment corpus creation.}
Human comments are influenced by both content and social relationships \citep{choi2016social, garcia2017understanding, nasim2013commenting, stavros2014understanding}. Since modeling full social dynamics requires complex longitudinal simulation \citep{yang2024oasis, zhang2025socioverse}, we focus on \textbf{content-driven discourse}, where comments primarily respond to shared stimuli. This motivates our choice of platform. Unlike Reddit, Facebook, or Twitter, which are shaped by community norms \citep{ocal2021reasoning, thukral2018analyzing, yu2024characterizing}, social ties \citep{lapides2015news, maiz2016factors, nasim2013commenting, stavros2014understanding}, or network effects \citep{garcia2017understanding, khaund2021social, naaman2011hip} respectively, \textbf{YouTube} offers a more content-centric environment \citep{choi2016social, arthurs2018researching, schultes2013leave, ma2022m}. To further reduce confounding factors, we restrict our analysis to first-level comments, excluding replies. While this abstraction simplifies real-world interaction, it allows us to study a controlled distribution of content-driven comments $P_H(\cdot\mid v)$ that is suitable for comparison with LLM outputs.

\begin{wraptable}{r}{0.4\textwidth}
\caption{Human dataset summary.}
\label{tab:dataset_summary}
\centering
\footnotesize
\setlength{\tabcolsep}{5pt}
    \resizebox{0.4\columnwidth}{!}{
\begin{tabular}{lccccc}
\toprule
\textbf{Domain} & \textbf{\# Ch.} & \textbf{\# Vid.} & \textbf{\# Comm.} & \textbf{\# Words} & \textbf{Asp./Vid.} \\
\midrule
News         & 14 & 3,857 & 518.7K & 10.27M & $8.15 \pm 4.19$ \\
Pop & 22 & 2,203 & 326.5K & 4.56M  & $8.21 \pm 4.14$ \\
Tech         & 10 & 1,514 & 269.0K & 5.34M  & $9.30 \pm 4.38$ \\
\bottomrule
\end{tabular}
}
\end{wraptable}

We further select domains where content can be reasonably approximated through text, namely \textbf{news}, \textbf{pop culture}, and \textbf{tech}, ensuring accessibility to LLMs without relying on visual input. To avoid contamination from AI-generated content \citep{goldstein2023coming, heaton2025chatgpt, sun2025we} and major global disruptions \citep{li2022youtube}, we focus on a pre-LLM time frame \textbf{from mid-2018 to mid-2019}. 
Within this setting, we select English videos across multiple channels 
based on engagement, duration, and description availability. 
We collect up to 400 top-ranked comments per video,
to capture salient responses, 
which reflect the most visible \& high-engagement portion of user interaction, while acknowledging that this may not fully represent the entire spectrum of discourse. We then apply filtering steps to ensure quality and consistency, including removing duplicates, restricting comments to a fixed time window, and excluding excessively long, non-English,URL-only, spam, or toxic comments. Table \ref{tab:dataset_summary} summarizes the resulting dataset. Further details on domain \& timeline rationales and video/comment selection criteria are provided in the Appendix \ref{subsec_domain_rationale} \& \ref{subsec_human_dataset}.

\paragraph{LLM comment generation.}
We hypothesize that diversity arises from both heterogeneity across models and variation in discourse aspects.
 Let $\{M_k\}_{k=1}^K$ be the available LLMs, $\mathcal{A}$ the aspect set. To unravel their individual and combined effects, we study four primary generation settings.

\textbf{Setting \textbf{(S)} : single-LLM, no aspects.}
$M_k$ generates all comments
:
$    \mathcal{G}^{(\mathrm{S})}(\cdot \mid v)
    \;=\;
    P_{M_k}(\cdot \mid v).
    \label{eq:setting_S}$

\textbf{Setting (\textbf{S\textsubscript{A}}): single-LLM, with aspects:}
$M_k$ generates comments conditioned on each  aspect
$a \in \mathcal{A}$:
$   \mathcal{G}^{(\mathrm{S_A})}(\cdot \mid v)
    \;=\;
    \sum_{a \in \mathcal{A}} \lambda_a\, P_{M_k}(\cdot \mid v, a),
    \label{eq:setting_SA}$
where $\lambda_a \ge 0$ and $\sum_{a}\lambda_a = 1$ for aspect weights.

\textbf{Setting (\textbf{M}): multi-LLM, no aspects.}
Multiple models from $\{M_k\}_{k=1}^K$, each generate comments without aspect
conditioning:
$    \mathcal{G}^{(\mathrm{M})}(\cdot \mid v)
    \;=\;
    \sum_{k=1}^{K} \pi_k\, P_{M_k}(\cdot \mid v),
    \label{eq:setting_M}$, 
where $\pi_k \ge 0$ and $\sum_k \pi_k = 1$.

\textbf{Setting (\textbf{M\textsubscript{A}}): Multi-LLM, with aspects.}
Each model $M_k$ is assigned to \textbf{exactly one aspect} via a mapping
$\sigma: \{1,\dots,K\} \to \mathcal{A}$ (we ensure $K > A$) and generates comments on aspect $\sigma(k)$ only.
Let $\omega_k$  are the per-model weights ($\sum_{k=1}^K \omega_k = 1$ and $\omega_k \ge 0$) and aspect coverage induced by $\sigma$ is: $    \lambda_a
    \;=\;
    \sum_{k \in \mathcal{K}_a} \omega_k,
    \quad a \in \mathcal{A},$
so $\sum_a \lambda_a = 1$, our overall generation is:
$    \mathcal{G}^{(\mathrm{M_A})}(\cdot \mid v) =
    \sum_{a \in \mathcal{A}}\; \sum_{k \in \mathcal{K}_a}
    \omega_{k}\, P_{M_k}(\cdot \mid v, a)
    $. Algorithms (\ref{alg:setting_s_generation}-\ref{alg:setting_ma_generation}) provide details for each setting (Appendix \ref{subsec_llm_generation}).

We consider 35 LLMs from 9 providers, covering both proprietary and open-source models from OpenAI, Anthropic, Google, Meta, DeepSeek, Qwen, Mistral, xAI, and Microsoft (Table \ref{tab:model_list}).
For each video, we use up to $K \approx 15$ models in multi-LLM settings. To ensure fair comparison, we match the length distribution of generated comments to human comments \citep{hu2024explaining, munoz2024contrasting, tripto2025beyond}. In the single-LLM setting ($S$), we generate comments across predefined length bins using repeated runs, while in the multi-LLM setting ($M$), these bins are distributed across models. We condition the generation on video metadata and use  \textbf{temperature = 1} to encourage diversity. We also include an \textbf{ablation} ($S_D$) that varies decoding strategies to test whether diversity can be achieved from a single model.

For aspect-conditioned settings ($S_A$, $M_A$), we incorporate \textit{aspect} information to better reflect the structure of real-world discourse \citep{alafwan2023comments, alsayat2016social}. An aspect represents a coherent mode of discussion, capturing either topical focus or socio-pragmatic intent such as sentiment or humor \citep{sari2018topic, tai2020online, singh2021pragmatics, hasan2020socio}. For example, in a \textit{OnePlus 7 Pro review video}, one aspect may focus on technical comparisons, while another reflects short, community-specific reactions. We extract aspects from human comments using clustering: short comments are grouped by socio-pragmatic signals (\texttt{KModes} clustering \citep{chaturvedi2001k}), and longer comments by semantic similarity ((\texttt{BERTopic} \citep{grootendorst2022bertopic}). Each cluster is summarized into an aspect description that guides generation (Appendix \ref{appendix_aspect_generation} for aspect  and \ref{subsec_llm_generation} for LLM generation details).

Finally, to test generalization, we perform an \textbf{ablation} on newer videos where aspects are not derived from human comments but generated by a LLM based on video context (Appendix \ref{appendix_subsec_ablation}). This allows us to evaluate whether aspect-guided generation remains effective in more realistic, unseen settings.

\paragraph{Features extraction.}
We extract diverse  features from both human and LLM-generated comments to enable a comprehensive, multi-dimensional evaluation. We capture both \textbf{per-comment characteristics} and \textbf{aggregate video-level properties} computed over all comments for a given video and setting. 
First, we consider \textbf{semantic features} from dense text embeddings using \texttt{EmbeddingGemma } \citep{vera2025embeddinggemma} (a leading model in the \texttt{MMTEB} benchmark \citep{enevoldsen2025mmteb} at the time of our study) and perform KMeans clustering on video-level to represent distinct semantic themes within a video.
Second, we extract \textbf{linguistic} (lexical \& syntactic) features to characterize how comments are written. These include lexical diversity (type-token ratio \citep{herdan1960type}), readability metrics \citep{kincaid1975derivation}, sentence and word lengths, 
self-BLEU \citep{zhu2018texygen}, and POS-tag distributions using spaCy \citep{honnibal2020spacy} and TextDescriptive \citep{hansen2023textdescriptives}. 
Finally, we incorporate \textbf{socio-pragmatic} features, reflecting the interactional nature of comments. These include categorical attributes such as \textit{sentiment} \citep{singh2021efficient}, \textit{emotion} \citep{hartmann2022emotionenglish}, \textit{formality} \citep{babakov2023formality}, \textit{humor} \citep{baranov-etal-2023-humor}, \textit{sarcasm} \citep{romero2020t5sarcasm}, \textit{metaphor} \citep{wachowiak2022metaphor}, \textit{toxicity} \citep{vidgen2021ltoxicity}, and \textit{bias} \citep{fang2024bias}. Together, these features provide a multi-dimensional analysis of comment diversity across content, form, and intent (Appendix \ref{subsec_feature_extractions} for details).

\section{Can LLMs replicate human discourse diversity?}
\label{sec_rq1_result}

In this section, we address \textbf{RQ1} by evaluating whether LLM generation settings can replicate human discourse diversity. We present a comprehensive evaluation framework, summarize key findings across feature sets, and complement them with an overall quality assessment.


\begin{table}[h]
\caption{Primary evaluation metrics across dispersion, coverage, and alignment for each feature set. * denotes metrics bounded in [0--1]; $\uparrow$ higher is better, $\downarrow$ lower is better.}
\label{tab_evaluation_framework}
\centering
\small
\resizebox{\columnwidth}{!}{%
\begin{tabular}{@{}lllll@{}}
\toprule
\textbf{Axis} & \textbf{Definition} & \textbf{Semantic} & \textbf{Lexical} & \textbf{Pragmatic} \\ \midrule

\textbf{Dispersion} &
\begin{tabular}[c]{@{}l@{}}
Intra-set variation\\ within comments
\end{tabular}
&
\begin{tabular}[c]{@{}l@{}}
Clipped pairwise cosine\\ dissimilarity (*, $\uparrow$) \\ \citep{cann2023using, guo2025benchmarking, zhang2025verbalized}\\
$\mathcal{D}_{\text{sem}} = \mathbb{E}_{i,j}[1 - \widetilde{\cos}(c_i,c_j)]$
\end{tabular}
&
\begin{tabular}[c]{@{}l@{}}
N-gram diversity ($\uparrow$)\citep{padmakumar2024does}\\ 
Self-repetition score ($\downarrow$)\citep{salkar2022self_repitiion_score}\\ 
POS Compression ratio ($\downarrow$)\citep{shaib-etal-2025-standardizing}\\ 
Homogenization Score ($\uparrow$)\citep{lin2004rouge}
\end{tabular}
&
\begin{tabular}[c]{@{}l@{}}
Simpson diversity \\ index (*, $\uparrow$) \citep{simpson1949measurement}\\
$\mathcal{D}_{\text{prag}} = 1 - \sum_k \bar{P}_k^2$\\
$\bar{P}_k = \frac{1}{N}\sum_j p_{j,k}$
\end{tabular}
\\ \midrule

\textbf{Coverage} &
\begin{tabular}[c]{@{}l@{}}
Fraction of human\\ patterns recovered
\end{tabular}
&
\begin{tabular}[c]{@{}l@{}}
Weighted manifold \\ recall (*, $\uparrow$) \citep{kynkaanniemi2019manifold_precision}\\
$\mathcal{C}_{\text{sem}}= \frac{\sum_{k \in \mathcal{K}(H)} \mathbf{1}[\mathcal{K}_k \cap M \neq \emptyset] \cdot w_k}{\sum_{k \in \mathcal{K}(H)} w_k}$
\end{tabular}
&
\begin{tabular}[c]{@{}l@{}}
Weighted n-gram \& POS\\
coverage (*, $\uparrow$)\\
$\mathcal{C}_{\text{ling}} = \frac{\sum_{g \in \mathcal{P}(H)\cap \mathcal{P}(M)} f_H(g)}{\sum_{g \in \mathcal{P}(H)} f_H(g)}$
\end{tabular}
&
\begin{tabular}[c]{@{}l@{}}
Categorical coverage \\ ratio (*, $\uparrow$)\\
$\mathcal{C}_{\text{prag}} = \frac{|\mathcal{K}(H) \cap \mathcal{K}(M)|}{|\mathcal{K}(H)|}$
\end{tabular}
\\ \midrule

\textbf{Alignment} &
\begin{tabular}[c]{@{}l@{}}
Distribution similarity\\ between $H$ and $M$
\end{tabular}
&
\begin{tabular}[c]{@{}l@{}}
Maximum Mean \\ Discrepancy (RBF kernel) ($\downarrow$)
\end{tabular}
&
\begin{tabular}[c]{@{}l@{}}
Avg. JSD score from ($\downarrow$)\\
$\mathrm{JSD}_{\text{POS}}, \mathrm{JSD}_{\text{ngram}}, \mathrm{JSD}_{\text{length}}$
\end{tabular}
&
\begin{tabular}[c]{@{}l@{}}
Avg. category JSD ($\downarrow$)\\
$\mathcal{A}_{\text{prag}} = \frac{1}{|\mathcal{K}|} \sum_{k \in \mathcal{K}} \mathrm{JSD}_k$
\end{tabular}
\\ \bottomrule
\end{tabular}%
}
\end{table}

\paragraph{Evaluation framework.}
Inspired by prior work on synthetic data evaluation \citep{zamzmi2025scorecard_llm}, we adopt a multi-axis framework that captures three complementary dimensions: \textbf{dispersion}, measuring variation within generated comments; \textbf{coverage}, assessing how well generated comments capture the range of themes present in human discourse; and \textbf{alignment}, evaluating how closely the overall distribution matches human comments.
We compute these axes across three feature spaces to provide a more fine-grained analysis. While \textbf{semantic} features capture topical and discourse-level patterns, \textbf{linguistic} features reflect surface-level text properties, and \textbf{socio-pragmatic} features serve as proxies for human behavior in online interactions.
Table \ref{tab_evaluation_framework} summarizes the primary metrics used for each axis and feature space, with detailed formulations and rationales provided in Appendix \ref{app_evaluation_metric}.

\paragraph{Diversity across dispersion, coverage, and alignment.}

\begin{figure}[h]
    \centering
     \begin{subfigure}{0.3\textwidth}
        \centering
        \includegraphics[width=\linewidth]{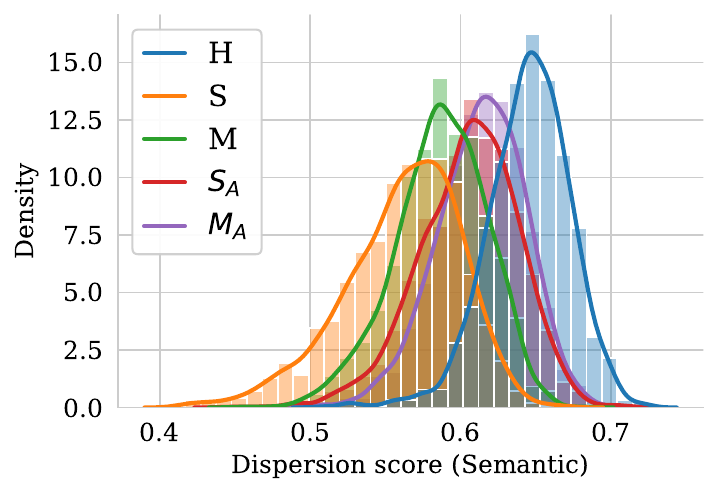}
    \end{subfigure}%
    \begin{subfigure}{0.3\textwidth}
        \centering
        \includegraphics[width=\linewidth]{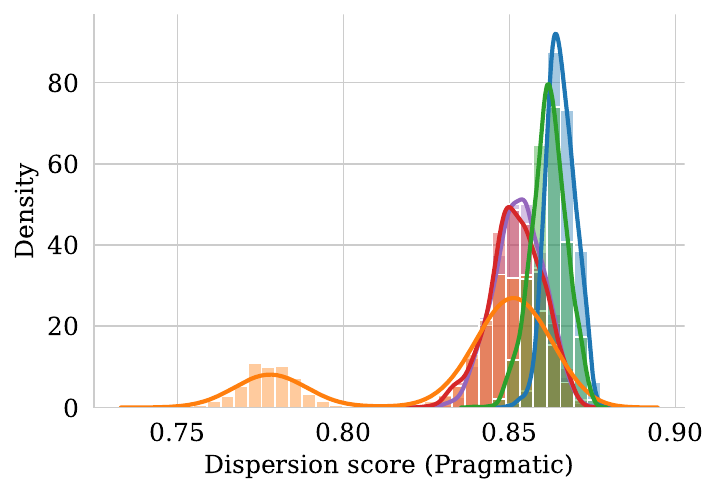}
    \end{subfigure}%
    \begin{subfigure}{0.3\textwidth}
        \centering
        \includegraphics[width=\linewidth]{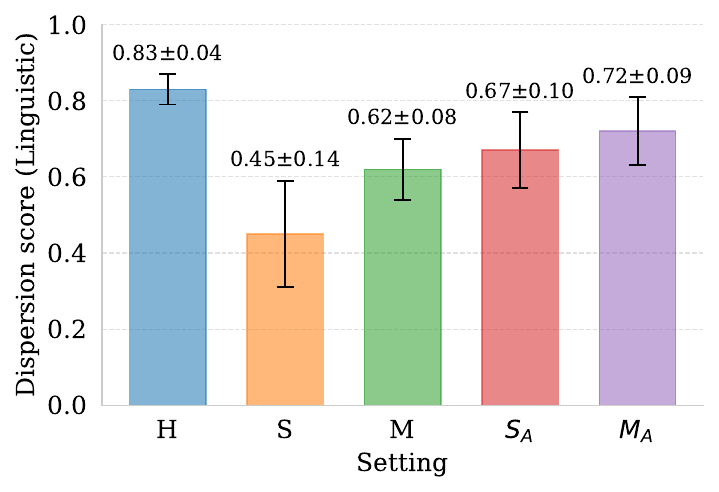}
    \end{subfigure}
    \caption{Dispersion across features. \textbf{Left}: semantic and \textbf{middle}: pragmatic dispersion distribution across videos respectively. \textbf{Right}: normalized linguistic dispersion aggregated across metrics.}
    \label{fig_main_findings_1}
\end{figure}
As shown in Figure~\ref{fig_main_findings_1}, human comment sets consistently exhibit the highest dispersion across all feature spaces. Among LLM settings, \textbf{multi-LLM with aspect conditioning ($M_A$)} performs best overall, while \textbf{single-LLM ($S$)} shows the lowest performance. For semantic and linguistic features, we observe a consistent ordering $H > M_A \approx S_A > M > S_D > S$, indicating that aspect conditioning improves topical and lexical variation, though gains over $S_A$ remain modest. In contrast, for socio-pragmatic features, the ordering shifts to $H > M > M_A > S_A > S$, suggesting that combining multiple models increases behavioral spread, while aspect conditioning produces more grounded  expressions. The ordering  is based on paired per-video tests at \(p<0.05\), with full results reported in Appendix~\ref{appendix_sec_detailed_rq1}.
These results highlight that diversity is not uniform across feature spaces and depends on how generation is structured. 

We next examine how well-generated comments capture and align with human discourse (Figure \ref{fig_main_findings_2}). Coverage trends largely mirror dispersion for semantic features, with $M_A$ and $S_A$ achieving the highest scores. However, in linguistic coverage, differences across settings are not statistically significant, suggesting that LLMs already capture dominant lexical and syntactic patterns. In the pragmatic dimension, $M$ achieves higher dispersion but weaker alignment, while $M_A$ shows improved alignment despite slightly lower dispersion. This reflects a key trade-off: combining multiple models broadens outputs due to differing biases, whereas aspect conditioning targets pragmatic categories observed in human comments, producing distributions more aligned with human proportions. Overall, model heterogeneity and guided generation improve coverage and alignment, but human discourse remains more nuanced and difficult to replicate fully.

\begin{figure}[h]
    \centering
     \begin{subfigure}{0.3\textwidth}
        \centering
        \includegraphics[width=\linewidth]{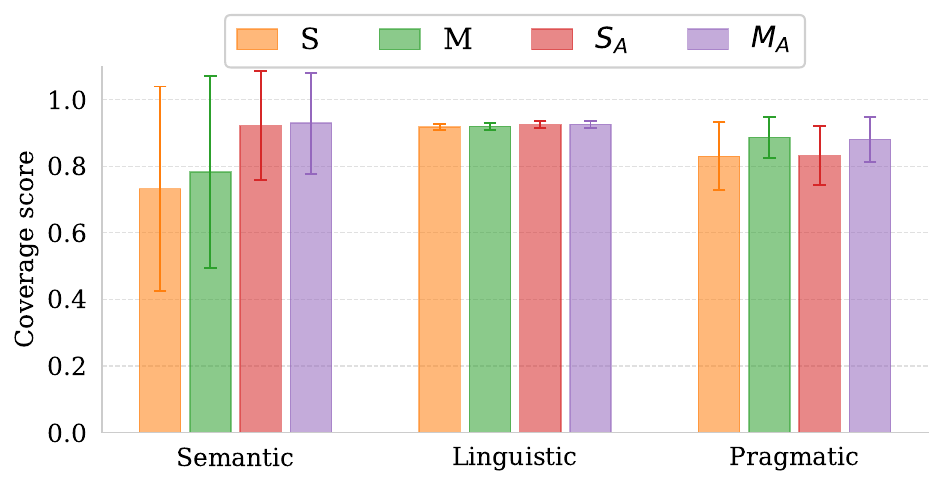}
    \end{subfigure}%
    \begin{subfigure}{0.3\textwidth}
        \centering
        \includegraphics[width=\linewidth]{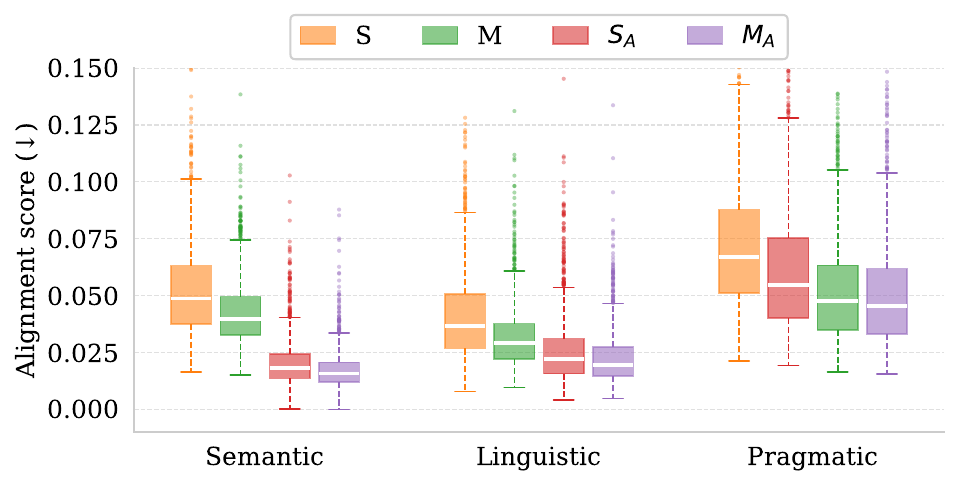}
    \end{subfigure}%
     \begin{subfigure}{0.22\textwidth}
        \centering
        \includegraphics[width=\linewidth]{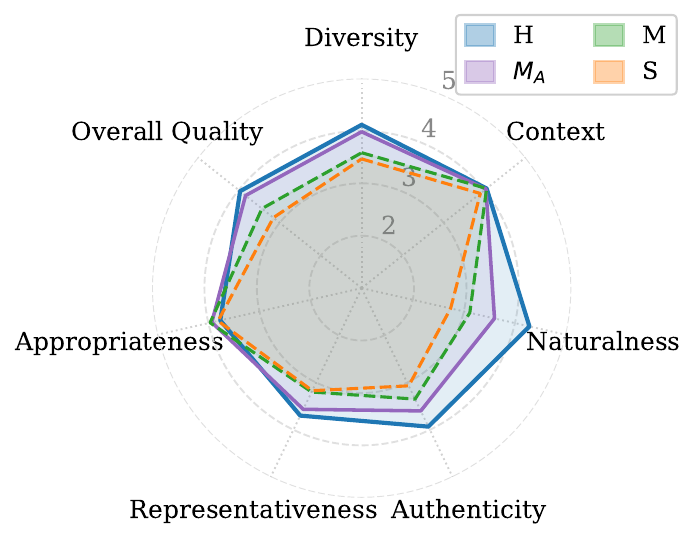}
    \end{subfigure}
    \caption{Coverage, alignment, and perceived quality across settings. \textbf{Left:} mean coverage scores with std.
\textbf{Middle:} alignment scores (lower is better).
\textbf{Right:} Overall comment quality and realism.}
    \label{fig_main_findings_2}
\end{figure}



\paragraph{Overall comment quality and realism.}
Beyond metric-based evaluation, we assess how comments are perceived as a \textbf{coherent comment thread}, reflecting real user experience. Following prior work \citep{kumpel2015user}, we present annotators with a small set of representative comments (10 per video) from each setting, corresponding to what users typically see on first view. We conduct a focused human study on 30 videos, where annotators rate each set on a Likert scale (1–5) across diversity, context, naturalness, authenticity, representativeness, appropriateness, and overall quality.
Results (Figure \ref{fig_main_findings_2} - right ) show that \textbf{human comments remain strongest in naturalness and authenticity}, highlighting the difficulty of replicating human expression. Multi-LLM with aspect conditioning ($M_A$) closely matches human performance on diversity, context, and appropriateness, while single-LLM ($S$) consistently underperforms. To extend this analysis at scale, we use an LLM-as-judge framework over the full dataset with multiple models and observe consistent trends (Appendix \ref{appendix_subsec_overall_comment_quality}). 

Overall, our holistic evaluation shows that while structured generation strategies improve perceived quality, fully matching human diversity and realism remains an open challenge.

\section{Synthetic data utility}
Data quality is critical for language model training, where both filtering \citep{longpre2024pretrainer, li2024datacomp, wang2025ultra} and diversity \citep{anonymous2024beyond, tirumala2023d4} play key roles. Thus, we address \textbf{RQ2} by evaluating whether LLM-generated comments can serve as useful training data. Since full-scale pre-training is infeasible at our scale \citep{hoffmann2022training, li2024datacomp}, we use standard data curation pipelines as a proxy to assess their retention relative to human data, and further evaluate their effectiveness in fine-tuning and instruction-tuning settings.


\paragraph{Pre-training data curation and quality analysis}
To assess suitability for pre-training, we adopt a curation pipeline inspired by \texttt{IBM GneissWeb} \citep{gohari2025gneissweb} and \texttt{NVIDIA NeMo Curator} \citep{su2025nemotron}. Since individual comments are short, we first aggregate them into 256–512 token documents. We then apply standard filtering steps, including deduplication \citep{lee-etal-2022-deduplicating}, readability and tokenization checks, perplexity filtering \citep{anknerperplexed}, and semantic deduplication \citep{abbas2023semdedup}, following prior thresholds \citep{gohari2025gneissweb, anknerperplexed, su2025nemotron}. Finally, we evaluate \textbf{document quality} and \textbf{educational value} using pretrained classifiers \citep{penedo2024fineweb, su2025nemotron}, providing a practical proxy of real-world data curation (Appendix \ref{appendix_pretraining_data}).


\begin{figure}[h]
    \centering
     \begin{subfigure}{0.3\textwidth}
        \centering
        \includegraphics[width=\linewidth]{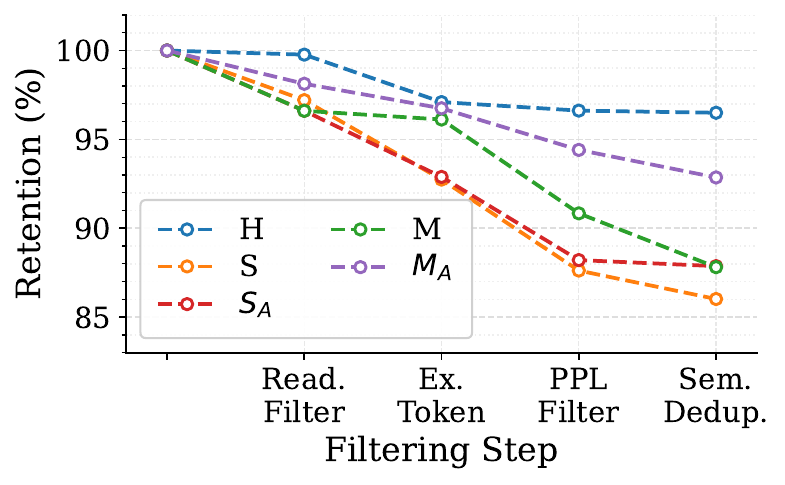}
    \end{subfigure}%
    \begin{subfigure}{0.25\textwidth}
        \centering
        \includegraphics[width=\linewidth]{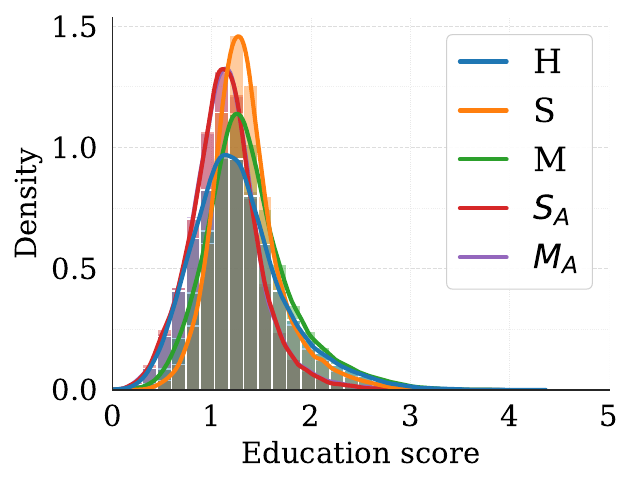}
    \end{subfigure}%
    \begin{subfigure}{0.25\textwidth}
        \centering
        \includegraphics[width=\linewidth]{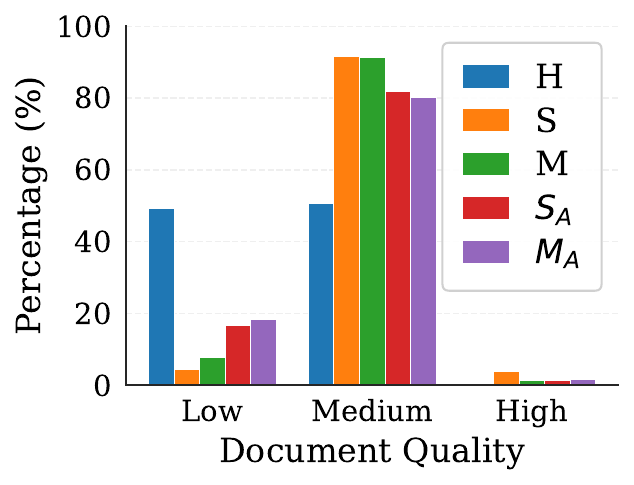}
    \end{subfigure}
    \caption{Pre-training data curation and quality. \textbf{Left:} \% of data retained after each step. \textbf{Middle:} distribution of educational value and \textbf{right:} \% of document quality of retained samples respectively.}
    \label{fig_rq2_pretraining_results}
\end{figure}

Figure \ref{fig_rq2_pretraining_results} summarizes the pre-training curation results. Human (H) and multi-LLM multi-aspect ($M_A$) settings retain the most data, indicating that higher diversity improves retention. Educational value distributions are similar across all settings, suggesting social media content generally lies at the lower end and synthetic data does not degrade this signal. For document quality, LLM settings ($S$ and $M$) concentrate in the medium tier, while human data shows a broader spread. Overall, diversity aids retention, but standard filters favor structured, knowledge-rich text, underrepresenting social discourse, while LLM outputs introduce more such patterns.


\begin{figure}[ht]
  \centering
  \begin{minipage}{0.45\textwidth}
    \centering
    \footnotesize
     \resizebox{\columnwidth}{!}{
    \begin{tabular}{@{}lccccc@{}}
\toprule
\multicolumn{1}{c}{\textbf{Setting}} & \textbf{Macro F1} & \textbf{Fear} & \textbf{Anger} & \textbf{Disgust} & \textbf{Surprise} \\ \midrule
$H$                          & 0.64              & 0.46          & 0.55           & \textbf{0.55}    & \textbf{0.71}     \\
$S$                          & 0.63              & 0.55          & 0.49           & 0.52             & 0.66              \\
$M$                        & 0.62              & 0.51          & 0.51           & 0.47             & 0.69              \\
\textbf{$S_A$}                        & 0.64              & 0.54          & 0.54           & 0.53             & 0.70              \\
\textbf{$M_A$}                        & 0.64              & 0.56          & \textbf{0.57}  & 0.53             & 0.69              \\
\textbf{$H_{aug}$}                      & \textbf{0.66}     & \textbf{0.61} & 0.55           & \textbf{0.55}    & 0.69              \\ \bottomrule
\end{tabular}
}
   
  \end{minipage}
  \hfill 
  \begin{minipage}{0.22\textwidth}
    \centering
    \includegraphics[width=\textwidth]{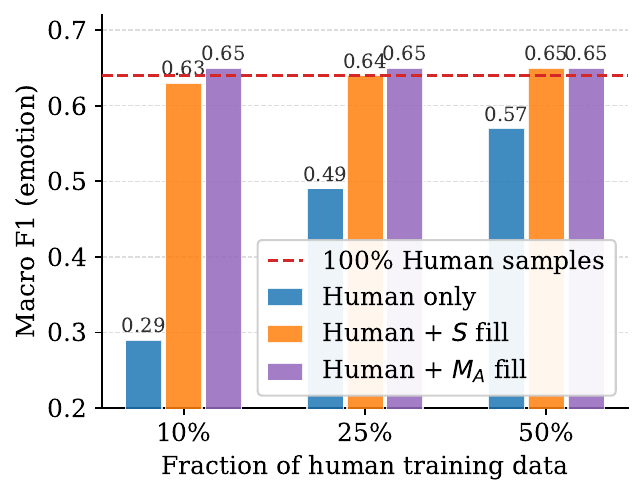}
  \end{minipage}
   \hfill 
  \begin{minipage}{0.22\textwidth}
    \centering
    \includegraphics[width=\textwidth]{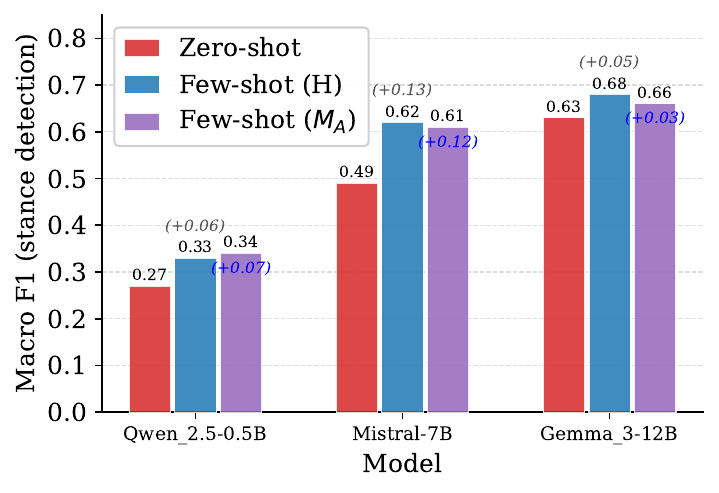}
  \end{minipage}
  \caption{Fine-tuning and instruction-tuning performance. \textbf{Left:} F1 score comparisons for  different training portions. \textbf{Middle:} impact of augmenting human data with LLM-generated samples. \textbf{Right:} instruction-tuning performance across zero-shot, few-shot (human), and few-shot ($M_A$) settings.}
  \label{fig_finetuning_instruction_tuning}
\end{figure}

\paragraph{Fine-tuning on synthetic comment data}
We next evaluate whether LLM-generated comments can serve as effective \textbf{fine-tuning data} using multi-class emotion classification (Ekman’s \citep{ekman2014expression} six emotions plus neutral), a task well aligned with social media expression. We train a \texttt{BERT-base} classifier \citep{sun2019fine} on human ($H$) and LLM-generated data ($S$, $M$, $S_A$, $M_A$), and \textbf{evaluate on a fixed human test set}. Results (Figure \ref{fig_finetuning_instruction_tuning}) show that synthetic data achieves comparable performance to human-trained models, often improving minority classes. The hybrid setting ($H_{aug}$) performs best, highlighting the value of LLM data for augmentation. Importantly, under data scarcity, combining limited human data with LLM-generated samples yields performance close to training on full human data (see Appendix \ref{appendix_finetuning_data} for data preparation, model training \& detailed results).

\paragraph{Instruction tuning}
We further evaluate LLM-generated comments for instruction tuning using a \textbf{stance detection task}, where a comment is classified as \textit{conservative}, \textit{liberal}, or \textit{neutral} given video context. We use comments from three news channels and obtain labels via multiple LLM annotators with majority voting. We compare zero-shot performance with few-shot prompting using human or LLM-generated examples ($M_A$) across instruction-tuned models of different sizes. Results (Figure \ref{fig_finetuning_instruction_tuning}) show that few-shot prompting with LLM-generated comments consistently improves over zero-shot and achieves performance comparable to human examples (Appendix \ref{appendix_instruction_data} for details). Overall, these results show that synthetic social media data is a practical and scalable resource for pre-training curation, fine-tuning, and instruction tuning, especially under limited or imbalanced human data.

\section{Discussion and conclusion}
\paragraph{Model swarms vs. guided generation.}
A central question in this work is whether diversity in social discourse is better approximated through \textbf{model heterogeneity (model swarms)} or \textbf{guided generation (aspect conditioning)}. Our results show that both mechanisms contribute to improved diversity, and their combination consistently performs best, though still falling short of human-level diversity. Importantly, the two approaches offer complementary strengths: \textbf{aspect conditioning improves semantic and linguistic diversity} by capturing dominant discourse themes, while \textbf{multi-model generation better captures socio-pragmatic} variation and behavioral nuances (Figure \ref{fig_discussion_1}). 


\paragraph{ Can one model do it all? }
We examine whether diversity can be achieved by varying decoding strategies (e.g., temperature, sampling) within a single model (setting $S_D$). However, we observe only marginal and statistically insignificant gains over the baseline single-LLM setting ($S$), with performance remaining well below multi-LLM and aspect-conditioned approaches. This suggests that decoding-level randomness alone is insufficient to recover the diversity seen in human discourse.

\begin{figure}[h]
    \centering
     \begin{subfigure}{0.35\textwidth}
        \centering
        \includegraphics[width=\linewidth]{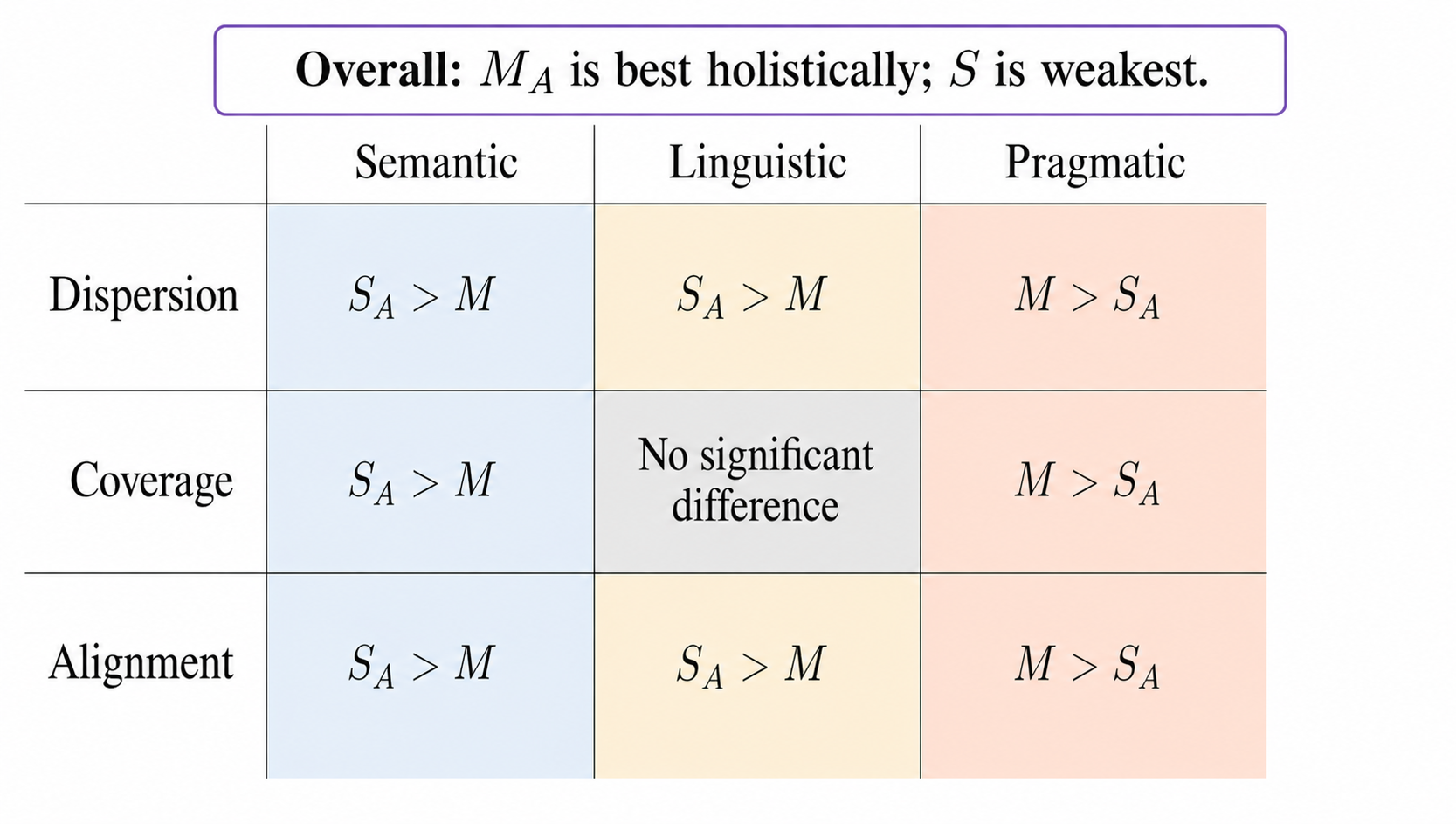}
    \end{subfigure}%
     \begin{subfigure}{0.28\textwidth}
        \centering
        \includegraphics[width=\linewidth]{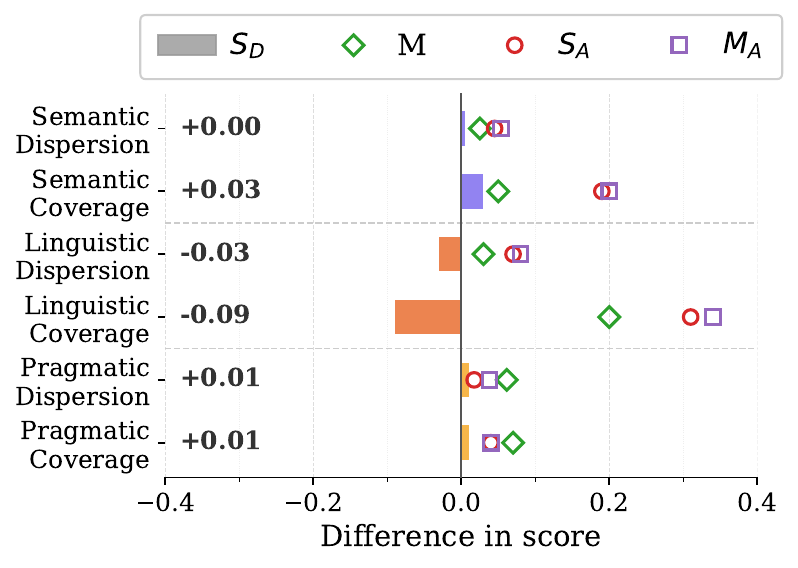}
    \end{subfigure}%
    \begin{subfigure}{0.28\textwidth}
        \centering
        \includegraphics[width=\linewidth]{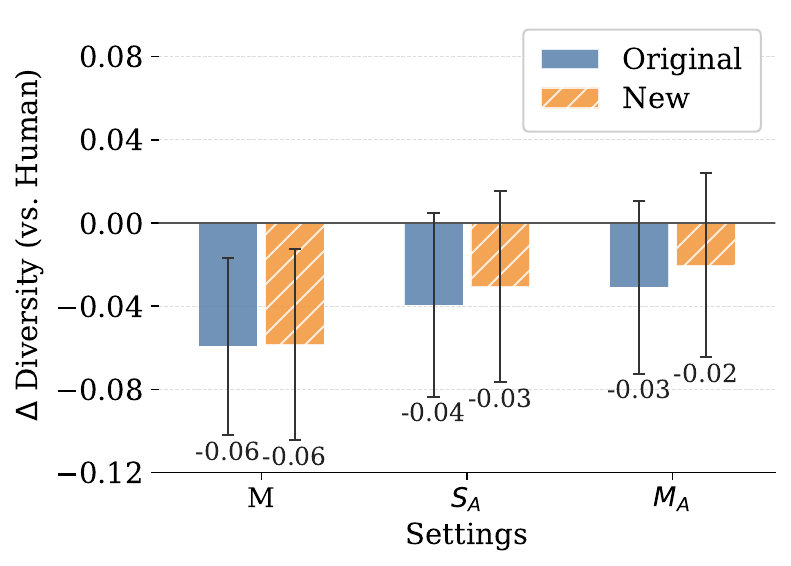}
    \end{subfigure}
    \caption{
\textbf{Left:} holistic comparison of aspect vs model swarms across different evaluations.
\textbf{Middle:} impact of decoding variation ($S_D$) relative to $S$ and other settings.
\textbf{Right:} performance differences on unseen content compared to original data across settings (normalized avg. evaluation score).
    }
    \label{fig_discussion_1}
\end{figure}

\paragraph{Generalization to unseen content}
We test generalization by moving to newer videos where no human comments are available and using an LLM as a planner to infer aspects directly from video context. This reflects a realistic deployment setting. Results (Figure \ref{fig_discussion_1}-right) show that while human comments remain more diverse, the gap between human and LLM is similar to the original setup, indicating stable performance. This suggests that LLMs can approximate discourse structure without human data, supporting scalability, though a more systematic evaluation is needed in future work.

\paragraph{Effect of comment set size on diversity}
We examine how diversity scales with the number of comments by progressively increasing the set size per video. As shown in Figure \ref{fig_discussion_2} (left), diversity decreases across all settings, including human comments, indicating a saturation effect. LLM-based settings (especially $M$, $S_A$, and $M_A$) achieve competitive or higher diversity for small sets but degrade more rapidly as size increases. This suggests that while LLMs can produce diverse samples in limited contexts, they struggle to sustain diversity at scale. In contrast, human discourse maintains diversity over larger collections, highlighting the challenge of replicating human-level pluralism.


\begin{figure}[h]
    \centering
     \begin{subfigure}{0.3\textwidth}
        \centering
        \includegraphics[width=\linewidth]{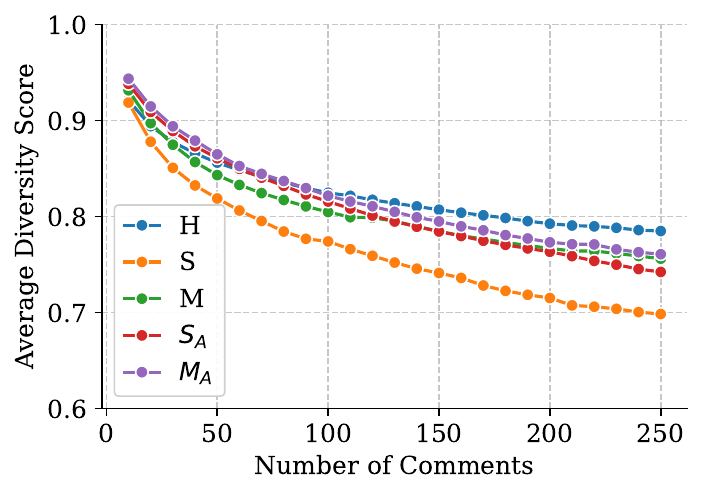}
    \end{subfigure}%
    \begin{subfigure}{0.3\textwidth}
        \centering
        \includegraphics[width=\linewidth]{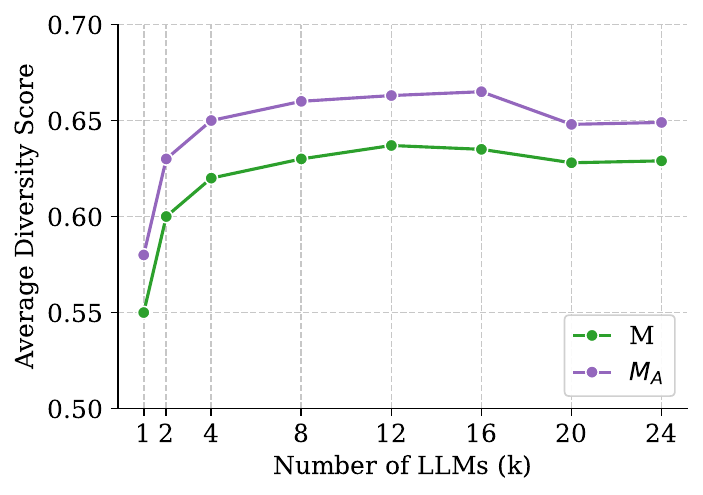}
    \end{subfigure}%
    \begin{subfigure}{0.25\textwidth}
        \centering
        \includegraphics[width=\linewidth]{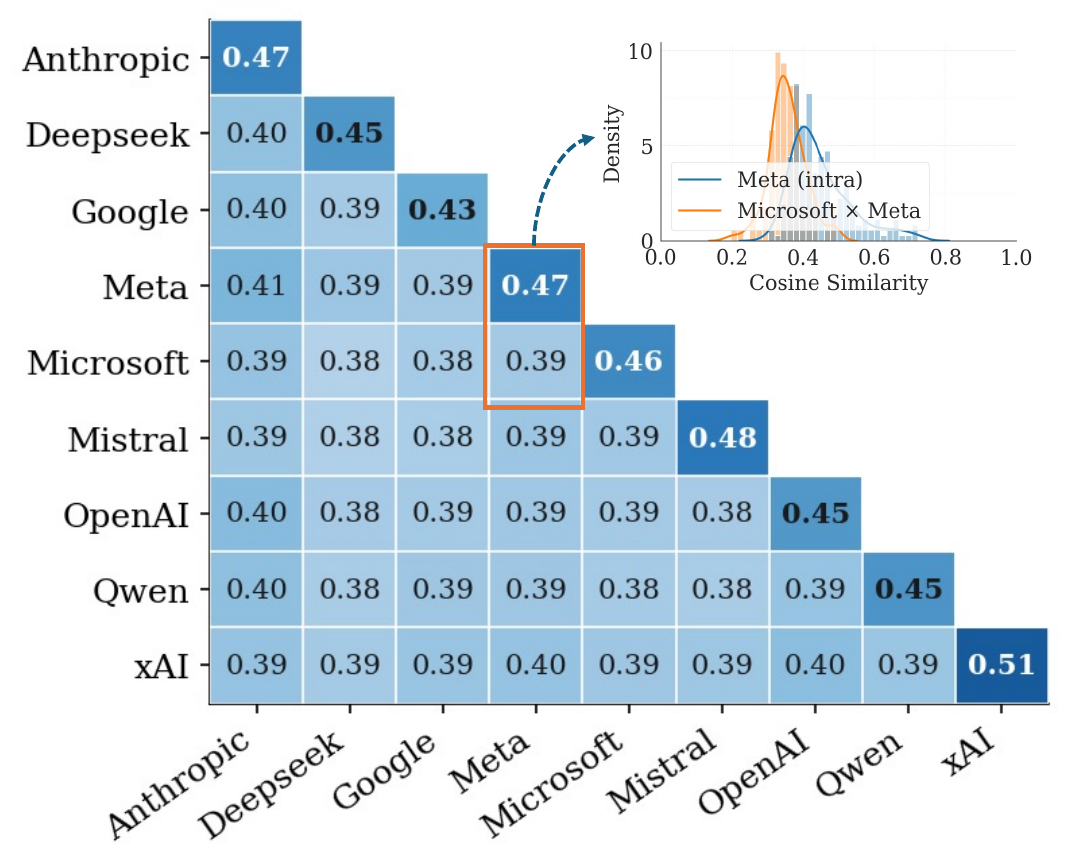}
    \end{subfigure}
    \caption{
    \textbf{Left:} diversity trends (avg. of the metrics) as the number of comments increases.
\textbf{Middle:} diversity trends with increasing number of LLMs per video ($M$, $M_A$).
\textbf{Right:} cross-provider similarity across models, with inset showing sample pairwise distribution (Meta–Microsoft).
    }
    \label{fig_discussion_2}
\end{figure}

\paragraph{Scaling model diversity: benefits and limits of model swarms.}

Finally, we analyze how diversity evolves as more models are added and how similarity varies across providers (Figure \ref{fig_discussion_2}). Results show a sharp initial increase in diversity with additional models, followed by diminishing returns, indicating that simply scaling model count is not sufficient beyond a point. Complementary analysis of inter and intra-provider similarity reveals that models from the same provider are consistently more similar than those across providers, suggesting that diversity gains derive from heterogeneity rather than quantity. 
Together, these findings suggest that diversity in LLM-generated discourse depends less on simply adding more models and more on combining models that bring genuinely different generative behaviors, offering practical guidance for designing model-swarm generation systems.



In conclusion, we study diversity in LLM-generated social media discourse, an increasingly important problem as synthetic content becomes part of the internet ecosystem. We introduce a framework that models diversity across dispersion, coverage, and alignment over semantic, linguistic, and socio-pragmatic spaces, and formalize generation strategies based on model heterogeneity and aspect conditioning. Our large-scale analysis shows that these approaches improve diversity, while still leaving measurable gaps compared to human comment distributions.
We further demonstrate that synthetic comments are useful for pretraining data suitability,  and downstream tasks, especially in low-resource settings. Overall, our findings show that diversity is driven by structural design rather than sampling alone, providing guidance for generating more realistic and diverse online discourse.


\section*{Acknowledgment}
This work was supported in part by U.S. NSF
awards no. 2438810 and 2555559, 2026 Amazon
Nova AI challenge award, and PSU Frymoyer chairship. Some experimental results were obtained
using computational resources provided by CloudBank, supported through U.S. NAIRR award no.
240336. In addition, we acknowledge the support
from the Linguistic Diversity Across the Lifespan Graduate Research Traineeship Program (NSF
grant no. 2125865).

\bibliographystyle{plain}
\bibliography{main}

@book{mcluhan1994understanding,
  title={Understanding media: The extensions of man},
  author={McLuhan, Marshall},
  year={1994},
  publisher={MIT press}
}

@misc{arnd2015sherry,
  title={Sherry Turkle: Alone Together: Why We Expect More from Technology and Less from Each Other: Basic Books, New York, 2011, 348 pp, ISBN 978-0465031467 (pbk)},
  author={Arnd-Caddigan, Margaret},
  year={2015},
  publisher={Springer}
}

@article{hong2018will,
  title={Will comments change your opinion? The persuasion effects of online comments and heuristic cues in crisis communication},
  author={Hong, Seoyeon and Cameron, Glen T},
  journal={Journal of Contingencies and Crisis Management},
  volume={26},
  number={1},
  pages={173--182},
  year={2018},
  publisher={Wiley Online Library}
}

@article{gao2025does,
  title={Does social bot help socialize? Evidence from a microblogging platform},
  author={Gao, Yang and Zhang, Maggie Mengqing and Lysyakov, Mikhail},
  journal={Information Systems Research},
  year={2025},
  publisher={INFORMS}
}

@article{wilson2002anthropology,
  title={The anthropology of online communities},
  author={Wilson, Samuel M and Peterson, Leighton C},
  journal={Annual review of anthropology},
  volume={31},
  number={1},
  pages={449--467},
  year={2002},
  publisher={Annual Reviews 4139 El Camino Way, PO Box 10139, Palo Alto, CA 94303-0139, USA}
}

@book{littlejohn2010theories,
  title={Theories of human communication},
  author={Littlejohn, Stephen W and Foss, Karen A},
  year={2010},
  publisher={Waveland press}
}

@article{van2008tsne,
  title={Visualizing data using t-SNE.},
  author={Van der Maaten, Laurens and Hinton, Geoffrey},
  journal={Journal of machine learning research},
  volume={9},
  number={11},
  year={2008}
}

@inproceedings{yang2024oasis,
  title={OASIS: Open Agents Social Interaction Simulations on One Million Agents},
  author={Yang, Ziyi and Zhang, Zaibin and Zheng, Zirui and Jiang, Yuxian and Gan, Ziyue and Wang, Zhiyu and Ling, Zijian and Ma, Martz and Dong, Bowen and Gupta, Prateek and others},
  booktitle={NeurIPS 2024 Workshop on Open-World Agents},
    year={2024}
}

@article{zhang2025socioverse,
  title={Socioverse: A world model for social simulation powered by llm agents and a pool of 10 million real-world users},
  author={Zhang, Xinnong and Lin, Jiayu and Mou, Xinyi and Yang, Shiyue and Liu, Xiawei and Sun, Libo and Lyu, Hanjia and Yang, Yihang and Qi, Weihong and Chen, Yue and others},
  journal={arXiv preprint arXiv:2504.10157},
  year={2025}
}

@misc{moltbook2026,
  author       = {{Moltbook}},
  title        = {Moltbook: The Online Community for AI Agents},
  year         = {2026},
  howpublished = {\url{https://www.moltbook.com/}},
  note         = {Accessed: 2026-04-30}
}

@article{holtz2026anatomy,
  title={The anatomy of the Moltbook social graph},
  author={Holtz, David},
  journal={arXiv preprint arXiv:2602.10131},
  year={2026}
}

@misc{rollingstone_meta_ai_users_2023,
  author       = {{Rolling Stone}},
  title        = {Meta Plans to Introduce AI Users on Facebook and Instagram},
  howpublished = {\url{https://www.rollingstone.com/culture/culture-news/meta-ai-users-facebook-instagram-1235221430/}},
  year         = {2023},
  note         = {Accessed: 2026-04-06}
}

@article{cau2025language,
  title={Language-driven opinion dynamics in agent-based simulations with llms},
  author={Cau, Erica and Pansanella, Valentina and Pedreschi, Dino and Rossetti, Giulio},
  journal={arXiv preprint arXiv:2502.19098},
  year={2025}
}

@article{park2024generative,
  title={Generative agent simulations of 1,000 people},
  author={Park, Joon Sung and Zou, Carolyn Q and Shaw, Aaron and Hill, Benjamin Mako and Cai, Carrie and Morris, Meredith Ringel and Willer, Robb and Liang, Percy and Bernstein, Michael S},
  journal={arXiv preprint arXiv:2411.10109},
  year={2024}
}

@article{ge2024scaling,
  title={Scaling synthetic data creation with 1,000,000,000 personas},
  author={Ge, Tao and Chan, Xin and Wang, Xiaoyang and Yu, Dian and Mi, Haitao and Yu, Dong},
  journal={arXiv preprint arXiv:2406.20094},
  year={2024}
}

@article{chen2024persona,
  title={From persona to personalization: A survey on role-playing language agents},
  author={Chen, Jiangjie and Wang, Xintao and Xu, Rui and Yuan, Siyu and Zhang, Yikai and Shi, Wei and Xie, Jian and Li, Shuang and Yang, Ruihan and Zhu, Tinghui and others},
  journal={arXiv preprint arXiv:2404.18231},
  year={2024}
}

@inproceedings{tseng2024two,
  title={Two tales of persona in llms: A survey of role-playing and personalization},
  author={Tseng, Yu-Min and Huang, Yu-Chao and Hsiao, Teng-Yun and Chen, Wei-Lin and Huang, Chao-Wei and Meng, Yu and Chen, Yun-Nung},
  booktitle={Findings of the Association for Computational Linguistics: EMNLP 2024},
  pages={16612--16631},
  year={2024}
}

@inproceedings{
zhou2025personaeval,
title={PersonaEval: Are {LLM} Evaluators Human Enough to Judge Role-Play?},
author={Lingfeng Zhou and Jialing Zhang and Jin Gao and Mohan Jiang and Dequan Wang},
booktitle={Second Conference on Language Modeling},
year={2025},
url={https://openreview.net/forum?id=drdrFhKYjP}
}

@incollection{jill2025ways,
  title={Ways to React on Social Media: Expressing the Self},
  author={Jill Jill, Mercy},
  booktitle={Digital Landscape: Communication in the 21st Century},
  pages={61--77},
  year={2025},
  publisher={Springer}
}

@article{kramer2021feel,
  title={I feel what they say: The effect of social media comments on viewers’ affective reactions toward elevating online videos},
  author={Kr{\"a}mer, Nicole C and Neubaum, German and Winter, Stephan and Schaewitz, Leonie and Eimler, Sabrina and Oliver, Mary Beth},
  journal={Media Psychology},
  volume={24},
  number={3},
  pages={332--358},
  year={2021},
  publisher={Taylor \& Francis}
}

@article{matook2022user,
  title={User comments in social media firestorms: A mixed-method study of purpose, tone, and motivation},
  author={Matook, Sabine and Dennis, Alan R and Wang, Yazhu Maggie},
  journal={Journal of Management Information Systems},
  volume={39},
  number={3},
  pages={673--705},
  year={2022},
  publisher={Taylor \& Francis}
}

@article{northup2022personality,
  title={Personality traits, personal motivations, and online news and social media commenting},
  author={Northup, Temple and Santana, Arthur D and Choi, HoJoon and Puspita, Ratna},
  journal={Journal of Media and Communication Studies},
  volume={14},
  number={3},
  pages={68--78},
  year={2022},
  publisher={Academic Journals}
}

@article{abdalla2025future,
  title={The Future of Artificial Intelligence in the Face of Data Scarcity.},
  author={Abdalla, Hemn Barzan and Kumar, Yulia and Marchena, Jose and Guzman, Stephany and Awlla, Ardalan and Gheisari, Mehdi and Cheraghy, Maryam},
  journal={Computers, Materials \& Continua},
  volume={84},
  number={1},
  year={2025}
}

@article{liu2025datasets,
  title={Datasets for large language models: A comprehensive survey},
  author={Liu, Yang and Cao, Jiahuan and Liu, Chongyu and Ding, Kai and Jin, Lianwen},
  journal={Artificial Intelligence Review},
  volume={58},
  number={12},
  pages={403},
  year={2025},
  publisher={Springer}
}

@inproceedings{villalobos2024position,
  title={Position: Will we run out of data? Limits of LLM scaling based on human-generated data},
  author={Villalobos, Pablo and Ho, Anson and Sevilla, Jaime and Besiroglu, Tamay and Heim, Lennart and Hobbhahn, Marius},
  booktitle={Forty-first International Conference on Machine Learning},
  year={2024}
}

@inproceedings{
longpre2025bridging,
title={Bridging the Data Provenance Gap Across Text, Speech, and Video},
author={Shayne Longpre and Nikhil Singh and Manuel Cherep and Kushagra Tiwary and Joanna Materzynska and William Brannon and Robert Mahari and Naana Obeng-Marnu and Manan Dey and Mohammed Hamdy and Nayan Saxena and Ahmad Mustafa Anis and Emad A. Alghamdi and Vu Minh Chien and Da Yin and Kun Qian and Yizhi LI and Minnie Liang and An Dinh and Shrestha Mohanty and Deividas Mataciunas and Tobin South and Jianguo Zhang and Ariel N. Lee and Campbell S. Lund and Christopher Klamm and Damien Sileo and Diganta Misra and Enrico Shippole and Kevin Klyman and Lester James Validad Miranda and Niklas Muennighoff and Seonghyeon Ye and Seungone Kim and Vipul Gupta and Vivek Sharma and Xuhui Zhou and Caiming Xiong and Luis Villa and Stella Biderman and Alex Pentland and Sara Hooker and Jad Kabbara},
booktitle={The Thirteenth International Conference on Learning Representations},
year={2025},
url={https://openreview.net/forum?id=G5DziesYxL}
}

@inproceedings{kearney2025echoes,
  title={Echoes amplified: a study of AI-generated content and digital echo chambers},
  author={Kearney, Ashley and Poredi, Nihal and Shelton, Joseph A and Akcinaroglu, Seden and Karakoc, Ekrem and Tran, Thi and Chen, Yu},
  booktitle={Disruptive Technologies in Information Sciences IX},
  volume={13480},
  pages={164--202},
  year={2025},
  organization={SPIE}
}

@article{fang2024bias,
  title={Bias of AI-generated content: an examination of news produced by large language models},
  author={Fang, Xiao and Che, Shangkun and Mao, Minjia and Zhang, Hongzhe and Zhao, Ming and Zhao, Xiaohang},
  journal={Scientific Reports},
  volume={14},
  number={1},
  pages={5224},
  year={2024},
  publisher={Nature Publishing Group UK London}
}

@inproceedings{anderson2024homogenization,
  title={Homogenization effects of large language models on human creative ideation},
  author={Anderson, Barrett R and Shah, Jash Hemant and Kreminski, Max},
  booktitle={Proceedings of the 16th conference on creativity \& cognition},
  pages={413--425},
  year={2024}
}

@inproceedings{
seddik2024how,
title={How bad is training on synthetic data? A statistical analysis of language model collapse},
author={Mohamed El Amine Seddik and Suei-Wen Chen and Soufiane Hayou and Pierre Youssef and Merouane Abdelkader DEBBAH},
booktitle={First Conference on Language Modeling},
year={2024},
url={https://openreview.net/forum?id=t3z6UlV09o}
}

@article{shumailov2023curse,
  title={The curse of recursion: Training on generated data makes models forget},
  author={Shumailov, Ilia and Shumaylov, Zakhar and Zhao, Yiren and Gal, Yarin and Papernot, Nicolas and Anderson, Ross},
  journal={arXiv preprint arXiv:2305.17493},
  year={2023}
}

@article{shumailov2024ai,
  title={AI models collapse when trained on recursively generated data},
  author={Shumailov, Ilia and Shumaylov, Zakhar and Zhao, Yiren and Papernot, Nicolas and Anderson, Ross and Gal, Yarin},
  journal={Nature},
  volume={631},
  number={8022},
  pages={755--759},
  year={2024},
  publisher={Nature Publishing Group UK London}
}

@inproceedings{gerstgrasser2024model,
  title={Is Model Collapse Inevitable? Breaking the Curse of Recursion by Accumulating Real and Synthetic Data},
  author={Gerstgrasser, M and Schaeffer, R and Dey, A and Rafailov, R and Korbak, T and Sleight, H and Agrawal, R and Hughes, J and Pai, DB and Gromov, A and others},
  year={2024},
    booktitle={First Conference on Language Modeling}
}

@article{feng2025one,
  title={When one llm drools, multi-llm collaboration rules},
  author={Feng, Shangbin and Ding, Wenxuan and Liu, Alisa and Wang, Zifeng and Shi, Weijia and Wang, Yike and Shen, Zejiang and Han, Xiaochuang and Lang, Hunter and Lee, Chen-Yu and others},
  journal={arXiv preprint arXiv:2502.04506},
  year={2025}
}

@inproceedings{feng2025model,
  title={Model Swarms: Collaborative Search to Adapt LLM Experts via Swarm Intelligence},
  author={Feng, Shangbin and Wang, Zifeng and Wang, Yike and Ebrahimi, Sayna and Palangi, Hamid and Miculicich, Lesly and Kulshrestha, Achin and Rauschmayr, Nathalie and Choi, Yejin and Tsvetkov, Yulia and others},
  booktitle={International Conference on Machine Learning},
  pages={16904--16930},
  year={2025},
  organization={PMLR}
}

@article{zhang2025verbalized,
  title={Verbalized sampling: How to mitigate mode collapse and unlock llm diversity},
  author={Zhang, Jiayi and Yu, Simon and Chong, Derek and Sicilia, Anthony and Tomz, Michael R and Manning, Christopher D and Shi, Weiyan},
  journal={arXiv preprint arXiv:2510.01171},
  year={2025}
}

@inproceedings{jiangartificial,
  title={Artificial Hivemind: The Open-Ended Homogeneity of Language Models (and Beyond)},
  author={Jiang, Liwei and Chai, Yuanjun and Li, Margaret and Liu, Mickel and Fok, Raymond and Dziri, Nouha and Tsvetkov, Yulia and Sap, Maarten and Choi, Yejin},
  booktitle={The Thirty-ninth Annual Conference on Neural Information Processing Systems Datasets and Benchmarks Track},
    year={2025}
}

@article{zhang2025noveltybench,
  title={NoveltyBench: Evaluating Language Models for Humanlike Diversity},
  author={Zhang, Yiming and Diddee, Harshita and Holm, Susan and Liu, Hanchen and Liu, Xinyue and Samuel, Vinay and Wang, Barry and Ippolito, Daphne},
  journal={arXiv preprint arXiv:2504.05228},
  year={2025}
}

@inproceedings{lu2024creativity_index_paper,
  title={AI as Humanity's Salieri: Quantifying Linguistic Creativity of Language Models via Systematic Attribution of Machine Text against Web Text},
  author={Lu, Ximing and Sclar, Melanie and Hallinan, Skyler and Mireshghallah, Niloofar and Liu, Jiacheng and Han, Seungju and Ettinger, Allyson and Jiang, Liwei and Chandu, Khyathi and Dziri, Nouha and others},
   booktitle={The Thirteenth International Conference on Learning Representations},
  year={2024}
}

@article{guo2025benchmarking,
  title={Benchmarking linguistic diversity of large language models},
  author={Guo, Yanzhu and Shang, Guokan and Clavel, Chlo{\'e}},
  journal={Transactions of the Association for Computational Linguistics},
  volume={13},
  pages={1507--1526},
  year={2025},
  publisher={MIT Press 255 Main Street, 9th Floor, Cambridge, Massachusetts 02142, USA~…}
}

@article{cann2023using,
  title={Using semantic similarity and text embedding to measure the social media echo of strategic communications},
  author={Cann, Tristan JB and Dennes, Ben and Coan, Travis and O'Neill, Saffron and Williams, Hywel TP},
  journal={arXiv preprint arXiv:2303.16694},
  year={2023}
}

@inproceedings{Chakrabarty2024art,
author = {Chakrabarty, Tuhin and Laban, Philippe and Agarwal, Divyansh and Muresan, Smaranda and Wu, Chien-Sheng},
title = {Art or Artifice? Large Language Models and the False Promise of Creativity},
year = {2024},
isbn = {9798400703300},
publisher = {Association for Computing Machinery},
address = {New York, NY, USA},
url = {https://doi.org/10.1145/3613904.3642731},
doi = {10.1145/3613904.3642731},
booktitle = {Proceedings of the 2024 CHI Conference on Human Factors in Computing Systems},
articleno = {30},
numpages = {34},
location = {Honolulu, HI, USA},
series = {CHI '24}
}

@inproceedings{
zhang2024forcing,
title={Forcing Diffuse Distributions out of Language Models},
author={Yiming Zhang and Avi Schwarzschild and Nicholas Carlini and J Zico Kolter and Daphne Ippolito},
booktitle={First Conference on Language Modeling},
year={2024},
url={https://openreview.net/forum?id=9JY1QLVFPZ}
}

@inproceedings{
west2025base,
title={Base Models Beat Aligned Models at Randomness and Creativity},
author={Peter West and Christopher Potts},
booktitle={Second Conference on Language Modeling},
year={2025},
url={https://openreview.net/forum?id=vqN8uom4A1}
}

@inproceedings{
zhou2024sotopia,
title={{SOTOPIA}: Interactive Evaluation for Social Intelligence in Language Agents},
author={Xuhui Zhou and Hao Zhu and Leena Mathur and Ruohong Zhang and Haofei Yu and Zhengyang Qi and Louis-Philippe Morency and Yonatan Bisk and Daniel Fried and Graham Neubig and Maarten Sap},
booktitle={The Twelfth International Conference on Learning Representations},
year={2024},
url={https://openreview.net/forum?id=mM7VurbA4r}
}

@inproceedings{
chung2025modifying,
title={Modifying Large Language Model Post-Training for Diverse Creative Writing},
author={John Joon Young Chung and Vishakh Padmakumar and Melissa Roemmele and Yuqian Sun and Max Kreminski},
booktitle={Second Conference on Language Modeling},
year={2025},
url={https://openreview.net/forum?id=1Pmuw08LoM}
}

@misc{lanchantin2025diversepreferenceoptimization,
      title={Diverse Preference Optimization}, 
      author={Jack Lanchantin and Angelica Chen and Shehzaad Dhuliawala and Ping Yu and Jason Weston and Sainbayar Sukhbaatar and Ilia Kulikov},
      year={2025},
      eprint={2501.18101},
      archivePrefix={arXiv},
      primaryClass={cs.CL},
      url={https://arxiv.org/abs/2501.18101}, 
}

@inproceedings{
Holtzman2020The,
title={The Curious Case of Neural Text Degeneration},
author={Ari Holtzman and Jan Buys and Li Du and Maxwell Forbes and Yejin Choi},
booktitle={International Conference on Learning Representations},
year={2020},
url={https://openreview.net/forum?id=rygGQyrFvH}
}

@inproceedings{yang-etal-2022-re3,
    title = "Re3: Generating Longer Stories With Recursive Reprompting and Revision",
    author = "Yang, Kevin  and
      Tian, Yuandong  and
      Peng, Nanyun  and
      Klein, Dan",
    editor = "Goldberg, Yoav  and
      Kozareva, Zornitsa  and
      Zhang, Yue",
    booktitle = "Proceedings of the 2022 Conference on Empirical Methods in Natural Language Processing",
    month = dec,
    year = "2022",
    address = "Abu Dhabi, United Arab Emirates",
    publisher = "Association for Computational Linguistics",
    url = "https://aclanthology.org/2022.emnlp-main.296/",
    doi = "10.18653/v1/2022.emnlp-main.296",
    pages = "4393--4479",
}

@inproceedings{ismayilzada-etal-2025-creative,
    title = "Creative Preference Optimization",
    author = "Ismayilzada, Mete  and
      Laverghetta Jr., Antonio  and
      Luchini, Simone A.  and
      Patel, Reet  and
      Bosselut, Antoine  and
      Plas, Lonneke Van Der  and
      Beaty, Roger E.",
    editor = "Christodoulopoulos, Christos  and
      Chakraborty, Tanmoy  and
      Rose, Carolyn  and
      Peng, Violet",
    booktitle = "Findings of the Association for Computational Linguistics: EMNLP 2025",
    month = nov,
    year = "2025",
    address = "Suzhou, China",
    publisher = "Association for Computational Linguistics",
    url = "https://aclanthology.org/2025.findings-emnlp.509/",
    doi = "10.18653/v1/2025.findings-emnlp.509",
    pages = "9580--9609",
    ISBN = "979-8-89176-335-7",
  
}

@inproceedings{
minh2025turning,
title={Turning Up the Heat: Min-p Sampling for Creative and Coherent {LLM} Outputs},
author={Nguyen Nhat Minh and Andrew Baker and Clement Neo and Allen G Roush and Andreas Kirsch and Ravid Shwartz-Ziv},
booktitle={The Thirteenth International Conference on Learning Representations},
year={2025},
url={https://openreview.net/forum?id=FBkpCyujtS}
}

@inproceedings{omahony2024attributing,
title={Attributing Mode Collapse in the fine-tuning of Large Language Models},
author={Laura O'Mahony and Leo Grinsztajn and Hailey Schoelkopf and Stella Biderman},
booktitle={ICLR 2024 Workshop on Mathematical and Empirical Understanding of Foundation Models},
year={2024},
url={https://openreview.net/forum?id=3pDMYjpOxk}
}

@inproceedings{kirk2024understanding,
title={Understanding the Effects of {RLHF} on {LLM} Generalisation and Diversity},
author={Robert Kirk and Ishita Mediratta and Christoforos Nalmpantis and Jelena Luketina and Eric Hambro and Edward Grefenstette and Roberta Raileanu},
booktitle={The Twelfth International Conference on Learning Representations},
year={2024},
url={https://openreview.net/forum?id=PXD3FAVHJT}
}

@inproceedings{guo-etal-2024-curious,
    title = "The Curious Decline of Linguistic Diversity: Training Language Models on Synthetic Text",
    author = "Guo, Yanzhu  and
      Shang, Guokan  and
      Vazirgiannis, Michalis  and
      Clavel, Chlo{\'e}",
    editor = "Duh, Kevin  and
      Gomez, Helena  and
      Bethard, Steven",
    booktitle = "Findings of the Association for Computational Linguistics: NAACL 2024",
    month = jun,
    year = "2024",
    address = "Mexico City, Mexico",
    publisher = "Association for Computational Linguistics",
    url = "https://aclanthology.org/2024.findings-naacl.228/",
    doi = "10.18653/v1/2024.findings-naacl.228",
    pages = "3589--3604",
}

@article{chen2024diversity,
  title={On the diversity of synthetic data and its impact on training large language models},
  author={Chen, Hao and Waheed, Abdul and Li, Xiang and Wang, Yidong and Wang, Jindong and Raj, Bhiksha and Abdin, Marah I},
  journal={arXiv preprint arXiv:2410.15226},
  year={2024}
}

@inproceedings{kang-etal-2025-demystifying,
    title = "Demystifying Synthetic Data in {LLM} Pre-training: A Systematic Study of Scaling Laws, Benefits, and Pitfalls",
    author = "Kang, Feiyang  and
      Ardalani, Newsha  and
      Kuchnik, Michael  and
      Emad, Youssef  and
      Elhoushi, Mostafa  and
      Sengupta, Shubhabrata  and
      Li, Shang-Wen  and
      Raghavendra, Ramya  and
      Jia, Ruoxi  and
      Wu, Carole-Jean",
    editor = "Christodoulopoulos, Christos  and
      Chakraborty, Tanmoy  and
      Rose, Carolyn  and
      Peng, Violet",
    booktitle = "Proceedings of the 2025 Conference on Empirical Methods in Natural Language Processing",
    month = nov,
    year = "2025",
    address = "Suzhou, China",
    publisher = "Association for Computational Linguistics",
    url = "https://aclanthology.org/2025.emnlp-main.544/",
    doi = "10.18653/v1/2025.emnlp-main.544",
    pages = "10739--10758",
    ISBN = "979-8-89176-332-6",
}

@article{qin2025scaling,
  title={Scaling laws of synthetic data for language models},
  author={Qin, Zeyu and Dong, Qingxiu and Zhang, Xingxing and Dong, Li and Huang, Xiaolong and Yang, Ziyi and Khademi, Mahmoud and Zhang, Dongdong and Awadalla, Hany Hassan and Fung, Yi R and others},
  journal={arXiv preprint arXiv:2503.19551},
  year={2025}
}

@ARTICLE{11080380,
  author={Nadǎş, Mihai and Dioşan, Laura and Tomescu, Andreea},
  journal={IEEE Access}, 
  title={Synthetic Data Generation Using Large Language Models: Advances in Text and Code}, 
  year={2025},
  volume={13},
  number={},
  pages={134615-134633},
  doi={10.1109/ACCESS.2025.3589503}}

@software{benallal2024cosmopedia,
  author = {Ben Allal, Loubna and Lozhkov, Anton and Penedo, Guilherme and Wolf, Thomas and von Werra, Leandro},
  title = {Cosmopedia},
  month = February,
  year = 2024,
  url = {https://huggingface.co/datasets/HuggingFaceTB/cosmopedia}
}

@misc{finephrase,
  title={The Synthetic Data Playbook: Generating Trillions of the Finest Tokens},
  author={Joel Niklaus and Guilherme Penedo and Hynek Kydlicek and Elie Bakouch and Lewis Tunstall and Ed Beeching and Thibaud Frere and Colin Raffel and Leandro von Werra and Thomas Wolf},
  year={2026},
}

@incollection{hoeffding1992class,
  title={A class of statistics with asymptotically normal distribution},
  author={Hoeffding, Wassily},
  booktitle={Breakthroughs in statistics: Foundations and basic theory},
  pages={308--334},
  year={1992},
  publisher={Springer}
}

@article{heaton2025chatgpt,
  title={“ChatGPT says no”: agency, trust, and blame in Twitter discourses after the launch of ChatGPT},
  author={Heaton, Dan and Nichele, Elena and Clos, Jeremie and Fischer, Joel E},
  journal={AI and Ethics},
  volume={5},
  number={1},
  pages={653--675},
  year={2025},
  publisher={Springer}
}

@inproceedings{sun2025we,
  title={Are we in the AI-generated text world already? Quantifying and monitoring AIGT on social media},
  author={Sun, Zhen and Zhang, Zongmin and Shen, Xinyue and Zhang, Ziyi and Liu, Yule and Backes, Michael and Zhang, Yang and He, Xinlei},
  booktitle={Proceedings of the 63rd Annual Meeting of the Association for Computational Linguistics (Volume 1: Long Papers)},
  pages={22975--23005},
  year={2025}
}

@article{goldstein2023coming,
  title={The coming age of AI-powered propaganda},
  author={Goldstein, Josh A and Sastry, Girish},
  journal={Foreign Affairs},
  volume={7},
  year={2023}
}

@inproceedings{thukral2018analyzing,
  title={Analyzing behavioral trends in community driven discussion platforms like reddit},
  author={Thukral, Sachin and Meisheri, Hardik and Kataria, Tushar and Agarwal, Aman and Verma, Ishan and Chatterjee, Arnab and Dey, Lipika},
  booktitle={2018 IEEE/ACM International Conference on Advances in Social Networks Analysis and Mining (ASONAM)},
  pages={662--669},
  year={2018},
  organization={IEEE}
}

@article{yu2024characterizing,
  title={Characterizing the structure of online conversations across Reddit},
  author={Yu, Yulin and Jiang, Julie and Dhillon, Paramveer S},
  journal={Proceedings of the ACM on Human-Computer Interaction},
  volume={8},
  number={CSCW2},
  pages={1--23},
  year={2024},
  publisher={ACM New York, NY, USA}
}

@article{ocal2021reasoning,
  title={Reasoning in social media: insights from Reddit “Change My View” submissions},
  author={{\"O}cal, Ay{\c{s}}e and Xiao, Lu and Park, Jaihyun},
  journal={Online Information Review},
  volume={45},
  number={7},
  pages={1208--1226},
  year={2021},
  publisher={Emerald Publishing Limited}
}

@article{choi2016social,
  title={What social media data should i use in my research?: A comparative analysis of twitter, youtube, reddit, and the new york times comments},
  author={Choi, Dongho and Matni, Ziad and Shah, Chirag},
  journal={Proceedings of the Association for Information Science and Technology},
  volume={53},
  number={1},
  pages={1--6},
  year={2016},
  publisher={Wiley Online Library}
}

@article{li2022youtube,
  title={YouTube as a source of misinformation on COVID-19 vaccination: a systematic analysis},
  author={Li, Heidi Oi-Yee and Pastukhova, Elena and Brandts-Longtin, Olivier and Tan, Marcus G and Kirchhof, Mark G},
  journal={BMJ global health},
  volume={7},
  number={3},
  year={2022},
  publisher={BMJ Publishing Group Ltd}
}

@inproceedings{nasim2013commenting,
  title={On commenting behavior of Facebook users},
  author={Nasim, Mehwish and Ilyas, Muhammad U and Rextin, Aimal and Nasim, Nazish},
  booktitle={Proceedings of the 24th ACM Conference on Hypertext and Social Media},
  pages={179--183},
  year={2013}
}

@article{stavros2014understanding,
  title={Understanding fan motivation for interacting on social media},
  author={Stavros, Constantino and Meng, Matthew D and Westberg, Kate and Farrelly, Francis},
  journal={Sport management review},
  volume={17},
  number={4},
  pages={455--469},
  year={2014},
  publisher={Elsevier}
}

@inproceedings{lapides2015news,
  title={News Feed: What's in it for Me?},
  author={Lapides, Paul and Chokshi, Apoorve and Carpendale, Sheelagh and Greenberg, Saul},
  booktitle={Proceedings of the 33rd annual ACM conference on human factors in computing systems},
  pages={163--172},
  year={2015}
}

@article{maiz2016factors,
  title={Factors affecting social interaction on social network sites: the Facebook case},
  author={Maiz, Ander and Arranz, Nieves and Fdez. de Arroyabe, Juan Carlos},
  journal={Journal of Enterprise Information Management},
  volume={29},
  number={5},
  pages={630--649},
  year={2016},
  publisher={Emerald Group Publishing Limited}
}

@article{garcia2017understanding,
  title={Understanding popularity, reputation, and social influence in the twitter society},
  author={Garcia, David and Mavrodiev, Pavlin and Casati, Daniele and Schweitzer, Frank},
  journal={Policy \& Internet},
  volume={9},
  number={3},
  pages={343--364},
  year={2017},
  publisher={Wiley Online Library}
}

@article{keller2020political,
  title={Political astroturfing on Twitter: How to coordinate a disinformation campaign},
  author={Keller, Franziska B and Schoch, David and Stier, Sebastian and Yang, JungHwan},
  journal={Political communication},
  volume={37},
  number={2},
  pages={256--280},
  year={2020},
  publisher={Taylor \& Francis}
}

@article{khaund2021social,
  title={Social bots and their coordination during online campaigns: a survey},
  author={Khaund, Tuja and Kirdemir, Baris and Agarwal, Nitin and Liu, Huan and Morstatter, Fred},
  journal={IEEE Transactions on Computational Social Systems},
  volume={9},
  number={2},
  pages={530--545},
  year={2021},
  publisher={IEEE}
}

@article{naaman2011hip,
  title={Hip and trendy: Characterizing emerging trends on Twitter},
  author={Naaman, Mor and Becker, Hila and Gravano, Luis},
  journal={Journal of the American Society for Information Science and Technology},
  volume={62},
  number={5},
  pages={902--918},
  year={2011},
  publisher={Wiley Online Library}
}

@article{nutakki2025there,
  title={Is there anything left?: A global analysis on changes in engagement with political content on twitter in the musk era},
  author={Nutakki, Brahmani and Navarrete, Rosa M and Carteny, Giuseppe and Weber, Ingmar},
  journal={Journal of Quantitative Description: Digital Media},
  volume={5},
  year={2025}
}

@article{ma2022m,
  title={" I'm not sure what difference is between their content and mine, other than the person itself" A Study of Fairness Perception of Content Moderation on YouTube},
  author={Ma, Renkai and Kou, Yubo},
  journal={Proceedings of the ACM on Human-Computer Interaction},
  volume={6},
  number={CSCW2},
  pages={1--28},
  year={2022},
  publisher={ACM New York, NY, USA}
}

@inproceedings{schultes2013leave,
  author    = {Schultes, Peter and Dorner, Verena and Lehner, Franz},
  title     = {Leave a Comment! An In-Depth Analysis of User Comments on {YouTube}},
  booktitle = {Proceedings of the 11th International Conference on Wirtschaftsinformatik (WI)},
  year      = {2013},
  pages     = {42},
  url       = {https://aisel.aisnet.org/wi2013/42},
  address   = {Leipzig, Germany}
}

@misc{arthurs2018researching,
  title={Researching youtube},
  author={Arthurs, Jane and Drakopoulou, Sophia and Gandini, Alessandro},
  journal={Convergence},
  volume={24},
  number={1},
  pages={3--15},
  year={2018},
  publisher={SAGE Publications Sage UK: London, England}
}

@article{adcock2026llama4,
  title={The Llama 4 Herd: Architecture, Training, Evaluation, and Deployment Notes},
  author={Adcock, Aaron and Srivastava, Aayushi and Dubey, Abhimanyu and Jauhri, Abhinav and Pande, Abhinav and Pandey, Abhinav and Sharma, Abhinav and Kadian, Abhishek and Kumawat, Abhishek and Kelsey, Adam and others},
  journal={arXiv preprint arXiv:2601.11659},
  year={2026}
}

@article{jiang2024mixtral,
  title={Mixtral of experts},
  author={Jiang, Albert Q and Sablayrolles, Alexandre and Roux, Antoine and Mensch, Arthur and Savary, Blanche and Bamford, Chris and Chaplot, Devendra Singh and Casas, Diego de las and Hanna, Emma Bou and Bressand, Florian and others},
  journal={arXiv preprint arXiv:2401.04088},
  year = 2024,
  }

@article{yang2025qwen3,
  title={Qwen3 technical report},
  author={Yang, An and Li, Anfeng and Yang, Baosong and Zhang, Beichen and Hui, Binyuan and Zheng, Bo and Yu, Bowen and Gao, Chang and Huang, Chengen and Lv, Chenxu and others},
  journal={arXiv preprint arXiv:2505.09388},
  year={2025}
}

@article{qwen2024qwen2.5,
  title={Qwen2. 5 technical report},
  author={Qwen, A Yang and Yang, Baosong and Zhang, Beichen and Hui, Binyuan and Zheng, Bo and Yu, Bowen and Li, Chengpeng and Liu, Dayiheng and Huang, Fei and Wei, Haoran and others},
  journal={arXiv preprint},
  year={2024}
}

@article{agarwal2025gptoss,
  title={gpt-oss-120b \& gpt-oss-20b model card},
  author={Agarwal, Sandhini and Ahmad, Lama and Ai, Jason and Altman, Sam and Applebaum, Andy and Arbus, Edwin and Arora, Rahul K and Bai, Yu and Baker, Bowen and Bao, Haiming and others},
  journal={arXiv preprint arXiv:2508.10925},
  year={2025}
}

@article{achiam2023gpt4,
  title={Gpt-4 technical report},
  author={Achiam, Josh and Adler, Steven and Agarwal, Sandhini and Ahmad, Lama and Akkaya, Ilge and Aleman, Florencia Leoni and Almeida, Diogo and Altenschmidt, Janko and Altman, Sam and Anadkat, Shyamal and others},
  journal={arXiv preprint arXiv:2303.08774},
  year={2023}
}

@article{singh2025gpt5,
  title={Openai gpt-5 system card},
  author={Singh, Aaditya and Fry, Adam and Perelman, Adam and Tart, Adam and Ganesh, Adi and El-Kishky, Ahmed and McLaughlin, Aidan and Low, Aiden and Ostrow, AJ and Ananthram, Akhila and others},
  journal={arXiv preprint arXiv:2601.03267},
  year={2025}
}

@misc{anthropic2025claude37,
  author       = {Anthropic},
  title        = {Claude 3.7 Sonnet and the first hybrid reasoning model},
  year         = {2025},
  howpublished = {\url{https://www.anthropic.com/news/claude-3-7-sonnet}},
  note         = {Accessed: 2026-04-08}
}

@misc{anthropic2025claudesonnet45,
  author       = {Anthropic},
  title        = {Introducing Claude Sonnet 4.5},
  year         = {2025},
  month        = {September},
  howpublished = {\url{https://www.anthropic.com/news/claude-sonnet-4-5}},
  note         = {Accessed: 2026-04-08}
}

@misc{anthropic2025claudehaiku45,
  author       = {Anthropic},
  title        = {Introducing Claude Haiku 4.5},
  year         = {2025},
  month        = {October},
  howpublished = {\url{https://www.anthropic.com/news/claude-haiku-4-5}},
  note         = {Accessed: 2026-04-08}
}

@article{comanici2025gemini2.5,
  title={Gemini 2.5: Pushing the frontier with advanced reasoning, multimodality, long context, and next generation agentic capabilities},
  author={Comanici, Gheorghe and Bieber, Eric and Schaekermann, Mike and Pasupat, Ice and Sachdeva, Noveen and Dhillon, Inderjit and Blistein, Marcel and Ram, Ori and Zhang, Dan and Rosen, Evan and others},
  journal={arXiv preprint arXiv:2507.06261},
  year={2025}
}

@misc{gemmateam2025gemma3technicalreport,
      title={Gemma 3 Technical Report}, 
      author={Gemma Team and Aishwarya Kamath and Johan Ferret and Shreya Pathak and Nino Vieillard and Ramona Merhej and Sarah Perrin and Tatiana Matejovicova and Alexandre Ramé and Morgane Rivière and Louis Rouillard and Thomas Mesnard and Geoffrey Cideron and Jean-bastien Grill and Sabela Ramos and Edouard Yvinec and Michelle Casbon and Etienne Pot and Ivo Penchev and Gaël Liu and Francesco Visin and Kathleen Kenealy and Lucas Beyer and Xiaohai Zhai and Anton Tsitsulin and Robert Busa-Fekete and Alex Feng and Noveen Sachdeva and Benjamin Coleman and Yi Gao and Basil Mustafa and Iain Barr and Emilio Parisotto and David Tian and Matan Eyal and Colin Cherry and Jan-Thorsten Peter and Danila Sinopalnikov and Surya Bhupatiraju and Rishabh Agarwal and Mehran Kazemi and Dan Malkin and Ravin Kumar and David Vilar and Idan Brusilovsky and Jiaming Luo and Andreas Steiner and Abe Friesen and Abhanshu Sharma and Abheesht Sharma and Adi Mayrav Gilady and Adrian Goedeckemeyer and Alaa Saade and Alex Feng and Alexander Kolesnikov and Alexei Bendebury and Alvin Abdagic and Amit Vadi and András György and André Susano Pinto and Anil Das and Ankur Bapna and Antoine Miech and Antoine Yang and Antonia Paterson and Ashish Shenoy and Ayan Chakrabarti and Bilal Piot and Bo Wu and Bobak Shahriari and Bryce Petrini and Charlie Chen and Charline Le Lan and Christopher A. Choquette-Choo and CJ Carey and Cormac Brick and Daniel Deutsch and Danielle Eisenbud and Dee Cattle and Derek Cheng and Dimitris Paparas and Divyashree Shivakumar Sreepathihalli and Doug Reid and Dustin Tran and Dustin Zelle and Eric Noland and Erwin Huizenga and Eugene Kharitonov and Frederick Liu and Gagik Amirkhanyan and Glenn Cameron and Hadi Hashemi and Hanna Klimczak-Plucińska and Harman Singh and Harsh Mehta and Harshal Tushar Lehri and Hussein Hazimeh and Ian Ballantyne and Idan Szpektor and Ivan Nardini and Jean Pouget-Abadie and Jetha Chan and Joe Stanton and John Wieting and Jonathan Lai and Jordi Orbay and Joseph Fernandez and Josh Newlan and Ju-yeong Ji and Jyotinder Singh and Kat Black and Kathy Yu and Kevin Hui and Kiran Vodrahalli and Klaus Greff and Linhai Qiu and Marcella Valentine and Marina Coelho and Marvin Ritter and Matt Hoffman and Matthew Watson and Mayank Chaturvedi and Michael Moynihan and Min Ma and Nabila Babar and Natasha Noy and Nathan Byrd and Nick Roy and Nikola Momchev and Nilay Chauhan and Noveen Sachdeva and Oskar Bunyan and Pankil Botarda and Paul Caron and Paul Kishan Rubenstein and Phil Culliton and Philipp Schmid and Pier Giuseppe Sessa and Pingmei Xu and Piotr Stanczyk and Pouya Tafti and Rakesh Shivanna and Renjie Wu and Renke Pan and Reza Rokni and Rob Willoughby and Rohith Vallu and Ryan Mullins and Sammy Jerome and Sara Smoot and Sertan Girgin and Shariq Iqbal and Shashir Reddy and Shruti Sheth and Siim Põder and Sijal Bhatnagar and Sindhu Raghuram Panyam and Sivan Eiger and Susan Zhang and Tianqi Liu and Trevor Yacovone and Tyler Liechty and Uday Kalra and Utku Evci and Vedant Misra and Vincent Roseberry and Vlad Feinberg and Vlad Kolesnikov and Woohyun Han and Woosuk Kwon and Xi Chen and Yinlam Chow and Yuvein Zhu and Zichuan Wei and Zoltan Egyed and Victor Cotruta and Minh Giang and Phoebe Kirk and Anand Rao and Kat Black and Nabila Babar and Jessica Lo and Erica Moreira and Luiz Gustavo Martins and Omar Sanseviero and Lucas Gonzalez and Zach Gleicher and Tris Warkentin and Vahab Mirrokni and Evan Senter and Eli Collins and Joelle Barral and Zoubin Ghahramani and Raia Hadsell and Yossi Matias and D. Sculley and Slav Petrov and Noah Fiedel and Noam Shazeer and Oriol Vinyals and Jeff Dean and Demis Hassabis and Koray Kavukcuoglu and Clement Farabet and Elena Buchatskaya and Jean-Baptiste Alayrac and Rohan Anil and Dmitry and Lepikhin and Sebastian Borgeaud and Olivier Bachem and Armand Joulin and Alek Andreev and Cassidy Hardin and Robert Dadashi and Léonard Hussenot},
      year={2025},
      eprint={2503.19786},
      archivePrefix={arXiv},
      primaryClass={cs.CL},
      url={https://arxiv.org/abs/2503.19786}, 
}

@article{liu2024deepseek,
  title={Deepseek-v3 technical report},
  author={Liu, Aixin and Feng, Bei and Xue, Bing and Wang, Bingxuan and Wu, Bochao and Lu, Chengda and Zhao, Chenggang and Deng, Chengqi and Zhang, Chenyu and Ruan, Chong and others},
  journal={arXiv preprint arXiv:2412.19437},
  year={2024}
}

@article{abdin2024phi4,
  title={Phi-4 technical report},
  author={Abdin, Marah and Aneja, Jyoti and Behl, Harkirat and Bubeck, S{\'e}bastien and Eldan, Ronen and Gunasekar, Suriya and Harrison, Michael and Hewett, Russell J and Javaheripi, Mojan and Kauffmann, Piero and others},
  journal={arXiv preprint arXiv:2412.08905},
  year={2024}
}

@article{abouelenin2025phi4-mini,
  title={Phi-4-mini technical report: Compact yet powerful multimodal language models via mixture-of-loras},
  author={Abouelenin, Abdelrahman and Ashfaq, Atabak and Atkinson, Adam and Awadalla, Hany and Bach, Nguyen and Bao, Jianmin and Benhaim, Alon and Cai, Martin and Chaudhary, Vishrav and Chen, Congcong and others},
  journal={arXiv preprint arXiv:2503.01743},
  year={2025}
}

@misc{abdin2024phi3,
      title={Phi-3 Technical Report: A Highly Capable Language Model Locally on Your Phone}, 
      author={Marah Abdin and Jyoti Aneja and Hany Awadalla and Ahmed Awadallah and Ammar Ahmad Awan and Nguyen Bach and Amit Bahree and Arash Bakhtiari and Jianmin Bao and Harkirat Behl and Alon Benhaim and Misha Bilenko and Johan Bjorck and Sébastien Bubeck and Martin Cai and Qin Cai and Vishrav Chaudhary and Dong Chen and Dongdong Chen and Weizhu Chen and Yen-Chun Chen and Yi-Ling Chen and Hao Cheng and Parul Chopra and Xiyang Dai and Matthew Dixon and Ronen Eldan and Victor Fragoso and Jianfeng Gao and Mei Gao and Min Gao and Amit Garg and Allie Del Giorno and Abhishek Goswami and Suriya Gunasekar and Emman Haider and Junheng Hao and Russell J. Hewett and Wenxiang Hu and Jamie Huynh and Dan Iter and Sam Ade Jacobs and Mojan Javaheripi and Xin Jin and Nikos Karampatziakis and Piero Kauffmann and Mahoud Khademi and Dongwoo Kim and Young Jin Kim and Lev Kurilenko and James R. Lee and Yin Tat Lee and Yuanzhi Li and Yunsheng Li and Chen Liang and Lars Liden and Xihui Lin and Zeqi Lin and Ce Liu and Liyuan Liu and Mengchen Liu and Weishung Liu and Xiaodong Liu and Chong Luo and Piyush Madan and Ali Mahmoudzadeh and David Majercak and Matt Mazzola and Caio César Teodoro Mendes and Arindam Mitra and Hardik Modi and Anh Nguyen and Brandon Norick and Barun Patra and Daniel Perez-Becker and Thomas Portet and Reid Pryzant and Heyang Qin and Marko Radmilac and Liliang Ren and Gustavo de Rosa and Corby Rosset and Sambudha Roy and Olatunji Ruwase and Olli Saarikivi and Amin Saied and Adil Salim and Michael Santacroce and Shital Shah and Ning Shang and Hiteshi Sharma and Yelong Shen and Swadheen Shukla and Xia Song and Masahiro Tanaka and Andrea Tupini and Praneetha Vaddamanu and Chunyu Wang and Guanhua Wang and Lijuan Wang and Shuohang Wang and Xin Wang and Yu Wang and Rachel Ward and Wen Wen and Philipp Witte and Haiping Wu and Xiaoxia Wu and Michael Wyatt and Bin Xiao and Can Xu and Jiahang Xu and Weijian Xu and Jilong Xue and Sonali Yadav and Fan Yang and Jianwei Yang and Yifan Yang and Ziyi Yang and Donghan Yu and Lu Yuan and Chenruidong Zhang and Cyril Zhang and Jianwen Zhang and Li Lyna Zhang and Yi Zhang and Yue Zhang and Yunan Zhang and Xiren Zhou},
      year={2024},
      eprint={2404.14219},
      archivePrefix={arXiv},
      primaryClass={cs.CL},
      url={https://arxiv.org/abs/2404.14219}, 
}

@techreport{xai2025grok4,
  author      = {{xAI}},
  title       = {Grok-4 Model Card},
  institution = {xAI},
  year        = {2025},
  month       = {August},
  url         = {https://data.x.ai/2025-08-20-grok-4-model-card.pdf},
  note        = {Technical Report}
}

@techreport{xai2025grok4-1,
  author      = {{xAI}},
  title       = {Grok-4.1 Model Card},
  institution = {xAI},
  year        = {2025},
  month       = {November},
  url         = {https://data.x.ai/2025-11-17-grok-4-1-model-card.pdf},
  note        = {Technical Report}
}

@misc{xai2025grok3,
  author       = {{xAI}},
  title        = {Grok 3 Beta — The Age of Reasoning Agents},
  year         = {2025},
  month        = {February},
  howpublished = {\url{https://x.ai/news/grok-3}},
  note         = {Accessed: 2026-04-08}
}

@misc{shenoda2024robertaspam,
  author       = {Shenoda, Michael },
  title        = {RoBERTa-based Spam Message Detection},
  year         = {2024},
  publisher    = {Hugging Face},
  howpublished = {\url{https://huggingface.co/mshenoda/roberta-spam}},
  note         = {Fine-tuned RoBERTa-base model for spam detection in SMS, Telegram, and Enron email datasets}
}

@misc{hartmann2022emotionenglish,
  author={Hartmann, Jochen},
  title={Emotion English DistilRoBERTa-base},
  year={2022},
  howpublished = {\url{https://huggingface.co/j-hartmann/emotion-english-distilroberta-base/}},
}

@inproceedings{babakov2023formality,
  title={Don’t lose the message while paraphrasing: A study on content preserving style transfer},
  author={Babakov, Nikolay and Dale, David and Gusev, Ilya and Krotova, Irina and Panchenko, Alexander},
  booktitle={International Conference on Applications of Natural Language to Information Systems},
  pages={47--61},
  year={2023},
  organization={Springer}
}

@misc{wu2023stereotype,
  author    = {Wu, Zekun},
  title     = {Sentence-Level Stereotype Detector},
  year      = {2023},
  publisher = {Hugging Face},
  journal   = {Hugging Face Repository},
  howpublished = {\url{https://huggingface.co/wu981526092/Sentence-Level-Stereotype-Detector}},
  note      = {Fine-tuned DistilBERT model for multi-class stereotype classification}
}

@inproceedings{vidgen2021ltoxicity,
  title={Learning from the Worst: Dynamically Generated Datasets to Improve Online Hate Detection},
  author={Bertie Vidgen and Tristan Thrush and Zeerak Waseem and Douwe Kiela},
  booktitle={ACL},
  year={2021}
}

@article{gehman2020realtoxicityprompts,
  title={Realtoxicityprompts: Evaluating neural toxic degeneration in language models},
  author={Gehman, Samuel and Gururangan, Suchin and Sap, Maarten and Choi, Yejin and Smith, Noah A},
  journal={arXiv preprint arXiv:2009.11462},
  year={2020}
}

@inproceedings{singh2021efficient,
  title={An efficient method for aspect based sentiment analysis using spacy and vader},
  author={Singh, Akhilesh Kumar and Verma, Ananya},
  booktitle={2021 10th IEEE International Conference on Communication Systems and Network Technologies (CSNT)},
  pages={130--135},
  year={2021},
  organization={IEEE}
}

@article{gaspar2016beyond,
  title={Beyond positive or negative: Qualitative sentiment analysis of social media reactions to unexpected stressful events},
  author={Gaspar, Rui and Pedro, Cl{\'a}udia and Panagiotopoulos, Panos and Seibt, Beate},
  journal={Computers in Human Behavior},
  volume={56},
  pages={179--191},
  year={2016},
  publisher={Elsevier}
}

@inproceedings{baranov-etal-2023-humor,
    title = "You Told Me That Joke Twice: A Systematic Investigation of Transferability and Robustness of Humor Detection Models",
    author = "Baranov, Alexander  and
      Kniazhevsky, Vladimir  and
      Braslavski, Pavel",
    editor = "Bouamor, Houda  and
      Pino, Juan  and
      Bali, Kalika",
    booktitle = "Proceedings of the 2023 Conference on Empirical Methods in Natural Language Processing",
    month = dec,
    year = "2023",
    address = "Singapore",
    publisher = "Association for Computational Linguistics",
    url = "https://aclanthology.org/2023.emnlp-main.845",
    doi = "10.18653/v1/2023.emnlp-main.845",
    pages = "13701--13715",
}

@inproceedings{wachowiak2022metaphor,
 title={Drum Up SUPPORT: Systematic Analysis of Image-Schematic Conceptual Metaphors},
 author={Wachowiak, Lennart and Gromann, Dagmar and Xu, Chao},
 booktitle={Proceedings of the 3rd Workshop on Figurative Language Processing (FLP)},
 pages={44--53},
 year={2022}
}

@misc{romero2020t5sarcasm,
  author       = {Romero, Manuel},
  title        = {T5-base fine-tuned for Sarcasm Detection on Twitter},
  year         = {2020},
  publisher    = {Hugging Face},
  howpublished = {\url{https://huggingface.co/mrm8488/t5-base-finetuned-sarcasm-twitter}},
  note         = {Modelo T5-base ajustado para la detección de sarcasmo utilizando el dataset de Twitter}
}

@article{grootendorst2022bertopic,
  title={BERTopic: Neural topic modeling with a class-based TF-IDF procedure},
  author={Grootendorst, Maarten},
  journal={arXiv preprint arXiv:2203.05794},
  year={2022}
}

@inproceedings{campello2013hdbscan,
  title={Density-based clustering based on hierarchical density estimates},
  author={Campello, Ricardo JGB and Moulavi, Davoud and Sander, J{\"o}rg},
  booktitle={Pacific-Asia conference on knowledge discovery and data mining},
  pages={160--172},
  year={2013},
  organization={Springer}
}

@article{chaturvedi2001k,
  title={K-modes clustering},
  author={Chaturvedi, Anil and Green, Paul E and Caroll, J Douglas},
  journal={Journal of classification},
  volume={18},
  number={1},
  pages={35--55},
  year={2001},
  publisher={Springer}
}

@article{vera2025embeddinggemma,
  title={Embeddinggemma: Powerful and lightweight text representations},
  author={Vera, Henrique Schechter and Dua, Sahil and Zhang, Biao and Salz, Daniel and Mullins, Ryan and Panyam, Sindhu Raghuram and Smoot, Sara and Naim, Iftekhar and Zou, Joe and Chen, Feiyang and others},
  journal={arXiv preprint arXiv:2509.20354},
  year={2025}
}

@inproceedings{
enevoldsen2025mmteb,
title={{MMTEB}: Massive Multilingual Text Embedding Benchmark},
author={Kenneth Enevoldsen and Isaac Chung and Imene Kerboua and M{\'a}rton Kardos and Ashwin Mathur and David Stap and Jay Gala and Wissam Siblini and Dominik Krzemi{\'n}ski and Genta Indra Winata and Saba Sturua and Saiteja Utpala and Mathieu Ciancone and Marion Schaeffer and Diganta Misra and Shreeya Dhakal and Jonathan Rystr{\o}m and Roman Solomatin and {\"O}mer Veysel {\c{C}}a{\u{g}}atan and Akash Kundu and Martin Bernstorff and Shitao Xiao and Akshita Sukhlecha and Bhavish Pahwa and Rafa{\l} Po{\'s}wiata and Kranthi Kiran GV and Shawon Ashraf and Daniel Auras and Bj{\"o}rn Pl{\"u}ster and Jan Philipp Harries and Lo{\"\i}c Magne and Isabelle Mohr and Dawei Zhu and Hippolyte Gisserot-Boukhlef and Tom Aarsen and Jan Kostkan and Konrad Wojtasik and Taemin Lee and Marek Suppa and Crystina Zhang and Roberta Rocca and Mohammed Hamdy and Andrianos Michail and John Yang and Manuel Faysse and Aleksei Vatolin and Nandan Thakur and Manan Dey and Dipam Vasani and Pranjal A Chitale and Simone Tedeschi and Nguyen Tai and Artem Snegirev and Mariya Hendriksen and Michael G{\"u}nther and Mengzhou Xia and Weijia Shi and Xing Han L{\`u} and Jordan Clive and Gayatri K and Maksimova Anna and Silvan Wehrli and Maria Tikhonova and Henil Shalin Panchal and Aleksandr Abramov and Malte Ostendorff and Zheng Liu and Simon Clematide and Lester James Validad Miranda and Alena Fenogenova and Guangyu Song and Ruqiya Bin Safi and Wen-Ding Li and Alessia Borghini and Federico Cassano and Lasse Hansen and Sara Hooker and Chenghao Xiao and Vaibhav Adlakha and Orion Weller and Siva Reddy and Niklas Muennighoff},
booktitle={The Thirteenth International Conference on Learning Representations},
year={2025},
url={https://openreview.net/forum?id=zl3pfz4VCV}
}

@article{kraskov2004estimating,
  title={Estimating mutual information},
  author={Kraskov, Alexander and St{\"o}gbauer, Harald and Grassberger, Peter},
  journal={Physical Review E—Statistical, Nonlinear, and Soft Matter Physics},
  volume={69},
  number={6},
  pages={066138},
  year={2004},
  publisher={APS}
}

@inproceedings{tripto2025beyond,
  title={Beyond Checkmate: Exploring the Creative Choke Points for AI Generated Texts},
  author={Tripto, Nafis Irtiza and Venkatraman, Saranya and Nahar, Mahjabin and Lee, Dongwon},
  booktitle={Proceedings of the 2025 Conference on Empirical Methods in Natural Language Processing},
  pages={11964--11981},
  year={2025}
}

@article{munoz2024contrasting,
  title={Contrasting linguistic patterns in human and LLM-generated news text},
  author={Mu{\~n}oz-Ortiz, Alberto and G{\'o}mez-Rodr{\'\i}guez, Carlos and Vilares, David},
  journal={Artificial Intelligence Review},
  volume={57},
  number={10},
  pages={265},
  year={2024},
  publisher={Springer}
}

@article{hu2024explaining,
  title={Explaining length bias in llm-based preference evaluations},
  author={Hu, Zhengyu and Song, Linxin and Zhang, Jieyu and Xiao, Zheyuan and Wang, Tianfu and Chen, Zhengyu and Yuan, Nicholas Jing and Lian, Jianxun and Ding, Kaize and Xiong, Hui},
  journal={arXiv preprint arXiv:2407.01085},
  year={2024}
}

@article{alafwan2023comments,
  title={Comments Analysis on Social Media: A Review.},
  author={Alafwan, Brian and Siallagan, Manahan and Putro, Utomo Sarjono},
  journal={EAI Endorsed Transactions on Scalable Information Systems},
  volume={10},
  number={6},
  year={2023}
}

@inproceedings{alsayat2016social,
  title={Social media analysis using optimized K-Means clustering},
  author={Alsayat, Ahmed and El-Sayed, Hoda},
  booktitle={2016 IEEE 14th international conference on software engineering research, management and applications (SERA)},
  pages={61--66},
  year={2016},
  organization={IEEE}
}

@inproceedings{sari2018topic,
  title={Topic or style? exploring the most useful features for authorship attribution},
  author={Sari, Yunita and Stevenson, Mark and Vlachos, Andreas},
  booktitle={Proceedings of the 27th international conference on computational linguistics},
  pages={343--353},
  year={2018}
}

@article{tai2020online,
  title={Online social networks and writing styles--a review of the multidisciplinary literature},
  author={Tai, Kah Yee and Dhaliwal, Jasbir and Shariff, Shafiza Mohd},
  journal={Ieee Access},
  volume={8},
  pages={67024--67046},
  year={2020},
  publisher={IEEE}
}

@article{singh2021pragmatics,
  title={Pragmatics to reveal intent in social media peer interactions: mixed methods study},
  author={Singh, Tavleen and Olivares, Sofia and Cohen, Trevor and Cobb, Nathan and Wang, Jing and Franklin, Amy and Myneni, Sahiti},
  journal={Journal of medical Internet research},
  volume={23},
  number={11},
  pages={e32167},
  year={2021},
  publisher={JMIR Publications Toronto, Canada}
}

@article{hasan2020socio,
  title={A Socio-Pragmatic analysis of social media comments on online learning at the time of COVID 19},
  author={Hasan, Atyaf and Idrees, Fatima Abdul Ghany},
  journal={ELS Journal on Interdisciplinary Studies in Humanities},
  volume={3},
  number={4},
  pages={638--650},
  year={2020}
}

@inproceedings{
padmakumar2024does,
title={Does Writing with Language Models Reduce Content Diversity?},
author={Vishakh Padmakumar and He He},
booktitle={The Twelfth International Conference on Learning Representations},
year={2024},
url={https://openreview.net/forum?id=Feiz5HtCD0}
}

@inproceedings{shaib-etal-2025-standardizing,
    title = "Standardizing the Measurement of Text Diversity: A Tool and Comparative Analysis",
    author = "Shaib, Chantal  and
      Govindarajan, Venkata S  and
      Barrow, Joe  and
      Sun, Jiuding  and
      Siu, Alexa  and
      Wallace, Byron C  and
      Nenkova, Ani",
    editor = "Liu, Xuebo  and
      Purwarianti, Ayu",
    booktitle = "Proceedings of The 14th International Joint Conference on Natural Language Processing and The 4th Conference of the Asia-Pacific Chapter of the Association for Computational Linguistics: System Demonstrations",
    month = dec,
    year = "2025",
    address = "Mumbai, India",
    publisher = "Association for Computational Linguistics",
    url = "https://aclanthology.org/2025.ijcnlp-demo.5/",
    pages = "36--46",
    ISBN = "979-8-89176-301-2",
}

@article{honnibal2020spacy,
  title={spaCy: Industrial-strength natural language processing in python},
  author={Honnibal, Matthew and Montani, Ines and Van Landeghem, Sofie and Boyd, Adriane and others},
  year={2020},
  publisher={Zenodo, Honolulu, HI, USA}
}

@article{hansen2023textdescriptives,
  title={TextDescriptives: A Python package for calculating a large variety of metrics from text},
  author={Hansen, Lasse and Olsen, Ludvig Renbo and Enevoldsen, Kenneth},
  journal={Journal of Open Source Software},
  volume={8},
  number={84},
  pages={5153},
  year={2023}
}

@article{zamzmi2025scorecard_llm,
  title={Scorecard for synthetic medical data evaluation},
  author={Zamzmi, Ghada and Subbaswamy, Adarsh and Sizikova, Elena and Margerrison, Edward and Delfino, Jana G and Badano, Aldo},
  journal={Communications Engineering},
  volume={4},
  number={1},
  pages={130},
  year={2025},
  publisher={Nature Publishing Group UK London}
}

@article{kynkaanniemi2019manifold_precision,
  title={Improved precision and recall metric for assessing generative models},
  author={Kynk{\"a}{\"a}nniemi, Tuomas and Karras, Tero and Laine, Samuli and Lehtinen, Jaakko and Aila, Timo},
  journal={Advances in neural information processing systems},
  volume={32},
  year={2019}
}

@article{simpson1949measurement,
  title={Measurement of diversity},
  author={Simpson, Edward H},
  journal={nature},
  volume={163},
  number={4148},
  pages={688--688},
  year={1949},
  publisher={Nature Publishing Group UK London}
}

@inproceedings{zhu2018texygen,
  title={Texygen: A benchmarking platform for text generation models},
  author={Zhu, Yaoming and Lu, Sidi and Zheng, Lei and Guo, Jiaxian and Zhang, Weinan and Wang, Jun and Yu, Yong},
  booktitle={The 41st international ACM SIGIR conference on research \& development in information retrieval},
  pages={1097--1100},
  year={2018}
}

@inproceedings{salkar2022self_repitiion_score,
  title={Self-repetition in abstractive neural summarizers},
  author={Salkar, Nikita and Trikalinos, Thomas and Wallace, Byron C and Nenkova, Ani},
  booktitle={Proceedings of the conference. Association for Computational Linguistics. Meeting},
  volume={2022},
  pages={341},
  year={2022}
}

@inproceedings{lin2004rouge,
  title={Rouge: A package for automatic evaluation of summaries},
  author={Lin, Chin-Yew},
  booktitle={Text summarization branches out},
  pages={74--81},
  year={2004}
}

@article{herdan1960type,
  title={Type-token mathematics: A textbook of mathematical linguistics},
  author={Herdan, Gustav},
  journal={(No Title)},
  year={1960}
}

@article{kincaid1975derivation,
  title={Derivation of new readability formulas (automated readability index, fog count and flesch reading ease formula) for navy enlisted personnel},
  author={Kincaid, J Peter and Fishburne Jr, Robert P and Rogers, Richard L and Chissom, Brad S},
  year={1975},
  publisher={Institute for Simulation and Training, University of Central Florida}
}

@inproceedings{longpre2024pretrainer,
  title={A pretrainer’s guide to training data: Measuring the effects of data age, domain coverage, quality, \& toxicity},
  author={Longpre, Shayne and Yauney, Gregory and Reif, Emily and Lee, Katherine and Roberts, Adam and Zoph, Barret and Zhou, Denny and Wei, Jason and Robinson, Kevin and Mimno, David and others},
  booktitle={Proceedings of the 2024 Conference of the North American Chapter of the Association for Computational Linguistics: Human Language Technologies (Volume 1: Long Papers)},
  pages={3245--3276},
  year={2024}
}

@inproceedings{
anonymous2024beyond,
title={Beyond Scale: The Diversity Coefficient as a Data Quality Metric for Variability in Natural Language Data},
author={Anonymous},
booktitle={ICLR 2024 Workshop on Data-centric Machine Learning Research (DMLR): Harnessing Momentum for Science},
year={2024},
url={https://openreview.net/forum?id=tgkWxsOapD}
}

@article{tirumala2023d4,
  title={D4: Improving llm pretraining via document de-duplication and diversification},
  author={Tirumala, Kushal and Simig, Daniel and Aghajanyan, Armen and Morcos, Ari},
  journal={Advances in Neural Information Processing Systems},
  volume={36},
  pages={53983--53995},
  year={2023}
}

@article{hoffmann2022training,
  title={Training compute-optimal large language models},
  author={Hoffmann, Jordan and Borgeaud, Sebastian and Mensch, Arthur and Buchatskaya, Elena and Cai, Trevor and Rutherford, Eliza and Casas, DDL and Hendricks, Lisa Anne and Welbl, Johannes and Clark, Aidan and others},
  journal={arXiv preprint arXiv:2203.15556},
  volume={10},
  year={2022}
}

@article{gohari2025gneissweb,
  title={Gneissweb: Preparing high quality data for llms at scale},
  author={Gohari, Hajar Emami and Kadhe, Swanand Ravindra and Shah, Syed Yousaf and Adam, Constantin and Adebayo, Abdulhamid and Adusumilli, Praneet and Ahmed, Farhan and Angel, Nathalie Baracaldo and Borse, Santosh Subhashrao and Chang, Yuan-Chi and others},
  journal={arXiv preprint arXiv:2502.14907},
  year={2025}
}

@article{li2024datacomp,
  title={Datacomp-lm: In search of the next generation of training sets for language models},
  author={Li, Jeffrey and Fang, Alex and Smyrnis, Georgios and Ivgi, Maor and Jordan, Matt and Gadre, Samir and Bansal, Hritik and Guha, Etash and Keh, Sedrick and Arora, Kushal and others},
  journal={Advances in Neural Information Processing Systems},
  volume={37},
  pages={14200--14282},
  year={2024}
}

@article{wang2025ultra,
  title={Ultra-fineweb: Efficient data filtering and verification for high-quality llm training data},
  author={Wang, Yudong and Fu, Zixuan and Cai, Jie and Tang, Peijun and Lyu, Hongya and Fang, Yewei and Zheng, Zhi and Zhou, Jie and Zeng, Guoyang and Xiao, Chaojun and others},
  journal={arXiv preprint arXiv:2505.05427},
  year={2025}
}

@inproceedings{su2025nemotron,
  title={Nemotron-cc: Transforming common crawl into a refined long-horizon pretraining dataset},
  author={Su, Dan and Kong, Kezhi and Lin, Ying and Jennings, Joseph and Norick, Brandon and Kliegl, Markus and Patwary, Mostofa and Shoeybi, Mohammad and Catanzaro, Bryan},
  booktitle={Proceedings of the 63rd Annual Meeting of the Association for Computational Linguistics (Volume 1: Long Papers)},
  pages={2459--2475},
  year={2025}
}

@article{penedo2024fineweb,
  title={The fineweb datasets: Decanting the web for the finest text data at scale},
  author={Penedo, Guilherme and Kydl{\'\i}{\v{c}}ek, Hynek and Lozhkov, Anton and Mitchell, Margaret and Raffel, Colin A and Von Werra, Leandro and Wolf, Thomas and others},
  journal={Advances in Neural Information Processing Systems},
  volume={37},
  pages={30811--30849},
  year={2024}
}

@inproceedings{abbas2023semdedup,
  title={SemDeDup: Data-efficient learning at web-scale through semantic deduplication},
  author={Abbas, Amro Kamal Mohamed and Tirumala, Kushal and Simig, Daniel and Ganguli, Surya and Morcos, Ari S},
  booktitle={ICLR 2023 Workshop on Mathematical and Empirical Understanding of Foundation Models},
    year={2023}
}

@inproceedings{lee-etal-2022-deduplicating,
    title = "Deduplicating Training Data Makes Language Models Better",
    author = "Lee, Katherine  and
      Ippolito, Daphne  and
      Nystrom, Andrew  and
      Zhang, Chiyuan  and
      Eck, Douglas  and
      Callison-Burch, Chris  and
      Carlini, Nicholas",
    editor = "Muresan, Smaranda  and
      Nakov, Preslav  and
      Villavicencio, Aline",
    booktitle = "Proceedings of the 60th Annual Meeting of the Association for Computational Linguistics (Volume 1: Long Papers)",
    month = may,
    year = "2022",
    address = "Dublin, Ireland",
    publisher = "Association for Computational Linguistics",
    url = "https://aclanthology.org/2022.acl-long.577/",
    doi = "10.18653/v1/2022.acl-long.577",
    pages = "8424--8445",
   
}

@inproceedings{anknerperplexed,
  title={Perplexed by Perplexity: Perplexity-Based Data Pruning With Small Reference Models},
  author={Ankner, Zachary and Blakeney, Cody and Sreenivasan, Kartik and Marion, Max and Leavitt, Matthew L and Paul, Mansheej},
  booktitle={The Thirteenth International Conference on Learning Representations},
  year = {2024}
}

@article{ekman2014expression,
  title={Expression and the nature of emotion},
  author={Ekman, Paul},
  journal={Approaches to emotion},
  pages={319--343},
  year={2014},
  publisher={Psychology Press}
}

@inproceedings{sun2019fine,
  title={How to fine-tune bert for text classification?},
  author={Sun, Chi and Qiu, Xipeng and Xu, Yige and Huang, Xuanjing},
  booktitle={China national conference on Chinese computational linguistics},
  pages={194--206},
  year={2019},
  organization={Springer}
}

@article{hosseinmardi2025unpacking,
  title={Unpacking media bias in the growing divide between cable and network news},
  author={Hosseinmardi, Homa and Wolken, Samuel and Rothschild, David M and Watts, Duncan J},
  journal={Scientific Reports},
  volume={15},
  number={1},
  pages={17607},
  year={2025},
  publisher={Nature Publishing Group UK London}
}

@book{mays2020social,
  title={Social media comments to YouTube videos by cnn and fox news viewers on 2017 legislative efforts to repeal and replace the affordable care act},
  author={Mays, Genesa L},
  year={2020},
  publisher={Central Michigan University}
}

@article{martin2014bias,
  title={Bias in cable news: Real effects and polarization},
  author={Martin, Gregory J and Yurukoglu, Ali},
  year={2014},
  publisher={National Bureau of Economic Research Cambridge, MA}
}

@inproceedings{tan2024large,
  title={Large language models for data annotation and synthesis: A survey},
  author={Tan, Zhen and Li, Dawei and Wang, Song and Beigi, Alimohammad and Jiang, Bohan and Bhattacharjee, Amrita and Karami, Mansooreh and Li, Jundong and Cheng, Lu and Liu, Huan},
  booktitle={Proceedings of the 2024 Conference on Empirical Methods in Natural Language Processing},
  pages={930--957},
  year={2024}
}

@inproceedings{li2023synthetic,
  title={Synthetic data generation with large language models for text classification: Potential and limitations},
  author={Li, Zhuoyan and Zhu, Hangxiao and Lu, Zhuoran and Yin, Ming},
  booktitle={Proceedings of the 2023 conference on empirical methods in natural language processing},
  pages={10443--10461},
  year={2023}
}

@article{kumpel2015user,
  title={HOW USER COMMENTS ON A NEWS SITE AFFECT PERCEPTIONS OF JOURNALISTIC QUALITY. AN EXPERIMENTAL STUDY USING STRUCTURAL EQUATION MODELING},
  author={Kumpel, Anna Sophie and Springer, Nina},
  journal={AoIR Selected Papers of Internet Research},
  year={2015}
}

@article{fleeson2001toward,
  title={Toward a structure-and process-integrated view of personality: Traits as density distributions of states.},
  author={Fleeson, William},
  journal={Journal of personality and social psychology},
  volume={80},
  number={6},
  pages={1011},
  year={2001},
  publisher={American Psychological Association}
}

@article{roberts2004traits,
  title={On traits, situations, and their integration: A developmental perspective},
  author={Roberts, Brent W and Pomerantz, Eva M},
  journal={Personality and Social Psychology Review},
  volume={8},
  number={4},
  pages={402--416},
  year={2004},
  publisher={Sage Publications Sage CA: Los Angeles, CA}
}

@article{marwick2011tweet,
  title={I tweet honestly, I tweet passionately: Twitter users, context collapse, and the imagined audience},
  author={Marwick, Alice E and Boyd, Danah},
  journal={New media \& society},
  volume={13},
  number={1},
  pages={114--133},
  year={2011},
  publisher={Sage Publications Sage UK: London, England}
}

@inproceedings{hu2024quantifying,
  title={Quantifying the persona effect in LLM simulations},
  author={Hu, Tiancheng and Collier, Nigel},
  booktitle={Proceedings of the 62nd Annual Meeting of the Association for Computational Linguistics (Volume 1: Long Papers)},
  pages={10289--10307},
  year={2024}
}

@inproceedings{
li2025llm,
title={{LLM} Generated Persona is a Promise with a Catch},
author={Ang Li and Haozhe Chen and Hongseok Namkoong and Tianyi Peng},
booktitle={First Workshop on Social Simulation with LLMs},
year={2025},
url={https://openreview.net/forum?id=mHn4n9YnkA}
}

@inproceedings{joulin-etal-2017-bag,
    title = "Bag of Tricks for Efficient Text Classification",
    author = "Joulin, Armand  and
      Grave, Edouard  and
      Bojanowski, Piotr  and
      Mikolov, Tomas",
    editor = "Lapata, Mirella  and
      Blunsom, Phil  and
      Koller, Alexander",
    booktitle = "Proceedings of the 15th Conference of the {E}uropean Chapter of the Association for Computational Linguistics: Volume 2, Short Papers",
    month = apr,
    year = "2017",
    address = "Valencia, Spain",
    publisher = "Association for Computational Linguistics",
    url = "https://aclanthology.org/E17-2068/",
    pages = "427--431"
}

@article{miller1956magical,
 title={The magical number seven, plus or minus two: Some limits on our capacity for processing information.},
 author={Miller, George A},
 journal={Psychological review},
 volume={63},
 number={2},
 pages={81},
 year={1956},
 publisher={American Psychological Association}
}

@inproceedings{khashabi-etal-2022-genie,
    title = "{GENIE}: Toward Reproducible and Standardized Human Evaluation for Text Generation",
    author = "Khashabi, Daniel  and
      Stanovsky, Gabriel  and
      Bragg, Jonathan  and
      Lourie, Nicholas  and
      Kasai, Jungo  and
      Choi, Yejin  and
      Smith, Noah A.  and
      Weld, Daniel",
    editor = "Goldberg, Yoav  and
      Kozareva, Zornitsa  and
      Zhang, Yue",
    booktitle = "Proceedings of the 2022 Conference on Empirical Methods in Natural Language Processing",
    month = dec,
    year = "2022",
    address = "Abu Dhabi, United Arab Emirates",
    publisher = "Association for Computational Linguistics",
    url = "https://aclanthology.org/2022.emnlp-main.787/",
    doi = "10.18653/v1/2022.emnlp-main.787",
    pages = "11444--11458"
}

@article{van2021human,
  title={Human evaluation of automatically generated text: Current trends and best practice guidelines},
  author={Van der Lee, Chris and Gatt, Albert and Van Miltenburg, Emiel and Krahmer, Emiel},
  journal={Computer Speech \& Language},
  volume={67},
  pages={101151},
  year={2021},
  publisher={Elsevier}
}

@inproceedings{han-etal-2022-measuring,
    title = "Measuring and Improving Semantic Diversity of Dialogue Generation",
    author = "Han, Seungju  and
      Kim, Beomsu  and
      Chang, Buru",
    editor = "Goldberg, Yoav  and
      Kozareva, Zornitsa  and
      Zhang, Yue",
    booktitle = "Findings of the Association for Computational Linguistics: EMNLP 2022",
    month = dec,
    year = "2022",
    address = "Abu Dhabi, United Arab Emirates",
    publisher = "Association for Computational Linguistics",
    url = "https://aclanthology.org/2022.findings-emnlp.66/",
    doi = "10.18653/v1/2022.findings-emnlp.66",
    pages = "934--950"
}

\newpage
\appendix
\startcontents[apx]
\hypersetup{linkcolor=blue}
\begin{center}
  {\Large\bfseries Appendix Contents}
  \vspace{0.5cm}
\end{center}
\setcounter{tocdepth}{2}
\printcontents[apx]{}{1}{}  

\newpage
\section{Broader statements}
\subsection{Limitations}
\label{subsec_limitation}

While our approach aims to approximate the latent factors that shape human comments through aspect-based generation, capturing both \textbf{topical perspectives} and \textbf{socio-pragmatic characteristics}, it does not fully account for the many subtle \& implicit factors that influence how people respond online. Human commenting behavior is shaped by complex, context-dependent cues, including personal experiences, social relationships, and evolving intentions \citep{nasim2013commenting, stavros2014understanding, garcia2017understanding, schultes2013leave, marwick2011tweet}. While persona-based generation offers an alternative, modeling human personas is challenging because human behavior is shaped not only by stable traits but also by situational context \& audience effects. Prior work in personality psychology shows substantial within-person variability across situations \citep{fleeson2001toward, roberts2004traits, marwick2011tweet}, and recent studies of LLM persona simulation further highlight limitations and biases in persona-based prompting \citep{hu2024quantifying, li2025llm}. We therefore use aspect-based generation as a more tractable abstraction. We interpret a combination of multiple LLMs (in their base form, without explicit persona conditioning) together with aspect guidance as an approximation of a group of individuals, each contributing specific perspectives or communicative styles to a shared discourse. This allows us to capture diverse perspectives without relying on potentially unrealistic assumptions about user identity.

Our goal is not to fully replicate social systems, but to provide a controlled and scalable approximation of discourse diversity. Thus, our study makes specific design choices to isolate and analyze comment generation, which introduces some natural limitations. \textbf{First}, we chose YouTube, where comments are primarily driven by shared content rather than pre-existing social relationships or network effects. While this allows us to model $P(\text{comment} \mid \text{context})$ in a controlled way, it abstracts away important factors present in other platforms, such as social ties, community dynamics, and coordinated behavior. Although our framework can be extended to other online comment spaces, the form of human diversity may differ substantially across platforms with stronger community norms, social graphs, or political dynamics. Moreover, although YouTube is inherently a video-based platform, our study focuses on the textual comment space, treating comments as responses to the underlying content rather than the full multimodal experience. We mitigate this by selecting domains that are less dependent on visual cues, restricting the time frame to avoid contamination from recent events or AI-generated content, and filtering comments that rely heavily on video-specific references. However, fully capturing the multimodal and socially embedded nature of real-world interactions remains beyond the scope of a single study. \textbf{Second}, our analysis depends on approximating the human comment distribution, which is inherently noisy and influenced by factors such as temporal dynamics, long-tail engagement, and the presence of toxic or extreme content. In addition, our use of top-ranked comments captures the most visible and high-engagement portion of discourse, but it may not reflect the full long-tail spectrum of human comments. While we apply filtering to reduce noise and ensure comparability with possible LLM generations, these steps may shift the distribution away from its raw form. At the same time, due to alignment constraints, LLMs are less likely to generate toxic comments, introducing an additional mismatch.

Our study is also subject to natural \textbf{scale–depth trade-offs}. To enable a large-scale analysis, we focus on first-level comments and do not model reply chains or nested interactions, which would require more targeted, domain-specific investigation and potentially human subject studies. We also restrict our analysis to English comments and emphasize broadly applicable feature spaces: semantic, linguistic, and socio-pragmatic, while acknowledging that additional dynamic or video-specific factors may further influence discourse. Exploring these dimensions in depth is important but challenging to address within a single study at this scale. We also rely on a set of automatic metrics as proxies for discourse diversity. Although our evaluation spans multiple axes and feature spaces, no metric suite can fully capture the richness of human interpretation, humor, irony, or context-specific meaning. Despite these limitations, our framework offers a \textbf{controlled, scalable, and comprehensive} foundation for studying social discourse.

Finally, our study is subject to practical limitations innate to working with modern LLMs and large-scale evaluation. Although we include a diverse set of 35 models across nine providers, model outputs are inherently non-deterministic, and the landscape evolves rapidly as newer models emerge. Even with fixed prompts and decoding parameters, repeated generations may vary, which is a standard but important limitation of LLM evaluation. Results may also be sensitive to prompt design, although we use consistent templates across settings to control this factor. We also do not perform full pre-training experiments due to the prohibitive scale required. Instead, we adopt data curation pipelines as a proxy for assessing pre-training suitability. While this abstraction has limitations, we complement it with fine-tuning and instruction-tuning case studies, which are more realistic scenarios for utilizing such data. Our findings also highlight that social media content, by nature, does not align perfectly with traditional notions of “high-quality” training data, as indicated by lower scores from standard quality classifiers. Although additional downstream tasks may yield different outcomes, our study systematically explores a wide range of design choices and evaluation perspectives. \textbf{Overall, despite these inherent constraints, our work provides a comprehensive and well-grounded foundation for understanding, evaluating, and utilizing synthetic social discourse at scale.}

\subsection{Future directions}
\label{subsec_future_direction}
Our study opens several promising directions for advancing the understanding of diversity in social discourse. A natural next step is a depth-focused analysis on a smaller set of videos, where richer structures such as nested replies and conversational dynamics can be modeled. Incorporating multimodal signals through vision-language models would further align generation with the full video context. It will also enable the exploration of more dynamic, discourse-specific features (e.g., purchase intent in product review videos) and support human subject studies to evaluate engagement and nuanced perceptions of quality beyond surface-level metrics.

Another important direction is longitudinal analysis, where LLM-based agents are simulated over time using existing social agent frameworks. Such setups would allow us to study how diversity evolves in more realistic, temporally grounded interactions, including persona-driven behavior and community dynamics.

Finally, from a practical perspective, our framework can be extended to real-world integration, such as studying the effects of incorporating synthetic social media data in continual pre-training pipelines and evaluating potential model degradation in deployed systems. Together, these directions build on our work as a foundation for more realistic, scalable, and socially grounded studies of LLM-generated discourse.

\subsection{Broader implications}
\label{subsec_broader_implication}
\paragraph{Societal implications.} AI-generated content is becoming increasingly prevalent across online platforms. Also, with the emergence of fully AI-driven social environments \citep{zhang2025socioverse, yang2024oasis}, it is crucial to understand whether LLM-generated discourse can match the diversity of human communication. Our findings suggest that even with model swarms and aspect-conditioned generation, \textbf{human diversity remains difficult to replicate}. It has dual implications: on one hand, it reassures that human expression retains a unique richness. Alternatively, it raises concerns that widespread deployment of homogenized AI-generated content could gradually reduce diversity in online spaces, potentially leading to a less vibrant and more uniform digital ecosystem.

\paragraph{Methodological implications.} Our study contributes to the existing literature on LLM diversity evaluation by showing that outputs from the same model and even from models within the same provider tend to be more similar, highlighting the importance of cross-model and cross-provider heterogeneity. We introduce a comprehensive framework that characterizes diversity across dispersion, coverage, and alignment, providing a more nuanced lens than prior single-metric approaches. Our aspect-based formulation offers a practical abstraction for modeling discourse variation. Finally, our study also contributes to the broader landscape of human \& AI text distinction. Our large-scale dataset of human and LLM-generated comments, spanning multiple domains and styles, can serve as a valuable benchmark for future work, including research on AI text detection and human–AI distinction.

\paragraph{Data quality implications.} Our results also shed light on the evolving role of synthetic data in language model training. While multi-LLM and aspect-conditioned generation improve diversity and increase retention under standard data curation pipelines, this does not directly translate to higher document quality as defined by existing filters. In practice, these filters tend to favor structured, knowledge-rich text, whereas authentic social media discourse, human or synthetic, often appears less “high-quality” under such criteria. It reveals a fundamental tension: diversity helps data survive curation, yet social discourse is not aligned with traditional notions of quality. At the same time, LLM-generated comments often introduce more structured, knowledge-like patterns, distinguishing them from raw human data. Together, these findings highlight a mismatch between current data curation standards and the nature of social media content, suggesting the need to \textbf{revisit quality definitions when leveraging diverse, discourse-driven data for future language model training}.

\subsection{Ethical considerations}
\label{subsec_ethical}
Our study raises several ethical considerations associated with the generation and use of synthetic social discourse. We use publicly available YouTube comments as our human data source. We take care to work within platform-visible content without targeting or identifying individuals. At the same time, to realistically approximate social discourse, our generation framework allows for a range of socio-pragmatic expressions, which may include bias, stereotypes, or potentially harmful language. While it reflects the nature of real-world comment spaces, it also introduces the risk that such patterns could be amplified if such synthetic data is used without proper safeguards in future model training.

More broadly, our findings highlight a dual-use concern. The same techniques that improve diversity and realism, such as combining multiple models and aspect-guided generation, can also be misused to produce convincing, human-like comment threads at scale. This could enable malicious actors to artificially inflate engagement, shape public opinion, or support misinformation and disinformation campaigns. In particular, our results on unseen data suggest that such generation can generalize without access to existing human comments, further lowering the barrier for misuse.

Despite these risks, studying and understanding these capabilities is necessary. As AI-generated content becomes increasingly integrated into online platforms, it is important to proactively analyze both its benefits and potential harms. Our work aims to provide a transparent and systematic foundation for evaluating synthetic discourse, which can inform the development of safer generation practices, improved detection systems, and more responsible integration of AI in social environments.

\subsection{Use of Large Language Models}
\label{subsec_llm_use}
We disclose the use of large language models (LLMs) in supporting roles throughout this work. LLMs were used for writing assistance, including improving clarity, grammar, and overall presentation of the manuscript, as well as drafting portions of text. We also used LLM-based tools to assist with targeted literature search, understanding technical concepts, and generating initial versions of experimental code, visualizations, and \LaTeX-formatted derivations, including support in formalizing and writing proof steps. All such outputs were carefully reviewed, verified, and refined by the authors before use.

Importantly, all core ideas, experimental design, analysis, and conclusions originate from the authors. LLM-generated content was treated as an initial aid rather than a final output, requiring substantial human oversight and validation. The authors take full responsibility for the integrity, correctness, and originality of all aspects of this work, including any components that were initially supported by LLM assistance.


\section{Theoretical analysis details}
\label{appendix_sec_theory}
In this section, we provide the detailed proof originating from Assumption \ref{ass:separation}, regarding the Multi-LLM and Multi-LLM Multi-Aspect diversity gains (following from Section \ref{sec_mathematical_proof}). Table \ref{tab:notation} shows the notations used in the formulation, definition, and proof throughout the paper.

\textbf{Notation:}
For any two distributions $P,Q$, we write
$
    \delta_\kappa(P \| Q)
    \;\equiv\;
    \mathbb{E}_{C\sim P,\,C'\sim Q}\!\left[\kappa(C,C')\right]$
for the \emph{cross-distribution diversity}.
Note $\delta_\kappa(P\|P)=\delta_\kappa(P)$.

\newpage

\begin{center}
\renewcommand{\arraystretch}{1.5}
\begin{longtable}{p{0.26\textwidth} p{0.66\textwidth}}

\caption{Summary of notation used throughout the paper.}
\label{tab:notation} \\

\toprule
\textbf{Symbol} & \textbf{Description} \\
\midrule
\endfirsthead

\multicolumn{2}{l}{\small\itshape (Table~\ref{tab:notation} continued)} \\[2pt]
\toprule
\textbf{Symbol} & \textbf{Description} \\
\midrule
\endhead

\midrule
\multicolumn{2}{r}{\small\itshape Continued on next page} \\
\endfoot

\bottomrule
\endlastfoot

\multicolumn{2}{l}{\textit{Content and Comment Space}} \\[2pt]

$v$
    & A content item (video). \\
$\mathcal{C}$
    & The space of all possible comments. \\
$C \in \mathcal{C}$
    & A comment, modelled as a random variable. \\
$\phi: \mathcal{C} \to \mathbb{R}^d$
    & Fixed embedding function mapping comments into a
      $d$-dimensional semantic space. \\
$H = \{h_1, \dots, h_n\}$
    & A set of $n$ human-generated comments for video $v$. \\

\midrule
\multicolumn{2}{l}{\textit{Models and Aspects}} \\[2pt]

$M_k$
    & The $k$-th language model, $k \in \{1, \dots, K\}$. \\
$K$
    & Total number of language models. \\
$\mathcal{A} = \{a_1, \dots, a_A\}$
    & Finite set of $A$ discourse aspects
      (e.g., \emph{topic}, \emph{emotion}, \emph{humor}). \\
$A = |\mathcal{A}|$
    & Number of discourse aspects; $K > A$ in our setting. \\
$\sigma: \{1,\dots,K\} \to \mathcal{A}$
    & Surjective assignment mapping each model $M_k$ to exactly
      one aspect $\sigma(k)$. \\
$\mathcal{K}_a = \{k : \sigma(k) = a\}$
    & Set of models assigned to aspect $a$;
      $\{\mathcal{K}_a\}_{a \in \mathcal{A}}$ partitions
      $\{1,\dots,K\}$. \\

\midrule
\multicolumn{2}{l}{\textit{Underlying Model Distributions}} \\[2pt]

$P_H(\cdot \mid v)$
    & Unknown reference distribution of human comments for
      video $v$. \\
$P_{M_k}(\cdot \mid v)$
    & Comment distribution induced by model $M_k$ given video
      $v$, without aspect conditioning. \\
$P_{M_k}(\cdot \mid v, a)$
    & Comment distribution induced by model $M_k$ given video
      $v$ and aspect $a$. \\

\midrule
\multicolumn{2}{l}{\textit{Mixture Weights}} \\[2pt]

$\boldsymbol{\pi} = (\pi_1, \dots, \pi_K)$
    & Per-model mixture weights; $\pi_k \ge 0$,
      $\sum_{k}\pi_k = 1$. \\
$\boldsymbol{\lambda} = (\lambda_1, \dots, \lambda_A)$
    & Per-aspect coverage weights;
      $\lambda_a = \sum_{k \in \mathcal{K}_a}\omega_k \ge 0$,
      $\sum_a \lambda_a = 1$. \\
$\omega_k$
    & Per-model weight in the \textbf{M\textsubscript{A}}
      setting; $\omega_k \ge 0$, $\sum_k \omega_k = 1$.
      Uniform case: $\omega_k = 1/K$. \\
$\omega_{k,a} = \omega_k \cdot \mathbf{1}[\sigma(k)=a]$
    & Joint model--aspect weight; non-zero only for the aspect
      assigned to $M_k$. \\

\midrule
\multicolumn{2}{l}{\textit{Generation Settings}} \\[2pt]

$\mathcal{G}^{(\cdot)}(\cdot \mid v)$
    & Output distribution of a generation setting given
      video $v$. \\
$\mathcal{G}^{(\mathrm{S})}(\cdot \mid v)$
    & Single-LLM output:
      $P_{M_k}(\cdot\mid v)$. \\
$\mathcal{G}^{(\mathrm{S_A})}(\cdot \mid v)$
    & Single-LLM, aspect-conditioned output:
      $\sum_a \lambda_a\,P_{M_k}(\cdot\mid v,a)$. \\
$\mathcal{G}^{(\mathrm{M})}(\cdot \mid v)$
    & Multi-LLM output:
      $\sum_k \pi_k\, P_{M_k}(\cdot\mid v)$. \\
$\mathcal{G}^{(\mathrm{M_A})}(\cdot \mid v)$
    & Multi-LLM, aspect-conditioned output:
      $\sum_{a}\sum_{k\in\mathcal{K}_a}
       \omega_k\,P_{M_k}(\cdot\mid v,a)$. \\

\midrule
\multicolumn{2}{l}{\textit{Diversity Functional}} \\[2pt]

$\kappa: \mathcal{C}\times\mathcal{C} \to [0,1]$
    & Symmetric, bounded dissimilarity kernel. \\
$\kappa_{\cos}(c, c')$
    & Clipped cosine dissimilarity:
      $1 - \max(0, \cos(\phi(c), \phi(c')))$. \\
$\delta_\kappa(P)$
    & Expected diversity of distribution $P$ under kernel
      $\kappa$:
      $\mathbb{E}_{C,C' \overset{\mathrm{iid}}{\sim} P}
      [\kappa(C,C')]$. \\
$\delta(\cdot)$
    & Shorthand for $\delta_{\kappa_{\cos}}(\cdot)$. \\
$\delta_\kappa(P \| Q)$
    & Cross-distribution diversity:
      $\mathbb{E}_{C\sim P,\,C'\sim Q}[\kappa(C,C')]$. \\
$\widehat\delta_\kappa$
    & Empirical U-statistic estimator of $\delta_\kappa(P)$:
      $\frac{1}{n(n-1)}\sum_{i\neq j}\kappa(c_i,c_j)$. \\

\midrule
\multicolumn{2}{l}{\textit{Model Separation
    (Assumption~\ref{ass:separation})}} \\[2pt]

$\bar\rho$
    & Upper bound on cross-model expected cosine similarity. \\
$\epsilon$
    & Separation gap: within-model similarity exceeds
      cross-model similarity by at least $\epsilon > 0$. \\
$\bar\rho_\mathcal{A},\; \epsilon_\mathcal{A}$
    & Analogous constants for cross-aspect separation within
      a single model. \\

\end{longtable}
\end{center}

\begin{proposition}[Mixture Decomposition]
\label{prop:decomp}
Let $P = \sum_{i=1}^N w_i P_i$ with $w_i\ge 0$, $\sum_i w_i = 1$.  Then
\begin{equation}
    \delta_\kappa(P)
    \;=\;
    \sum_{i=1}^N\sum_{j=1}^N w_i w_j\,\delta_\kappa(P_i \| P_j).
    \label{eq:decomp}
\end{equation}
\end{proposition}

\begin{proof}
We introduce independent latent indices $Z, Z'$ with $\Pr(Z=i)=w_i$. Now, 
sampling $C\sim P$ is equivalent to sampling $Z$ then $C\mid Z=i\sim P_i$;
likewise for $C'\mid Z'$.  Since $Z\perp Z'$:
\[
    \delta_\kappa(P)
    \;=\; \mathbb{E}_{Z,Z'}\!\Bigl[\mathbb{E}_{C\sim P_Z,\,C'\sim P_{Z'}}
          \bigl[\kappa(C,C')\bigr]\Bigr]
    \;=\; \sum_{i,j} w_i w_j\,\delta_\kappa(P_i\|P_j). \qedhere
\]
\end{proof}

\begin{theorem}[Multi-LLM diversity gain over the mixture average]
\label{thm:multi_llm}
Under Assumption \ref{ass:separation} with constants \((\bar{\rho}, \epsilon)\), the diversity of the multi-LLM mixture satisfies
\[
\delta(P_{\mathrm{mix}})
\geq
\sum_{k=1}^{K} \pi_k \delta(P_{M_k})
+
\left(1-\sum_{k=1}^{K}\pi_k^2\right)\epsilon .
\tag{2}
\]
Thus, for any non-degenerate mixture, multi-LLM generation improves over the
mixture-weighted average diversity of its component models by an amount
controlled by the separation gap \(\epsilon\) and the effective diversity of the
mixture weights.
\end{theorem}
 \begin{proof}
By Proposition~\ref{prop:decomp},
\[
\delta(P_{\mathrm{mix}})
=
\sum_{k=1}^{K}\pi_k^2 \delta(P_{M_k})
+
\sum_{k\neq \ell}\pi_k\pi_\ell \delta(P_{M_k}\Vert P_{M_\ell}) .
\]
For cross-model terms, Assumption~1(i) gives
\[
\mathbb{E}_{C\sim P_{M_k}, C'\sim P_{M_\ell}}
\left[\max(0,\cos(\phi(C),\phi(C')))\right]
\leq \bar{\rho},
\quad k\neq \ell,
\]
and therefore
\[
\delta(P_{M_k}\Vert P_{M_\ell}) \geq 1-\bar{\rho}.
\]
Assumption~1(ii) gives, for each model \(M_k\),
\[
\mathbb{E}_{C,C'\sim P_{M_k}}
\left[\max(0,\cos(\phi(C),\phi(C')))\right]
\geq \bar{\rho}+\epsilon,
\]
which implies
\[
\delta(P_{M_k}) \leq 1-(\bar{\rho}+\epsilon).
\]
Hence,
\[
1-\bar{\rho} \geq \delta(P_{M_k})+\epsilon
\quad \text{for all } k.
\]
Using this lower bound for the cross-model terms, we obtain
\[
\begin{aligned}
\delta(P_{\mathrm{mix}})
&\geq
\sum_{k=1}^{K}\pi_k^2 \delta(P_{M_k})
+
\sum_{k\neq \ell}\pi_k\pi_\ell
\left(\delta(P_{M_k})+\epsilon\right) \\
&=
\sum_{k=1}^{K}\pi_k^2 \delta(P_{M_k})
+
\sum_{k\neq \ell}\pi_k\pi_\ell \delta(P_{M_k})
+
\epsilon\sum_{k\neq \ell}\pi_k\pi_\ell \\
&=
\sum_{k=1}^{K}\pi_k
\left(\pi_k+\sum_{\ell\neq k}\pi_\ell\right)
\delta(P_{M_k})
+
\epsilon\left(1-\sum_{k=1}^{K}\pi_k^2\right) \\
&=
\sum_{k=1}^{K}\pi_k \delta(P_{M_k})
+
\left(1-\sum_{k=1}^{K}\pi_k^2\right)\epsilon .
\end{aligned}
\]
This completes the proof. \qedhere
\end{proof}

\paragraph{Remark 1.}
The term \(1-\sum_k \pi_k^2\) is the complement of the Herfindahl--Hirschman
index of the mixture weights and reflects the effective number of models in
the mixture. Under uniform weights, it equals \(1-1/K\), increasing with the
number of models. The theorem therefore shows that diversity gains are larger
when generation is distributed across more distinct models, provided the model
separation assumption holds. Importantly, this is a gain over the
mixture-weighted average component diversity, not necessarily over the single
best component model.

\begin{theorem}[Diversity gain from model--aspect mixtures]
\label{thm:full}
Suppose Assumption~\ref{ass:separation} holds with constants $(\bar\rho,\epsilon)$.
Further assume:
\begin{enumerate}[label=(\roman*)]
    \item \emph{(Aspect separation)} For each model $M_k$ and all aspect pairs
    $a\neq a'$:
    \[
        \mathbb{E}_{C\sim P_{M_k,a},\,C'\sim P_{M_k,a'}}
        \!\bigl[\max(0,\cos(\phi(C),\phi(C')))\bigr]
        \;\le\; \bar\rho_\mathcal{A},
    \]
    with within-aspect similarity $\ge \bar\rho_\mathcal{A}+\epsilon_\mathcal{A}$
    for constants $\bar\rho_\mathcal{A}\in[0,1)$, $\epsilon_\mathcal{A}>0$.

    \item \emph{(Baseline dominance)} Each aspect-conditioned component satisfies
    $\delta(P_{M_k,a})\ge\max_k\delta(P_{M_k})$ for all $k,a$.
\end{enumerate}
Then:
\begin{equation}
    \delta(P_{\mathrm{multi}})
    \;\ge\;
    \max_k\,\delta(P_{M_k})
    \;+\;
    \left(1-\sum_{k,a}\omega_{k,a}^2\right)
    \min(\epsilon,\,\epsilon_\mathcal{A}),
    \label{eq:dominance}
\end{equation}
which implies $\delta(P_{\mathrm{multi}})\ge\delta(P_{\mathrm{mix}})\ge\max_k\delta(P_{M_k})$.
\end{theorem}

\begin{proof}
$P_{\mathrm{multi}}$ in \eqref{eq:pmulti} is a mixture over $N=K\!\times\!A$
components $\{P_{M_k,a}\}$ with weights $\{\omega_{k,a}\}$.
Now, we apply Proposition~\ref{prop:decomp}:
\[
    \delta(P_{\mathrm{multi}})
    = \underbrace{\sum_{k,a}\omega_{k,a}^2\,\delta(P_{M_k,a})}_{\text{diagonal}}
    + \underbrace{\sum_{(k,a)\neq(k',a')}\omega_{k,a}\omega_{k',a'}\,
      \delta(P_{M_k,a}\|P_{M_{k'},a'})}_{\text{off-diagonal}}.
\]
\textbf{Diagonal terms.}
By assumption (ii), $\delta(P_{M_k,a})\ge\max_k\delta(P_{M_k})\equiv\Delta^\star$
for all $(k,a)$.  Thus:
\[
    \sum_{k,a}\omega_{k,a}^2\,\delta(P_{M_k,a})
    \;\ge\; \Delta^\star\sum_{k,a}\omega_{k,a}^2.
\]
\textbf{Off-diagonal terms.}
For off-diagonal pairs $(k,a)\neq(k',a')$, two cases arise:
\begin{itemize}
    \item \emph{Different models} ($k\neq k'$, any $a,a'$): by
    Assumption~\ref{ass:separation}, the cross-component similarity is
    $\le\bar\rho$, so $\delta(P_{M_k,a}\|P_{M_{k'},a'})\ge 1-\bar\rho\ge\Delta^\star+\epsilon$.
    \item \emph{Same model, different aspects} ($k=k'$, $a\neq a'$): by
    assumption (i), the cross-aspect similarity is $\le\bar\rho_\mathcal{A}$,
    so $\delta(P_{M_k,a}\|P_{M_k,a'})\ge\Delta^\star+\epsilon_\mathcal{A}$.
\end{itemize}
In both cases $\delta(P_{M_k,a}\|P_{M_{k'},a'})\ge\Delta^\star+\min(\epsilon,\epsilon_\mathcal{A})$.
Therefore:
\begin{align*}
    \delta(P_{\mathrm{multi}})
    &\;\ge\; \Delta^\star\sum_{k,a}\omega_{k,a}^2
         + \bigl(\Delta^\star + \min(\epsilon,\epsilon_\mathcal{A})\bigr)
           \!\sum_{(k,a)\neq(k',a')}\omega_{k,a}\omega_{k',a'} \\
    &\;=\; \Delta^\star\!\underbrace{\sum_{k,a}\sum_{k',a'}\omega_{k,a}\omega_{k',a'}}_{=1}
         + \min(\epsilon,\epsilon_\mathcal{A})
           \underbrace{\sum_{(k,a)\neq(k',a')}\omega_{k,a}\omega_{k',a'}}_{=1-\sum_{k,a}\omega_{k,a}^2} \\
    &\;=\; \Delta^\star
         + \left(1-\sum_{k,a}\omega_{k,a}^2\right)\min(\epsilon,\epsilon_\mathcal{A}). \qedhere
\end{align*}
\end{proof}


Together, Theorems~\ref{thm:multi_llm}–\ref{thm:full} show that mixture-based generation increases expected
diversity relative to the mixture-weighted average diversity of its component
distributions, with gains controlled by model and aspect separation. In this
sense, model swarms and aspect conditioning provide structural mechanisms for
increasing diversity beyond what is obtained by repeatedly sampling from a
single fixed component. These theoretical results motivate, rather than replace,
our empirical evaluation across semantic, linguistic, and socio-pragmatic
feature spaces.

\section{Data, generation, and feature extraction details}

\subsection{Platform and domain selection}
\label{subsec_domain_rationale}
In designing our data and experimental setup, we carefully consider the potential trade-offs involved in studying real-world social discourse. Our goal is to balance ecological validity with experimental control, prioritizing a setting that is both representative of real-world discourse and tractable for systematic evaluation.

\paragraph{Why YouTube as a platform?} 
Human comments are influenced by multiple factors, including both the underlying content and social relationships. In this work, our goal is to approximate the distribution of comments using LLMs while minimizing confounding variables that are difficult to model. Although recent advances in AI-based social simulators attempt to capture these dynamics, faithfully replicating human relationships requires longitudinal modeling and remains a separate and complex research problem. To isolate the effect of content, we focus on a setting where comments are primarily driven by the shared discourse itself.

This motivates our choice of platform. On Reddit, meaningful participation often depends on community knowledge and prior engagement \citep{thukral2018analyzing, yu2024characterizing, ocal2021reasoning}. Platforms such as Facebook are strongly shaped by social interactions and personal networks \citep{nasim2013commenting, stavros2014understanding, lapides2015news, maiz2016factors}, while Twitter is influenced by large-scale social dynamics, including coordinated behavior and follower relationships \citep{garcia2017understanding, keller2020political, khaund2021social, naaman2011hip, nutakki2025there}. In contrast, YouTube provides relatively more content-driven compared to other platforms where users can engage directly with a video without requiring prior social context  \citep{arthurs2018researching, ma2022m, schultes2013leave}. Additionally, YouTube offers large-scale publicly accessible data suitable for research. For this reason, we focus on first-level comments and do not consider reply chains, allowing us to approximate a setting where comments are primarily conditioned on content. While this abstraction does not fully capture all aspects of real-world social interaction, it provides a controlled and tractable setting for studying discourse that is more amenable to LLM-based generation.

\paragraph{Domain selection rationale.}
While YouTube is inherently a video-centric platform, human comments are often influenced by visual stimuli that LLMs cannot directly perceive. To mitigate this mismatch, we carefully select domains where the core content is information-rich and largely recoverable through textual sources, ensuring that LLMs can reasonably approximate the underlying context. Specifically, we avoid visually intensive domains such as lifestyle, beauty and fitness, health-related content, personal vlogs, and gaming streams, where interpretation heavily depends on visual cues or personal experiences. Instead, we focus on domains such as news, pop culture, and technology reviews. News videos are typically content-heavy and align well with existing textual corpora used in LLM training; pop culture content, such as movie or TV trailers, is accompanied by widely available plot summaries and discussions; and technology reviews emphasize detailed descriptions that are often documented in blogs, forums, and knowledge bases. This domain selection ensures a more balanced and controlled setting for comparing human and LLM-generated comments.

\paragraph{Time-frame selection rationale.}
Online platforms have long been subject to bot activity and evolving patterns of user interaction, and YouTube is no exception. However, the end of 2022 marks a clear shift with the widespread adoption of generative AI following the release of ChatGPT. Since then, LLM-based agents have increasingly appeared in online spaces \citep{sun2025we, goldstein2023coming}, and many human comments are influenced or rewritten using these systems. To ensure that our dataset reflects predominantly human-authored discourse, we select a time frame before this transition.

We also avoid the 2020–2022 period, which was heavily influenced by the COVID-19 pandemic. This period introduced unique global dynamics, including heightened misinformation and topic drift across unrelated discussions \citep{li2022youtube, hasan2020socio}, which can distort the natural structure of comment distributions. Additionally, we aim to ensure that the underlying content is within the knowledge scope of modern LLMs, enabling them to generate contextually grounded comments.

Considering these factors, we select the period from mid-2018 to mid-2019. This interval reflects a stable phase of online interaction, following the introduction of YouTube’s community features \footnote{\url{https://blog.youtube/news-and-events/youtube-community-goes-beyond-video/}} and the launch of the Video Reach platform \footnote{\url{https://blog.google/products-and-platforms/products/ads/full-funnel-video/}}, which increased user engagement, while remaining sufficiently distant from recent shifts driven by generative AI and global events.

\subsection{Dataset construction}
\label{subsec_human_dataset}
\paragraph{Video selection criteria.}
For each domain, we first select YouTube channels that align with the nature of the content and ensure consistency across videos. In the news domain, we consider official YouTube channels of established news agencies, both from the United States and internationally, that also publish content in English. For pop culture, we focus on trailer content from official channels of OTT platforms (such as Netflix and Amazon Prime), production studios (e.g., Marvel, DC, Universal, PlayStation, Nintendo), and widely recognized aggregators such as IGN. In the tech domain, we select well-known English-language technology review channels with substantial audience engagement. Table \ref{tab:dataset_channels_links} lists the channels and corresponding videos included in our study. All data are collected using the official YouTube API in compliance with usage policies and quota limits, and no scraping methods that violate platform terms are employed.

\begin{table*}[h]
\caption{Dataset composition by domain. Channel names are clickable links to their respective YouTube pages with (\texttt{number of videos}, \texttt{number of comments considered in total}). }
\label{tab:dataset_channels_links}
\centering
\scriptsize
\setlength{\tabcolsep}{4pt}
\renewcommand{\arraystretch}{1.1}
\begin{tabularx}{\textwidth}{lcccX}
\toprule
\textbf{Domain} & \textbf{\#Channels} & \textbf{\#Videos} & \textbf{\#Comments} & \textbf{Channels (clickable, ordered by \#videos)} \\
\midrule

\textbf{News} & 14 & 3,857 & 518,718 &
\makecell[l]{
\yt{UCXIJgqnII2ZOINSWNOGFThA}{Fox News} (663, 134,427); 
\yt{UCBi2mrWuNuyYy4gbM6fU18Q}{ABC News} (572, 78,937); 
\yt{UC8p1vwvWtl6T73JiExfWs1g}{CBS News} (457, 47,384); \\
\yt{UCaXkIU1QidjPwiAYu6GcHjg}{MSNBC} (379, 55,935); 
\yt{UC16niRr50-MSBwiO3YDb3RA}{BBC News} (379, 42,824); 
\yt{UCeY0bbntWzzVIaj2z3QigXg}{NBC News} (275, 25,079); \\
\yt{UCknLrEdhRCp1aegoMqRaCZg}{DW News} (248, 19,516); 
\yt{UCvJJ_dzjViJCoLf5uKUTwoA}{CNBC} (215, 29,415); 
\yt{UCoMdktPbSTixAyNGwb-UYkQ}{Sky News} (178, 15,488); \\
\yt{UCupvZG-5ko_eiXAupbDfxWw}{CNN} (169, 34,187); 
\yt{UCNye-wNBqNL5ZzHSJj3l8Bg}{Al Jazeera} (157, 15,315); 
\yt{UCK7tptUDHh-RYDsdxO1-5QQ}{Wall Street Journal} (143, 18,314); 
\\
\yt{UC86dbj-lbDks_hZ5gRKL49Q}{AFP News} (15, 1,124); 
\yt{UC52X5wxOL_s5yw0dQk7NtgA}{Associated Press} (7, 773)
} \\

\midrule

\textbf{Pop} & 22 & 2,203 & 326,542 &
\makecell[l]{
\yt{UCGIY_O-8vW4rfX98KlMkvRg}{Nintendo} (469, 78,785); 
\yt{UC-2Y8dQb0S6DtpxNgAKoJKA}{PlayStation} (405, 63,392); 
\yt{UCWOA1ZGywLbqmigxE4Qlvuw}{Netflix} (363, 58,753); \\
\yt{UCUnRn1f78foyP26XGkRfWsA}{GameSpot Trailers} (253, 32,632); 
\yt{UCjBp_7RuDBUYbd1LegWEJ8g}{Xbox} (189, 17,743); 
\yt{UCz97F7dMxBNOfGYu3rx8aCw}{Sony Pictures} (60, 8,754); \\
\yt{UCZSNzBgFub_WWil6TOTYwAg}{Netflix India} (53, 7,878); 
\yt{UCQJWtTnAHhEG5w4uN0udnUQ}{Prime Video} (51, 7,470); 
\yt{UCvC4D8onUfXzvjTOM-dBfEA}{Marvel} (48, 8,164); \\
\yt{UCjmJDM5pRKbUlVIzDYYWb6g}{Warner Bros.} (46, 5,779); 
\yt{UCq0OueAsdxH6b8nyAspwViw}{Universal} (41, 5,910); 
\yt{UCJx5KP-pCUmL9eZUv-mIcNw}{GameTrailers} (39, 4,967); \\
\yt{UC2-BeLxzUBSs0uSrmzWhJuQ}{20th Century} (31, 4,343); 
\yt{UCE5mQnNl8Q4H2qcv4ikaXeA}{Hulu} (31, 3,479); 
\yt{UCF9imwPMSGz4Vq1NiTWCC7g}{Paramount} (28, 4,171); \\
\yt{UCiifkYAs_bq1pt_zbNAzYGg}{DC} (24, 4,204); 
\yt{UCKy1dAqELo0zrOtPkf0eTMw}{IGN} (22, 3,915); 
\yt{UC4zWG9LccdWGUlF77LZ8toA}{Prime India} (20, 3,201); \\
\yt{UC5Qk8mWBwtMyEj7iQQYRk1A}{Epic Games} (13, 1,168); 
\yt{UCwSIJCMWZC5GDM59wj7pMsg}{Prime UK} (10, 1,129); 
\yt{UC0KU8F9jJqSLS11LRXvFWmg}{Ubisoft} (5, 540); 
\yt{UC1Myj674wRVXB9I4c6Hm5zA}{Apple TV} (2, 165)
} \\

\midrule

\textbf{Tech} & 10 & 1,514 & 268,954 &
\makecell[l]{
\yt{UCXuqSBlHAE6Xw-yeJA0Tunw}{Linus Tech Tips} (272, 45,354); 
\yt{UCsTcErHg8oDvUnTzoqsYeNw}{Unbox Therapy} (207, 30,910); 
\yt{UCOmcA3f_RrH6b9NmcNa4tdg}{CNET} (201, 25,949); \\
\yt{UCMiJRAwDNSNzuYeN2uWa0pA}{Mrwhosetheboss} (158, 36,402); 
\yt{UCgyqtNWZmIxTx3b6OxTSALw}{Android Authority} (135, 19,224); 
 \\
 \yt{UCXGgrKt94gR6lmN4aN3mYTg}{Austin Evans} (131, 31,866);
\yt{UCbLq9tsbo8peV22VxbDAfXA}{GSMArena} (112, 17,750); 
\yt{UCVYamHliCI9rw1tHR1xbkfw}{Dave2D} (109, 22,779); 
 \\
 \yt{UC9fSZHEh6XsRpX-xJc6lT3A}{UrAvgConsumer} (103, 28,077);
\yt{UCBJycsmduvYEL83R_U4JriQ}{MKBHD} (86, 10,643)
} \\

\bottomrule
\end{tabularx}

\end{table*}

From these channels, we select videos based on a set of criteria to ensure relevance and consistency. First, we control for video type and duration. In the news domain, we focus on individual news clips rather than aggregated programs, talk shows, or live broadcasts. In pop culture, we restrict our selection to official trailers and exclude reaction videos, gameplay footage, or clips. In tech reviews, we consider videos centered on specific products. Figure \ref{fig_appendix_video_duration} shows the distribution of video lengths across domains.

\begin{figure}[h]
    \centering
     \begin{subfigure}{0.33\textwidth}
        \centering
        \includegraphics[width=\linewidth]{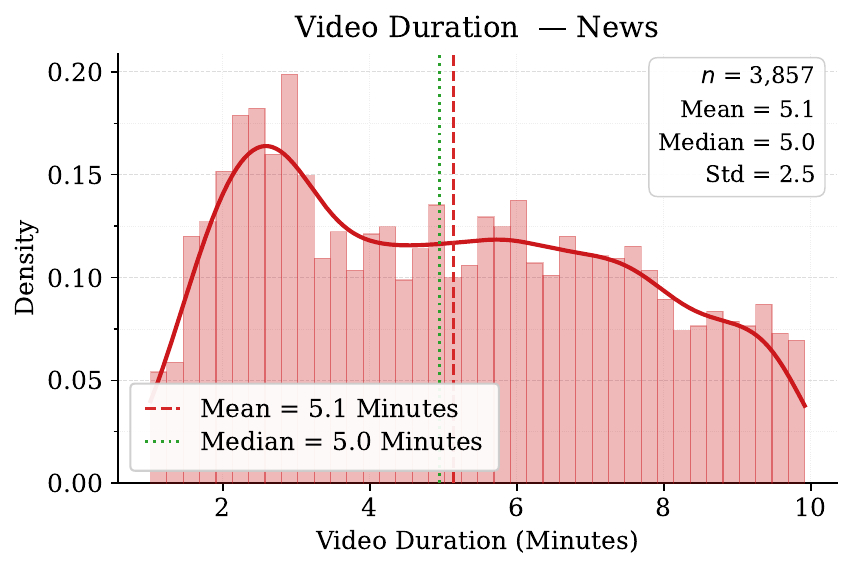}
    \end{subfigure}%
    \begin{subfigure}{0.33\textwidth}
        \centering
        \includegraphics[width=\linewidth]{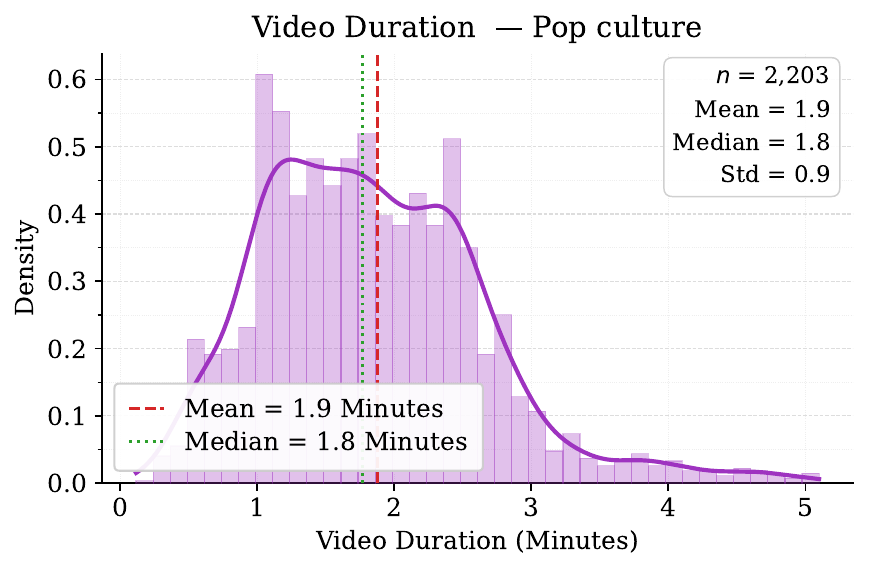}
    \end{subfigure}%
    \begin{subfigure}{0.33\textwidth}
        \centering
        \includegraphics[width=\linewidth]{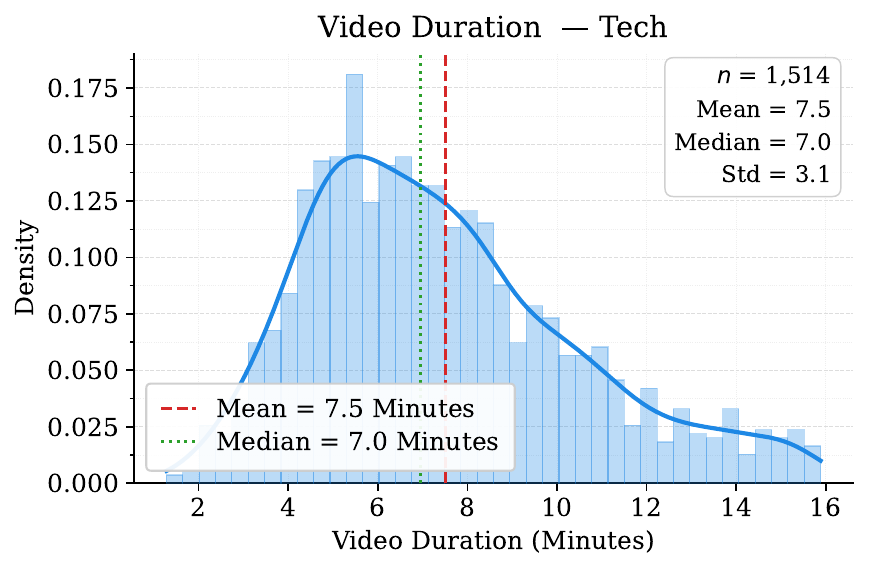}
    \end{subfigure}
    \caption{Distribution of video duration (minutes) in different domains in our dataset. We prefer shorter videos in specific channels to keep the context minimal.}
    \label{fig_appendix_video_duration}
\end{figure}

Second, we apply filters based on video characteristics and engagement. We include only English-language videos and prioritize those with available subtitles, particularly for news and tech content, to better capture the underlying discourse. We also focus on videos with meaningful engagement, selecting those within the top 75\% of view counts for each channel during the selected time frame, with at least 100 initial comments and a favorable like-to-dislike ratio. Figure \ref{fig_app_video_per_month} illustrates the distribution of selected videos over time.

\begin{figure}
    \centering
    \includegraphics[width=\linewidth]{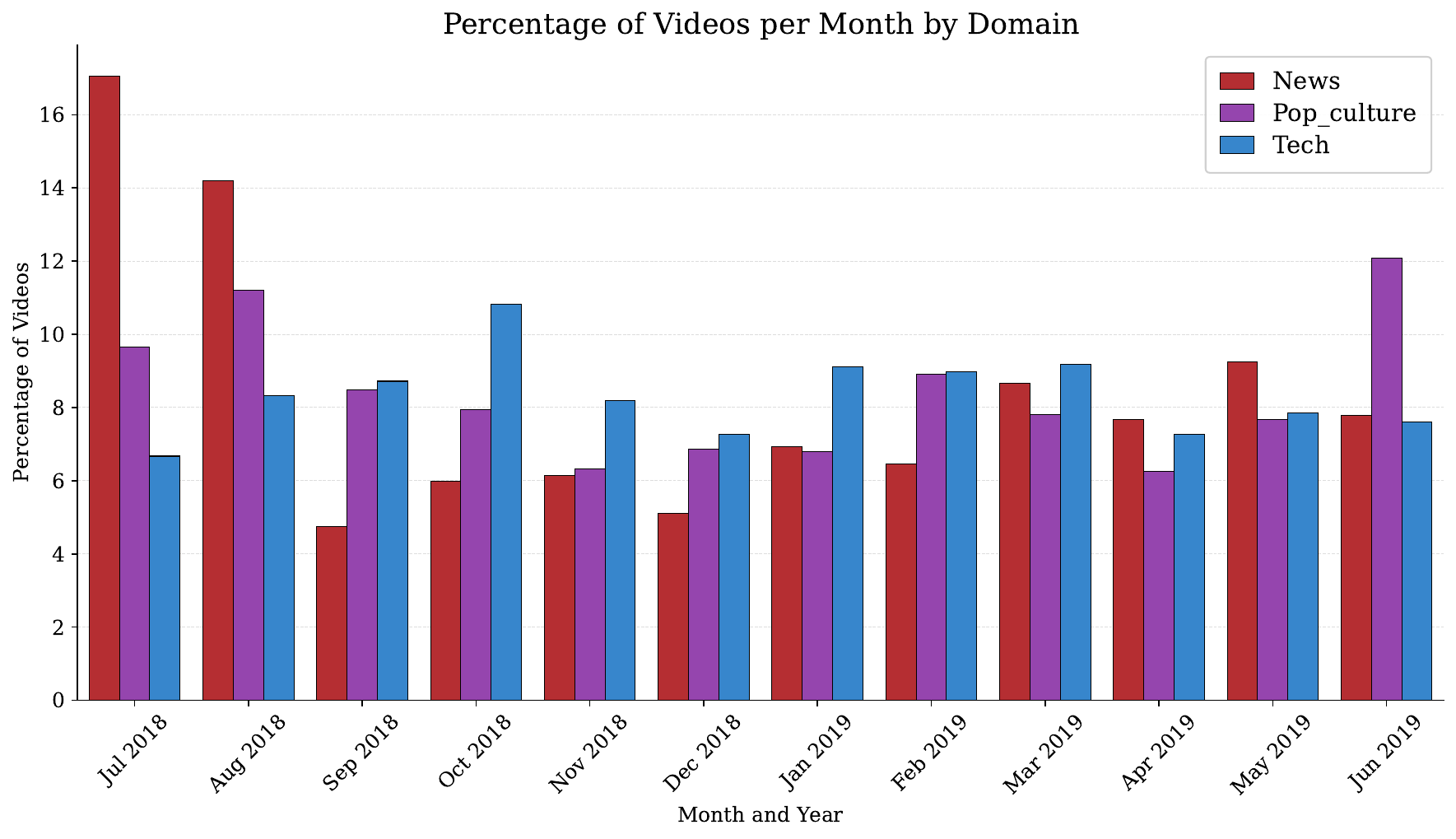}
    \caption{Distribution of selected videos per month in our time-frame.}
    \label{fig_app_video_per_month}
\end{figure}

The resulting dataset covers a diverse range of topics within each domain. Figure \ref{fig_appendix_wordcloud} presents word clouds derived from video titles, highlighting this diversity. For example, news videos span major global and political events and other specific news also, such as the Thai cave rescue (2018) and the Notre Dame fire (2019) during that time, while tech reviews cover a variety of consumer products, including smartphones and laptops released during the selected period.

For each selected video, we constructed a video-context record that could be reused by both aspect discovery and LLM comment generation. The record contains the original video metadata, including channel, domain, title, posting date, and transcript/subtitle availability, together with a compact summary of the video content. When subtitles or transcripts were available, the transcript text was cleaned and summarized with an LLM (Gemini-2.5) into a short description of the main claims, events, entities, and discussion points in the video. This summary was used instead of the full transcript in later prompts so that every generation call received consistent video context without exceeding model context limits.

\begin{figure}[h]
    \centering
     \begin{subfigure}{0.33\textwidth}
        \centering
        \includegraphics[width=\linewidth]{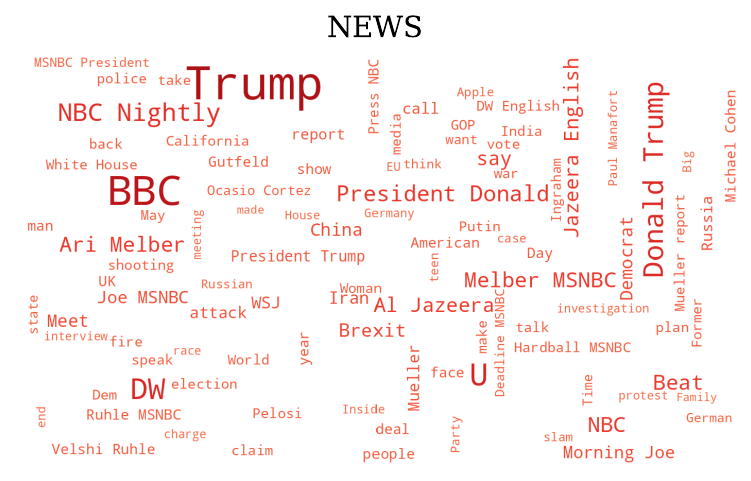}
    \end{subfigure}%
    \begin{subfigure}{0.33\textwidth}
        \centering
        \includegraphics[width=\linewidth]{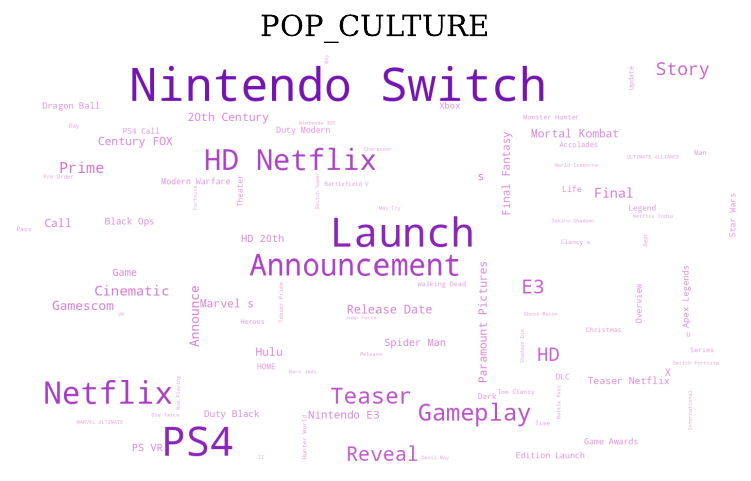}
    \end{subfigure}%
    \begin{subfigure}{0.33\textwidth}
        \centering
        \includegraphics[width=\linewidth]{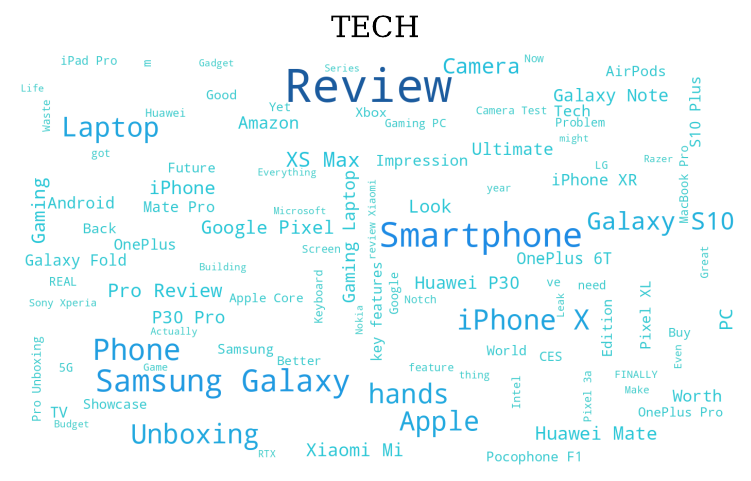}
    \end{subfigure}
    \caption{Word cloud distribution from the titles of the videos in our dataset.}
    \label{fig_appendix_wordcloud}
\end{figure}

\paragraph{Human comment collection and filtering.}
From the selected videos, we collect comments using YouTube’s “top comments” sorting to capture representative and highly visible responses. To reduce temporal effects, we restrict our analysis to comments posted within a fixed time window after the video is published. Based on the observed distribution of commenting activity, we primarily consider comments within the first one to three months, ensuring sufficient coverage while limiting long-term drift. For each video, we cap the number of comments at 400, as some videos attract extremely large volumes of comments. This cap helps maintain a balanced distribution across videos and ensures comparability with LLM-generated comments, for which generating very large quantities would be computationally prohibitive.

\begin{figure}[h]
    \centering
     \begin{subfigure}{0.33\textwidth}
        \centering
        \includegraphics[width=\linewidth]{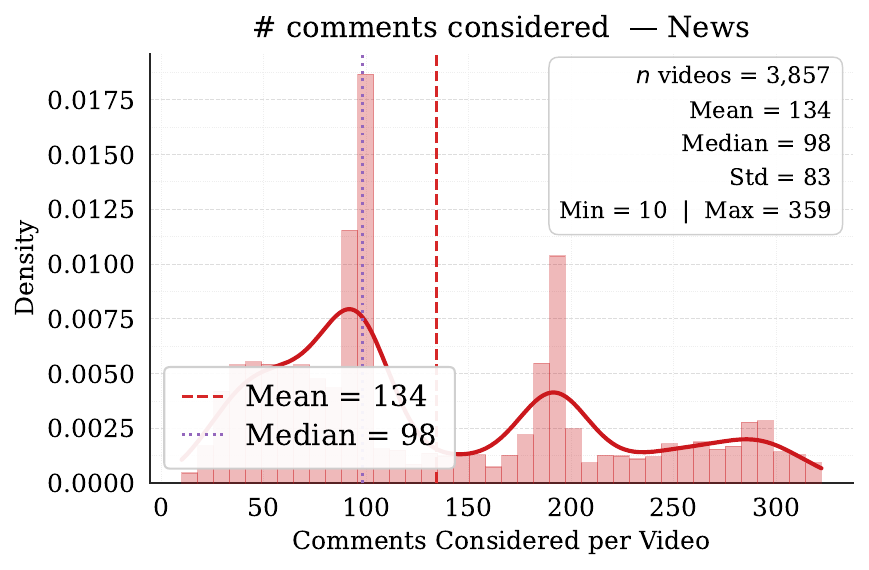}
    \end{subfigure}%
    \begin{subfigure}{0.33\textwidth}
        \centering
        \includegraphics[width=\linewidth]{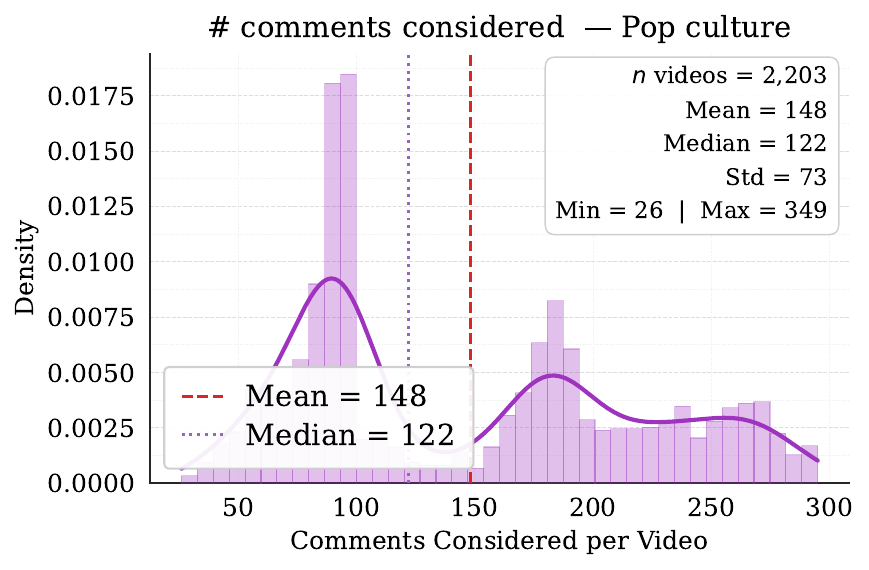}
    \end{subfigure}%
    \begin{subfigure}{0.33\textwidth}
        \centering
        \includegraphics[width=\linewidth]{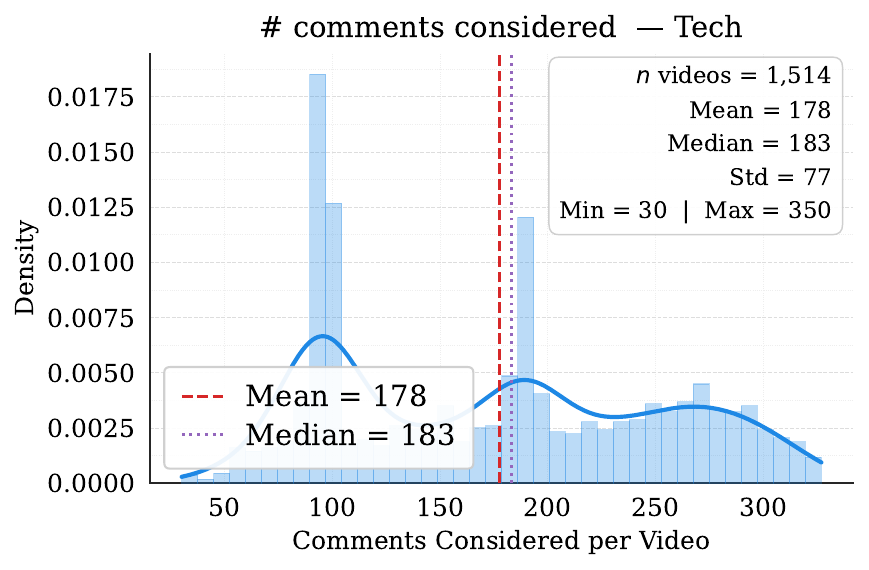}
    \end{subfigure}
    \caption{Distribution of the number of comments considered per video (after the filtering steps) in each domain.}
    \label{fig_appendix_comment_count}
\end{figure}

We apply a series of filtering steps to ensure comment quality, consistency, and comparability across videos and domains:

\begin{enumerate}

\item \textbf{Remove duplicate comments.} We drop duplicate rows using the combination of \texttt{Comment\_id}, \texttt{Video\_id}, and \texttt{Comment\_text}. This removes repeated entries from API pagination artifacts and ensures that duplicate comments from the same user within a video do not inflate frequency.

\item \textbf{Keep English comments.} We identify the language of each comment using a FastText-based classifier \citep{joulin-etal-2017-bag} and retain only comments labeled as English (\texttt{en}). This keeps the downstream feature extraction and evaluation pipeline consistent, since our lexical, syntactic, pragmatic, toxicity, and embedding models are primarily designed for English.

\item \textbf{Filter spam.} We remove comments identified as spam using a RoBERTa-based spam detector \citep{shenoda2024robertaspam}. This step is applied before ranking so that engagement signals from spam do not bias the selected human set.

\item \textbf{Filter toxic comments.} We remove highly toxic comments using Perspective-style toxicity models \citep{vidgen2021ltoxicity, gehman2020realtoxicityprompts}. Comments with toxicity scores above 0.95 are excluded, as such extreme content is unlikely to be produced by base LLMs and would introduce a mismatch in comparison.

\item \textbf{Apply domain-specific lag windows.} We filter comments based on the time lag between video publication and comment posting. The maximum lag is set to 7 days for news, 90 days for pop culture, and 30 days for tech. This accounts for domain differences, as news discussions are more sensitive to later events, while pop culture and technology videos support longer engagement windows.

\item \textbf{Compute a popularity-recency selection score.} For the remaining top-level comments, we rank comments within each video using a weighted score based on likes, replies, and recency:
\[
\text{selection\_score}
=
\log(1 + \text{likes})
+
\log(1 + \text{replies})
-
\log(1 + \text{Lag\_time\_days})
\]
We then select the top comments per video, with a target of 200-300 top-level comments per video in our configuration.

\item \textbf{Remove extreme length outliers.} We discard excessively long comments beyond the 95th percentile of the length distribution to avoid outliers, while retaining short comments since brevity is common in social media discourse.

\end{enumerate}


Figure \ref{fig_appendix_comment_count} shows the distribution of selected comments per video across domains. 
Overall, these filtering and selection steps ensure that the resulting dataset captures a clean, consistent, and representative subset of highly visible human comments, reflecting the dominant patterns of engagement and discourse that users are most likely to encounter online.

\paragraph{Human comment aspect extraction}
\label{appendix_aspect_generation}

For aspect-conditioned generation, we approximate discussion aspects directly from human comments. Our goal is to capture both what people talk about and how they express it. To this end, we design a two-stage, complementary extraction pipeline (Algorithm \ref{alg:aspect_extraction}), which jointly models semantic topics and socio-pragmatic styles.

\begin{algorithm}[h]
\caption{Aspect Extraction from Human Comments}
\label{alg:aspect_extraction}
\begin{algorithmic}

\State \textbf{Input:} Video $v$, human comment set $\mathcal{H}_v$
\State \textbf{Output:} Aspect set $\mathcal{A}_v = \{a_1, \dots, a_{A_v}\}$

\State $\mathcal{H}_v^{\text{long}} \gets \{h \in \mathcal{H}_v \mid n_{\text{sent}}(h) > 1\}$
\State $\mathcal{H}_v^{\text{short}} \gets \{h \in \mathcal{H}_v \mid n_{\text{sent}}(h) = 1\}$

\State $\mathcal{C}_v^{\text{long}} \gets \textsc{BERTTopic}(\{\boldsymbol{\theta}(h) \mid h \in \mathcal{H}_v^{\text{long}}\})$
\Comment{cluster long comments via topic distributions}

\State $\mathcal{C}_v^{\text{short}} \gets \textsc{KModes}(\{\mathbf{p}(h) \mid h \in \mathcal{H}_v^{\text{short}}\})$
\Comment{cluster short comments via pragmatic attributes}

\State $\mathcal{C}_v \gets \mathcal{C}_v^{\text{short}} \cup \mathcal{C}_v^{\text{long}}$
\State $\mathcal{A}_v \gets \emptyset$

\For{$C_j \in \mathcal{C}_v$}
    \State $\mathcal{R}_j \subseteq C_j$
    \Comment{select representative comments}
    \State $a_j \gets \textsc{LLM-Summarize}(\mathcal{R}_j)$
    \State $\mathcal{A}_v \gets \mathcal{A}_v \cup \{a_j\}$
\EndFor

\State \Return $\mathcal{A}_v$

\end{algorithmic}
\end{algorithm}

In the first stage, we extract \textbf{video-specific semantic aspects using topic modeling}. We focus on longer comments (with more than one sentence), as they tend to contain richer contextual signals. For each video, we represent comments using sentence embeddings from the EmbeddingGemma model \citep{vera2025embeddinggemma} and cluster them with BERTopic \citep{grootendorst2022bertopic} combined with HDBSCAN \citep{campello2013hdbscan}. We perform a grid search over key clustering parameters to balance granularity and robustness. Specifically, we vary $min\_samples$ (5, 10, 15, 20), which controls how aggressively outliers are filtered, and $min\_cluster\_size$ (15, 30, 45), which governs the minimum size of coherent discussion groups. Model selection is guided by topic coherence, diversity, and coverage to ensure that extracted aspects are meaningful, non-redundant, and collectively representative of the discussion space. We run this process independently for each video rather than globally, since our objective is to recover the local discourse structure specific to each video. The resulting clusters capture semantically coherent aspects such as reactions to specific claims, discussions about named entities or events, debates on technical points, or a focus on particular subtopics introduced in the content.

The second stage captures \textbf{socio-pragmatic aspects}, which are often missed by purely semantic clustering, especially for short or stylistically driven comments. For comments that are not confidently assigned to semantic clusters, as well as short comments, we construct feature vectors based on communicative attributes. These include sentiment, emotion, humor, sarcasm, toxicity or hate-related signals, metaphor usage, formality, bias, and related stylistic indicators. We cluster these representations using K-Modes clustering \citep{chaturvedi2001k} with Jaccard similarity, which is well-suited for categorical features. We explore cluster sizes from K=2 to 10 and select the optimal configuration using the silhouette coefficient, Dunn index, and intra/inter-cluster distance ratios. This stage ensures that stylistic variation, affective stance, and social signaling are explicitly modeled, complementing the semantic view.

\begin{figure}[h]
    \centering
     \begin{subfigure}{0.33\textwidth}
        \centering
        \includegraphics[width=\linewidth]{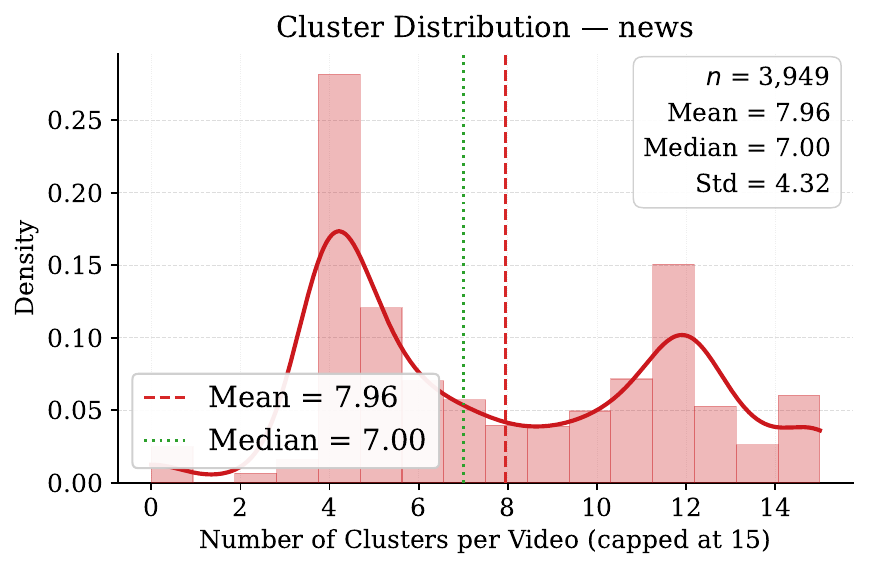}
    \end{subfigure}%
    \begin{subfigure}{0.33\textwidth}
        \centering
        \includegraphics[width=\linewidth]{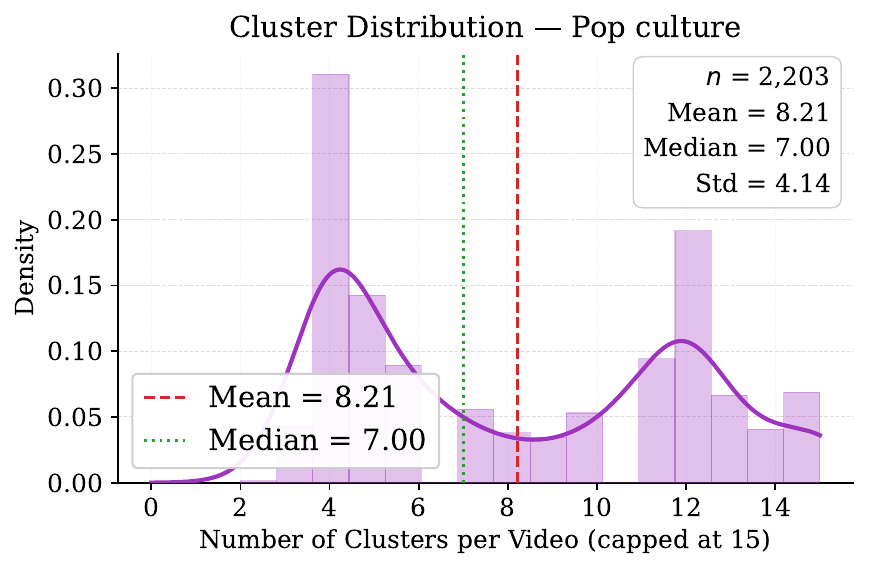}
    \end{subfigure}%
    \begin{subfigure}{0.33\textwidth}
        \centering
        \includegraphics[width=\linewidth]{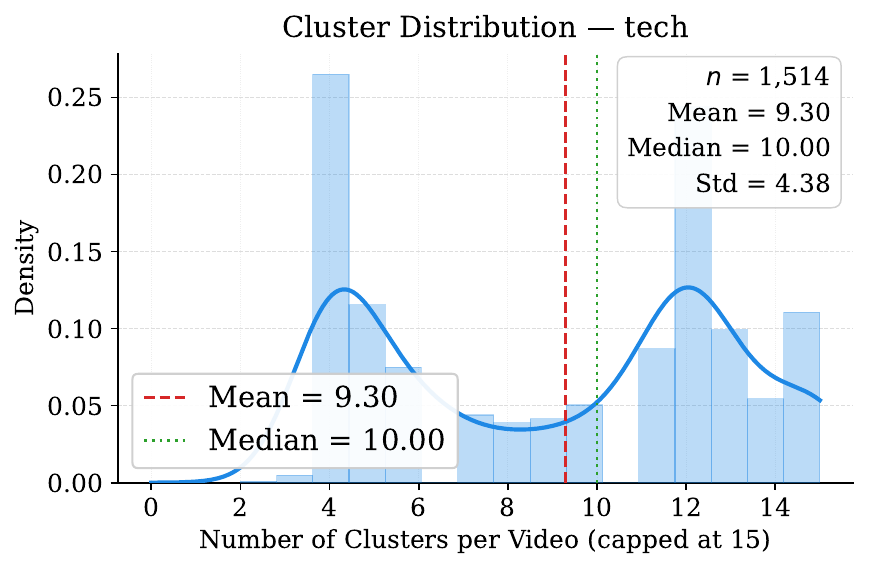}
    \end{subfigure}
    \caption{Distribution of aspects in different videos.}
    \label{fig_appendix_clustert_count}
\end{figure}

After clustering, we generate interpretable cluster descriptions to make the extracted aspects usable in downstream generation. We train a shallow decision tree to predict cluster labels from one-hot encoded features and compute mutual information \citep{kraskov2004estimating} between each feature and cluster assignments. For each cluster, we select features that are both frequent within the cluster and highly informative. These features are combined into concise descriptors that highlight the defining characteristics of the cluster, alongside representative comments for grounding. Figure \ref{fig_appendix_clustert_count} shows the distribution of aspects in different videos in our dataset.

\begin{tcolorbox}[
    colback=blue!10, 
    colframe=blue!40!black, 
    title=Prompt for comment summary for aspect definition, 
    rounded corners
]\footnotesize

\textbf{Task:} Define aspects from a given comment list.

\vspace{2pt}
\textbf{Video Title:} \texttt{\{video\_name\}} \\
\textbf{Channel Name:} \texttt{\{channel\_name\}} \\
\textbf{Post Date:} \texttt{\{post\_date\}}

\vspace{4pt}
Write a concise summary (3--4 sentences) describing what the comments discuss. 
Focus on:
\begin{itemize}
    \item Main topics and viewpoints
    \item Overall sentiment and emotional tone
    \item Stylistic elements (e.g., humor, sarcasm, bias, toxicity if present)
\end{itemize}

Your response should be a single cohesive paragraph without bullet points or special formatting.

\vspace{4pt}
\textbf{Comments:} \\
\texttt{\{comment\_list\}}

\end{tcolorbox}

Finally, we produce natural language summaries for each aspect by an LLM (\texttt{gpt-5}) using the above prompt with the representative comments for each aspect. For semantic clusters, the summaries emphasize shared discussion topics or content-specific themes. For pragmatic clusters, they focus on communicative style, tone, stance, or social intent. These summaries, along with cluster sizes and length statistics, are incorporated into the prompt construction pipeline. As a result, in aspect-conditioned settings, the model receives not only the video context but also a structured and comprehensive representation of human discussion aspects, enabling more faithful and diverse generation. Table \ref{tab:aspect_examples} shows example aspects extracted from a sample video.

\begin{table}[h]
\centering
\small
\setlength{\tabcolsep}{4pt}
\renewcommand{\arraystretch}{1.1}

\begin{tabular}{l c l p{6cm}}
\toprule
\textbf{Aspect} & \textbf{\%} & \textbf{Pragmatic} & \textbf{Aspect Description} \\
\midrule

\makecell[l]{\texttt{Enthusiastic Fan}\\\texttt{Reactions} \\ (Pragmatic)} 
& 65.07 
& \makecell[l]{informal \\ joy \\ non-sarcastic} 
& Strong excitement around the film and its Bruce Springsteen connection, with nostalgic references to the soundtrack and repeated praise for the trailer, alongside expressive and hyperbolic fan reactions. \\

\makecell[l]{\texttt{Music-driven Emotional}\\\texttt{Resonance} \\ (Topic)} 
& 13.97 
& -- 
& Viewers share emotional connections to Springsteen’s music and highlight how the film’s coming-of-age narrative resonates with personal experiences and nostalgia. \\

\makecell[l]{\texttt{Emotional \& Relatable}\\\texttt{Storytelling} \\ (Topic)} 
& 13.10 
& -- 
& Comments emphasize the film’s emotional depth, relatability, and themes of family, identity, and struggle, often describing cathartic and uplifting viewing experiences. \\

\makecell[l]{\texttt{Springsteen Legacy}\\\texttt{Appreciation} \\ (Topic)} 
& 4.80 
& -- 
& Discussion centers on admiration for Springsteen’s influence, with personal reflections on how his music shapes identity and enhances the film’s impact. \\

\makecell[l]{\texttt{Cultural Identity \&}\\\texttt{Representation} \\ (Topic)} 
& 2.18 
& -- 
& Comments focus on South Asian and working-class identity, highlighting cultural pressures and the empowering role of music in shaping self-expression. \\

\makecell[l]{\texttt{Playful Music Culture}\\\texttt{Observations} \\ (Pragmatic)} 
& 0.87 
& \makecell[l]{sentiment  \\ mixed \\ light humor} 
& Lighthearted remarks about music preferences and cultural trends, combining casual humor with appreciation for classic rock influences. \\

\bottomrule
\end{tabular}

\caption{Example of extracted aspects from a sample video (BLINDED BY THE LIGHT - Official Trailer, ID = \textit{DmmHvnS0IKM}).}
\label{tab:aspect_examples}
\end{table}

\subsection{LLM generation framework}
\label{subsec_llm_generation}
In this section, we present a detailed analysis of how we use our framework (Figure \ref{fig:methodology}) for generating comments from LLMs using multiple settings. 

\paragraph{Model selection.}
To enable a comprehensive evaluation of multi-LLM generation, we curate a diverse set of both proprietary and open-source models spanning multiple providers and architectural families. Our selection is designed to capture variation in training data, alignment strategies, model scale, and deployment paradigms, which are all known to influence generation behavior.
We include models from \textbf{OpenAI} (\texttt{GPT} \citep{achiam2023gpt4, singh2025gpt5} and \texttt{GPT-OSS} \citep{agarwal2025gptoss}), \textbf{Anthropic} (\texttt{Claude} \citep{anthropic2025claude37, anthropic2025claudehaiku45, anthropic2025claudesonnet45}), \textbf{Google} (\texttt{Gemini} \citep{comanici2025gemini2.5} and \texttt{Gemma} \citep{gemmateam2025gemma3technicalreport}), \textbf{Meta} (\texttt{Llama} \citep{adcock2026llama4}), \textbf{DeepSeek} \citep{liu2024deepseek}, \textbf{Qwen} \citep{qwen2024qwen2.5, yang2025qwen3}, \textbf{Mistral} \citep{jiang2024mixtral}, \textbf{xAI} (\texttt{Grok} \citep{xai2025grok3, xai2025grok4, xai2025grok4-1}), and \textbf{Microsoft} (\texttt{Phi} \citep{abdin2024phi3, abdin2024phi4, abouelenin2025phi4-mini}). Table \ref{tab:model_list} provides the complete list of models used in our study.
\begin{table}[h]
\centering
\small
\setlength{\tabcolsep}{4pt}
\begin{tabular}{l p{10cm}}
\toprule
\textbf{Provider} & \textbf{Models} \\
\midrule

\multirow{7}{*}{\textbf{OpenAI (7)}} 
& \ding{51} \texttt{gpt-oss-20b} \\
& \ding{51} \texttt{gpt-oss-120b} \\
& \ding{55} \texttt{gpt-5} \\
& \ding{55} \texttt{gpt-5-mini} \\
& \ding{55} \texttt{gpt-5-nano} \\
& \ding{55} \texttt{gpt-4.1} \\
& \ding{55} \texttt{gpt-4.1-mini} \\

\midrule

\multirow{4}{*}{\textbf{Anthropic (4)}}
& \ding{55} \texttt{claude-haiku-4-5-20251001} \\
& \ding{55} \texttt{claude-sonnet-4-5-20250929} \\
& \ding{55} \texttt{claude-3-5-haiku-20241022} \\
& \ding{55} \texttt{claude-3-7-sonnet-20250219} \\

\midrule

\multirow{7}{*}{\textbf{Google (7)}}
& \ding{55} \texttt{gemini-2.5-flash} \\
& \ding{55} \texttt{gemini-2.5-lite-latest} \\
& \ding{55} \texttt{gemini-2.5-pro} \\
& \ding{55} \texttt{gemini-2.0-flash} \\
& \ding{55} \texttt{gemini-2.0-flash-lite} \\
& \ding{51} \texttt{gemma-3-4b-it} \\
& \ding{51} \texttt{gemma-3-27b-it} \\

\midrule

\multirow{2}{*}{\textbf{DeepSeek (2)}}
& \ding{55} \texttt{deepseek-chat} \\
& \ding{55} \texttt{deepseek-reasoner} \\

\midrule

\multirow{2}{*}{\textbf{Mistral (2)}}
& \ding{51} \texttt{Mixtral-8x7B-Instruct-v0.1} \\
& \ding{51} \texttt{Mistral-Small-24B-Instruct-2501} \\

\midrule

\multirow{3}{*}{\textbf{Grok (3)}}
& \ding{55} \texttt{grok-4-1-fast-non-reasoning} \\
& \ding{55} \texttt{grok-4-fast-non-reasoning} \\
& \ding{55} \texttt{grok-3-mini} \\

\midrule

\multirow{4}{*}{\textbf{Microsoft (4)}}
& \ding{51} \texttt{phi-4} \\
& \ding{51} \texttt{phi-3.5-MoE-instruct} \\
& \ding{51} \texttt{phi-3-medium} \\
& \ding{51} \texttt{phi-4-mini-instruct} \\

\bottomrule
\end{tabular}

\vspace{2mm}
\caption{\small Models grouped by provider. \ding{51} denotes open-source models and \ding{55} denotes proprietary models.}
\label{tab:model_list}
\end{table}

Proprietary models are accessed through their official APIs, while open-source models are served via the Together.AI platform \footnote{\url{https://api.together.ai/}}. This unified access setup allows us to systematically compare models under consistent prompting and decoding conditions. We intentionally include models from a wide range of providers to avoid over-reliance on a single ecosystem and to better reflect the heterogeneity of real-world LLM deployments.
Our selection was constrained by model availability at the time of experimentation. As a result, some models have since been deprecated or updated by their providers. We retain them in our analysis to ensure reproducibility and to capture a representative snapshot of the LLM landscape during the study period.

\paragraph{Common generation guidelines}
For each generation setting, we first construct an assignment plan that specifies which model generates comments for each video and, when applicable, for each length bin, aspect, or style cluster. This design is motivated by the well-established role of length as a confounding factor in NLP generation and evaluation \citep{hu2024explaining, munoz2024contrasting, tripto2025beyond}. To control for this, we explicitly align generation with the observed length distribution of human comments, ensuring that comparisons across settings reflect differences in content and diversity rather than superficial length effects. Each assignment entry includes the video identifier, provider, model identifier, target number of comments, and the prompt-conditioning signals required for that setting. This structured plan decouples the generation pipeline from any single provider and enables consistent execution across both open-weight and proprietary models. Figure \ref{fig_appendix_comment_length} shows the distribution of comment lengths in different settings.

Prompt construction varies by setting but follows a shared backbone to maintain comparability. All prompts include core contextual information such as channel name, domain, video title, posting date, video summary or transcript summary, target number of comments, and a specified word-count range. Aspect-conditioned settings extend this with cluster-level descriptions, length guidelines tied to each aspect, and style-related cues when available. This unified structure ensures that all models operate under comparable informational constraints while allowing controlled variation across experimental conditions.

\begin{figure}[h]
    \centering
     \begin{subfigure}{0.33\textwidth}
        \centering
        \includegraphics[width=\linewidth]{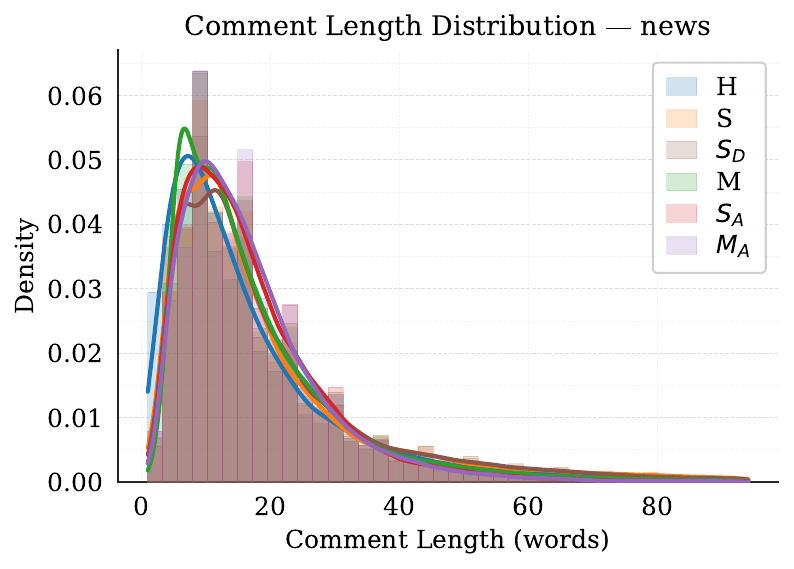}
    \end{subfigure}%
    \begin{subfigure}{0.33\textwidth}
        \centering
        \includegraphics[width=\linewidth]{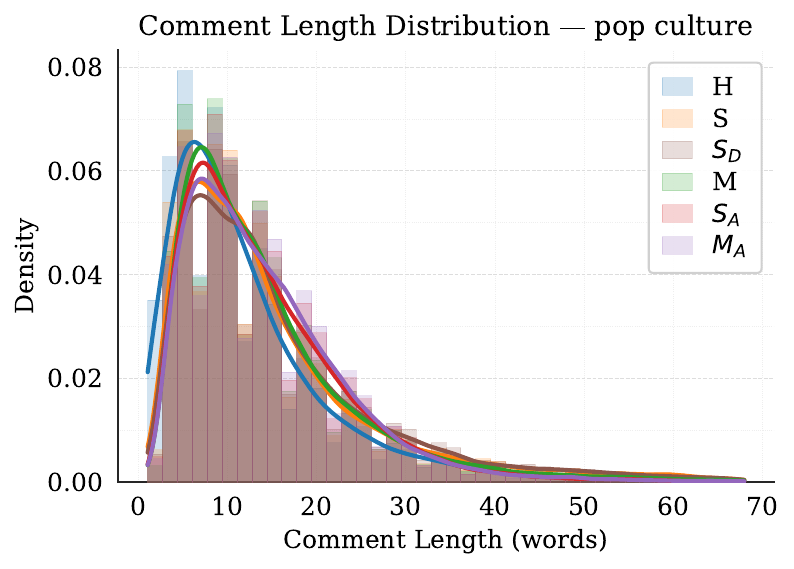}
    \end{subfigure}%
    \begin{subfigure}{0.33\textwidth}
        \centering
        \includegraphics[width=\linewidth]{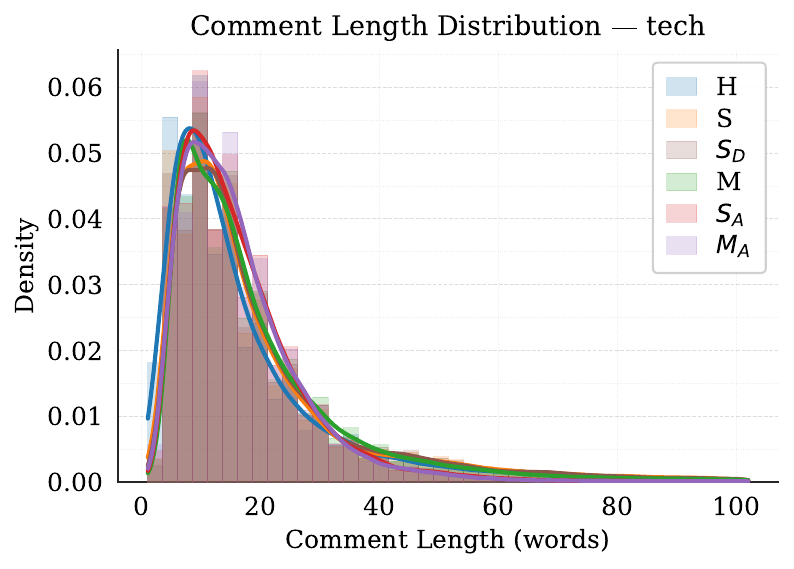}
    \end{subfigure}
    \caption{Comment length distribution in human comments and LLM-generated comments from different settings. No statistical significance difference was observed for each setting (wrt. human distribution) following a two-tailed t-test.}
    \label{fig_appendix_comment_length}
\end{figure}

To standardize outputs, we instruct models to return a numbered list of comments. This simple constraint enables reliable parsing of individual comments using a consistent regular expression, reducing post-processing variability across providers. Generation calls use provider-specific decoding parameters while preserving the same experimental intent. For the main non-ablation settings, we set the temperature to approximately 1.0, unless a provider enforces a different default or does not expose equivalent controls. We also allocate sufficiently large maximum output limits to avoid truncation by limiting how many comments to generate and allowing models to produce the full set of requested comments. In practice, most models are assigned output budgets of around 8192 tokens. Microsoft Phi models follow their model-specific constraints, with lower limits (2048 tokens) for smaller variants. This setup ensures both fairness and completeness across a heterogeneous set of generation interfaces. Also, in each generation call, we generate a small surplus of comments beyond the specified comment count as a practical safeguard against formatting or parsing failures.

\paragraph{Setting S: Single-LLM, No Aspects}
This setting serves as the simplest baseline, where all comments for a given video are generated using a single LLM without any explicit aspect or style conditioning (Algorithm \ref{alg:setting_s_generation}). For each video $v$, we first estimate the empirical length distribution $\mathcal{L}_v$ from the corresponding human comments $\mathcal{H}_v$. This step ensures that the generated comments follow realistic length patterns and controls for length as a confounding factor in downstream comparisons. We then uniformly sample one model $M_k$ from the model pool $\mathcal{M}$ and use it for all generations associated with that video. 

\begin{algorithm}[h]
\caption{Comment Generation for Setting S: Single-LLM, No Aspects}
\label{alg:setting_s_generation}
\begin{algorithmic}

\State \textbf{Input:} Video $v$, human comment set $\mathcal{H}_v$, model pool $\mathcal{M}=\{M_1,\dots,M_K\}$, target number of comments $n$
\State \textbf{Output:} Generated comment set $\mathcal{G}^{(\mathrm{S})}_v=\{g_1,\dots,g_n\}$

\State Estimate empirical length-bin distribution from $\mathcal{H}_v$:
\[
\mathcal{L}_v = \{(\ell_b, \pi_b)\}_{b=1}^{B},
\]
\Statex \hspace{\algorithmicindent} where $\ell_b$ denotes a word-length bin and $\pi_b$ is the proportion of comments.

\State Sample a single model $M_k \sim \mathrm{Uniform}(\mathcal{M})$
\Comment{single model for all generations}

\State $\mathcal{G}^{(\mathrm{S})}_v \gets \emptyset$

\For{$b \in \{1,\dots,B\}$}
    \State $n_b \gets \lfloor n \cdot \pi_b \rceil$
    \Comment{number of comments for bin $\ell_b$}

    \State $\mathcal{G}_b \sim P_{M_k}^{(n_b)}(\cdot \mid v, \ell_b)$
    \Comment{generates approx. $n_b$ comments conditioned on $v$ and $\ell_b$}

    \State $\mathcal{G}^{(\mathrm{S})}_v \gets \mathcal{G}^{(\mathrm{S})}_v \cup \mathcal{G}_b$
\EndFor

\State Adjust $\mathcal{G}^{(\mathrm{S})}_v$ to contain exactly $n$ comments if rounding affects total

\State \Return $\mathcal{G}^{(\mathrm{S})}_v$

\end{algorithmic}
\end{algorithm}

The total number of comments $n$ is allocated across length bins according to the observed proportions $\pi_b$, and the selected model generates comments conditioned only on the video context and the target length bin. The generated subsets from each bin are combined to form the final comment set, with minor adjustments applied if rounding affects the total count. This design isolates the behavior of a single model under realistic length constraints, providing a controlled reference point for understanding the contributions of multi-model and aspect-conditioned settings. 

\paragraph{Setting M: Multi-LLM, No Aspects}
Building on Setting S, this setting introduces model diversity while keeping the generation unconditioned on aspects. Instead of relying on a single model, we select a subset of models $\mathcal{M}_v \subseteq \mathcal{M}$ for each video and distribute the generation workload across them (Algorithm~\ref{alg:setting_m_generation}). The empirical length distribution $\mathcal{L}_v$ is estimated in the same way as in Setting S, ensuring that the overall output remains aligned with human length patterns.

The key difference lies in how generation is allocated. Rather than assigning all length bins to a single model, we construct an allocation plan $\mathcal{Q}_v$ that distributes length bins and corresponding comment counts across multiple models while preserving the aggregate length distribution. Each model $M_k$ generates a subset of comments conditioned on the video context and a specific length bin, and the outputs are combined to form the final comment set.  Box \ref{prompt_without_aspect} shows the prompt used for both S and M settings.

\begin{tcolorbox}[
    colback=blue!10, 
    colframe=blue!40!black, 
    title=Prompt for LLM comment generation (without aspect) : $S$ and $M$ settings, 
    label=prompt_without_aspect,
    rounded corners
]\footnotesize

\textbf{Task:} Generate \texttt{\{num\_comments\}} diverse and realistic YouTube comments for the given video.

\vspace{4pt}
\textbf{Video Details:}
\begin{itemize}
    \item Channel Name: \texttt{\{channel\_name\}}
    \item Domain: \texttt{\{domain\}}
    \item Video Title: \texttt{\{video\_title\}}
    \item Video Post Date: \texttt{\{video\_post\_date\}}
\end{itemize}

\vspace{4pt}
\textbf{Video Summary (for context):} \\
\texttt{\{video\_summary\}}

\vspace{4pt}
\textbf{Guidelines:}
\begin{itemize}
    \item Generate exactly \texttt{\{num\_comments\}} comments.
    \item Each comment should be between \texttt{\{min\_len\}} and \texttt{\{max\_len\}} words.
    \item Comments should sound natural and human-like, as if written by different YouTube users.
    \item Comments can take ideas from what the video mentions (based on the summary).
    \item Use a mix of sentence structures and varying levels of enthusiasm or critique.
    \item Comments can reflect opinions about the content, delivery, message, or production quality.
\end{itemize}

\vspace{4pt}
\textbf{Output Format:}
\begin{itemize}
    \item Return the comments \textbf{only} in the following numbered format:
\end{itemize}

\begin{itemize}
    \item 1. \texttt{<comment\_1>}
    \item 2. \texttt{<comment\_2>}
    \item $\dots$
    \item \texttt{\{num\_comments\}}. \texttt{<comment\_\{num\_comments\}>}
\end{itemize}

Do not include any text or explanation outside the numbered comments.

\end{tcolorbox}

\begin{algorithm}[h]
\caption{Comment Generation for Setting M: Multi-LLM, No Aspects}
\label{alg:setting_m_generation}
\begin{algorithmic}

\State \textbf{Input:} Video $v$, human comment set $\mathcal{H}_v$, model pool $\mathcal{M}=\{M_1,\dots,M_K\}$, target number of comments $n$
\State \textbf{Output:} Generated comment set $\mathcal{G}^{(\mathrm{M})}_v=\{g_1,\dots,g_n\}$

\State Select a subset of models $\mathcal{M}_v \subseteq \mathcal{M}$
\Comment{models used for video $v$}

\State Estimate empirical length-bin distribution from $\mathcal{H}_v$:
\[
\mathcal{L}_v = \{(\ell_b, \pi_b)\}_{b=1}^{B},
\]
\Statex \hspace{\algorithmicindent} where $\ell_b$ denotes a word-length bin and $\pi_b$ is its empirical proportion.

\State Allocate length bins and comment counts across models:
\[
\mathcal{Q}_v
=
\{(M_k, \ell_b, n_{k,b}) : M_k \in \mathcal{M}_v,\; \ell_b \in \mathcal{L}_v\},
\]
\Statex \hspace{\algorithmicindent} such that $\sum_{(k,b)} n_{k,b} = n$ and the aggregate length distribution follows $\mathcal{L}_v$.

\State $\mathcal{G}^{(\mathrm{M})}_v \gets \emptyset$

\For{$(M_k, \ell_b, n_{k,b}) \in \mathcal{Q}_v$}
    \State $\mathcal{G}_{k,b} \sim P_{M_k}^{(n_{k,b})}(\cdot \mid v, \ell_b)$
    \Comment{single prompt generates $n_{k,b}$ comments}
    \State $\mathcal{G}^{(\mathrm{M})}_v \gets \mathcal{G}^{(\mathrm{M})}_v \cup \mathcal{G}_{k,b}$
\EndFor

\State \Return $\mathcal{G}^{(\mathrm{M})}_v$

\end{algorithmic}
\end{algorithm}

\begin{algorithm}[h]
\caption{Comment Generation for Setting $S_A$: Single-LLM, With Aspects}
\label{alg:setting_sa_generation}
\begin{algorithmic}

\State \textbf{Input:} Video $v$, aspect set $\mathcal{A}_v=\{a_1,\dots,a_{A_v}\}$, human comment set $\mathcal{H}_v$, model pool $\mathcal{M}=\{M_1,\dots,M_K\}$, target number of comments $n$
\State \textbf{Output:} Generated comment set $\mathcal{G}^{(\mathrm{S_A})}_v=\{g_1,\dots,g_n\}$

\State Sample a single model $M_k \sim \mathrm{Uniform}(\mathcal{M})$
\Comment{single model for all aspect-conditioned generations}

\State Get empirical aspect distribution (from Algorithm \ref{alg:aspect_extraction}):
\[
\Lambda_v=\{(a_j,\lambda_j)\}_{j=1}^{A_v},
\]
\Statex \hspace{\algorithmicindent} where $\lambda_j$ is the proportion of human comments assigned to aspect $a_j$.

\State $\mathcal{G}^{(\mathrm{S_A})}_v \gets \emptyset$

\For{$(a_j,\lambda_j) \in \Lambda_v$}
    \State $n_j \gets \lfloor n \cdot \lambda_j \rceil$
    \Comment{comments allocated to aspect $a_j$}

    \State Estimate aspect-specific length-bin distribution:
    \[
    \mathcal{L}_{v,a_j} = \{(\ell_b,\pi_{j,b})\}_{b=1}^{B_j},
    \]
    \Statex \hspace{\algorithmicindent} where $\pi_{j,b}$ is the proportion of comments for aspect $a_j$ in length bin $\ell_b$.

    \For{$(\ell_b,\pi_{j,b}) \in \mathcal{L}_{v,a_j}$}
        \State $n_{j,b} \gets \lfloor n_j \cdot \pi_{j,b} \rceil$
        \Comment{comments allocated to aspect--length pair}

        \State $\mathcal{G}_{j,b} \sim P_{M_k}^{(n_{j,b})}(\cdot \mid v, a_j, \ell_b)$
        \Comment{single prompt conditioned on video, aspect, and length}

        \State $\mathcal{G}^{(\mathrm{S_A})}_v \gets \mathcal{G}^{(\mathrm{S_A})}_v \cup \mathcal{G}_{j,b}$
    \EndFor
\EndFor


\State \Return $\mathcal{G}^{(\mathrm{S_A})}_v$

\end{algorithmic}
\end{algorithm}

This design allows us to isolate the effect of cross-model heterogeneity while controlling for prompt structure and length. By aggregating outputs from multiple LLMs, the setting captures a broader range of generative behaviors compared to Setting $S$, providing a stronger baseline for evaluating the role of model diversity in comment generation.

\paragraph{Setting $S_A$: Single-LLM, With Aspects}
This setting extends Setting $S$ by introducing explicit aspect conditioning while retaining a single model for generation. For each video $v$, we sample one model $M_k$ and use it consistently across all generations (Algorithm~\ref{alg:setting_sa_generation}). In addition to the global length distribution, we incorporate the empirical aspect distribution $\Lambda_v$, obtained from the aspect extraction pipeline, to guide how comments are allocated across different discussion aspects.

The total number of comments $n$ is first distributed across aspects according to their observed proportions $\lambda_j$. Within each aspect $a_j$, we further estimate an aspect-specific length distribution $\mathcal{L}_{v,a_j}$, allowing us to capture how different types of discussions vary in length. The model then generates comments conditioned on the video context, the target aspect, and the corresponding length bin. All generated subsets are aggregated to form the final comment set.

Compared to Setting $S$ and $M$, this design introduces a structured constraint on content by explicitly modeling what is being discussed, while still isolating the behavior of a single model. It enables us to assess how well a single LLM can follow a diverse set of human-derived aspects under realistic length and distributional constraints.

\paragraph{Setting $M_A$: Multi-LLM, With Aspects}
This setting combines the two sources of diversity explored in earlier settings by introducing both aspect conditioning and multi-model generation. Building on Settings  $M$ and $S_A$, we select a subset of models $\mathcal{M}_v$ for each video and jointly allocate them across aspects and comment counts (Algorithm~\ref{alg:setting_ma_generation}). The empirical aspect distribution $\Lambda_v$ is used to determine how many comments should be generated for each aspect, preserving alignment with human discussion patterns. Box \ref{prompt_with_aspect} shows the prompt used for both $S_A$ amd $M_A$ settings.

A key design choice is to enforce structured specialization across models. Each model $M_k$ is assigned to \textbf{exactly one aspect} via a mapping $\sigma: \{1,\dots,K\} \to \mathcal{A}$, where we ensure $K > A$. As a result, each model generates comments only for its assigned aspect $\sigma(k)$. Let $\omega_k$ denote the per-model weights, where $\sum_{k=1}^K \omega_k = 1$ and $\omega_k \ge 0$. The induced aspect coverage is
\[
\lambda_a
\;=\;
\sum_{k \in \mathcal{K}_a} \omega_k,
\quad a \in \mathcal{A},
\]
where $\mathcal{K}_a$ is the set of models assigned to aspect $a$, ensuring $\sum_a \lambda_a = 1$. The overall generation process can then be expressed as
\[
\mathcal{G}^{(\mathrm{M_A})}(\cdot \mid v) =
\sum_{a \in \mathcal{A}}\; \sum_{k \in \mathcal{K}_a}
\omega_{k}\, P_{M_k}(\cdot \mid v, a).
\]

In practice, we operationalize this formulation by allocating comment counts across model–aspect pairs according to $\lambda_a$, and further distributing them across aspect-specific length bins as in Setting $S_A$. Each model generates comments conditioned on the video context, its assigned aspect, and the target length bin, and the outputs are aggregated to form the final comment set.

\begin{tcolorbox}[
    colback=blue!10, 
    colframe=blue!40!black, 
    title=Prompt for LLM comment generation (with aspect) : $S_A$ and $M_A$ settings, 
    label = prompt_with_aspect,
    rounded corners
]\footnotesize
\textbf{Task:} Generate \texttt{\{num\_comments\}} diverse, realistic YouTube comments for the following video. The comments must reflect how \textit{human users} are reacting to this topic, based on the provided comment summary.

\vspace{5pt}
\textbf{Video Details}
\begin{itemize}
    \item Channel Name: \texttt{\{channel\_name\}}
    \item Domain: \texttt{\{domain\}}
    \item Video Title: \texttt{\{video\_title\}}
    \item Video Post Date: \texttt{\{video\_post\_date\}}
\end{itemize}

\textbf{Video Summary (topic context):} \\
\texttt{\{video\_summary\}}

\vspace{5pt}
\textbf{Human Comment Summary (Primary Signal):} \\
The following summarizes how real users are reacting in the comment section.
Generated comments \textbf{must} be grounded in these themes, opinions, tones, and conflicts:

\vspace{2pt}
\texttt{\{comment\_summary\}}

\vspace{5pt}
\textbf{Length Guidelines:} \\
You must follow the length distribution below exactly:

\vspace{2pt}
\texttt{\{length\_text\}}

\vspace{5pt}
\textbf{Style / Aspect Guidance:} \\
\texttt{\{style\_block\}}

\vspace{5pt}
\textbf{Generation Rules:}
\begin{itemize}
    \item Generate exactly \texttt{\{num\_comments\}} comments---no more, no fewer.
    \item Comments \textbf{must} be topically aligned with the human comment summary.
    \item Maintain strong diversity:
    \begin{itemize}
        \item Different stances, emotional intensity, wording, and rhetorical styles
        \item Avoid repetitive phrasing or near-duplicates
    \end{itemize}
    \item Comments should sound like they come from different users.
    \item Tone may include emotion, sarcasm, humor, bias, or criticism \textbf{only if supported by the comment summary or aspect guidance}.
    \item Do \textbf{not} introduce facts or themes not implied by the video or comment summary.
    \item Strictly follow the length and style guidelines above.
\end{itemize}

\vspace{5pt}
\textbf{Output Format (Mandatory):} \\
Return the comments \textbf{only} in the following numbered format:

\begin{itemize}
    \item 1. \texttt{<comment\_1>}
    \item 2. \texttt{<comment\_2>}
    \item $\dots$
    \item \texttt{\{num\_comments\}}. \texttt{<comment\_\{num\_comments\}>}
\end{itemize}

Do \textbf{not} include explanations, headers, or any text outside the numbered comments.

\end{tcolorbox}

Compared to all previous settings, this design provides the most comprehensive control over both semantic coverage and generative diversity. It captures heterogeneity across models while enforcing structured alignment with human-derived aspects, allowing us to study the combined effect of model diversity and aspect conditioning in a unified and interpretable framework.

\begin{algorithm}[h]
\caption{Comment Generation for Setting $M_A$: Multi-LLM, With Aspects}
\label{alg:setting_ma_generation}
\begin{algorithmic}

\State \textbf{Input:} Video $v$, aspect set $\mathcal{A}_v=\{a_1,\dots,a_{A_v}\}$, human comment set $\mathcal{H}_v$, model pool $\mathcal{M}=\{M_1,\dots,M_K\}$, target number of comments $n$
\State \textbf{Output:} Generated comment set $\mathcal{G}^{(\mathrm{M_A})}_v=\{g_1,\dots,g_n\}$

\State Select a subset of models $\mathcal{M}_v \subseteq \mathcal{M}$
\Comment{models used for video $v$}

\State Estimate empirical aspect distribution:
\[
\Lambda_v=\{(a_j,\lambda_j)\}_{j=1}^{A_v},
\]
\Statex \hspace{\algorithmicindent} where $\lambda_j$ is the proportion of human comments assigned to aspect $a_j$.

\State Construct a surjective assignment $\sigma_v:\mathcal{M}_v \rightarrow \mathcal{A}_v$
\Comment{each model assigned to exactly one aspect}

\State For each aspect $a_j$, define assigned models:
\[
\mathcal{K}_{v,a_j}=\{M_k \in \mathcal{M}_v : \sigma_v(M_k)=a_j\}.
\]

\State Allocate comment counts across aspect--model pairs:
\[
\mathcal{Q}_v =
\{(M_k,a_j,n_{k,j}) : M_k \in \mathcal{K}_{v,a_j},\ a_j \in \mathcal{A}_v\},
\]
\Statex \hspace{\algorithmicindent} such that $\sum_{k,j} n_{k,j}=n$ and $\sum_{M_k \in \mathcal{K}_{v,a_j}} n_{k,j}\approx n\lambda_j$.

\State $\mathcal{G}^{(\mathrm{M_A})}_v \gets \emptyset$

\For{$(M_k,a_j,n_{k,j}) \in \mathcal{Q}_v$}

    \State Estimate aspect-specific length-bin distribution:
    \[
    \mathcal{L}_{v,a_j}=\{(\ell_b,\pi_{j,b})\}_{b=1}^{B_j},
    \]
    \Statex \hspace{\algorithmicindent} where $\pi_{j,b}$ is the proportion of comments for aspect $a_j$ in length bin $\ell_b$.

    \For{$(\ell_b,\pi_{j,b}) \in \mathcal{L}_{v,a_j}$}
        \State $n_{k,j,b} \gets \lfloor n_{k,j}\cdot \pi_{j,b} \rceil$
        \Comment{comments for model--aspect--length tuple}

        \State $\mathcal{G}_{k,j,b} \sim P_{M_k}^{(n_{k,j,b})}(\cdot \mid v,a_j,\ell_b)$
        \Comment{single prompt conditioned on video, aspect, and length}

        \State $\mathcal{G}^{(\mathrm{M_A})}_v \gets \mathcal{G}^{(\mathrm{M_A})}_v \cup \mathcal{G}_{k,j,b}$
    \EndFor

\EndFor


\State \Return $\mathcal{G}^{(\mathrm{M_A})}_v$

\end{algorithmic}
\end{algorithm}

\subsection{Feature extraction pipeline}
\label{subsec_feature_extractions}

\paragraph{Semantic features.}
Semantic features are designed to capture what comments discuss, independent of surface form or wording. For each human and generated comment, we compute a dense text embedding using \texttt{EmbeddingGemma} (\texttt{google/embeddinggemma-300m}) \citep{vera2025embeddinggemma}, using the default 768-dimensional representation. This choice provides a strong balance between expressiveness and computational efficiency, while remaining consistent across all datasets and generation settings.

These embeddings serve as a unified representation that supports multiple stages of our pipeline. First, they are used for BERTopic-based clustering during aspect extraction, enabling us to identify semantically coherent discussion groups. Second, they are used to measure semantic dispersion within a video through average pairwise dissimilarity among comment embeddings. Third, they support semantic coverage analysis by computing the maximum cosine similarity between each human comment and the generated set. Fourth, they enable distributional alignment comparisons between human and generated comments using metrics such as MMD-RBF. Finally, they are used in clustering and nearest-neighbor computations, including semantic deduplication in the RQ2 data quality pipeline.

By relying on a single, consistent embedding space across all these components, we ensure that semantic comparisons remain comparable and grounded, while providing a comprehensive view of content-level variation and alignment.

\paragraph{Linguistic features.}
Linguistic features capture how comments are written, focusing on form, structure, and stylistic patterns rather than topic or stance. We compute both per-comment descriptors and video-level diversity measures to obtain a comprehensive view of lexical and syntactic variation.

At the per-comment level, we extract a set of descriptive features, including word count, sentence count, mean sentence length, token-level statistics, and part-of-speech (POS) distributions. When available, we also include readability-oriented measures and sentiment scores produced by the text-descriptive pipeline. POS features are represented as distributions over tags for each comment, which are later aggregated at the video level. These aggregated representations allow us to compare human and generated comments in terms of syntactic usage and structural patterns, and are used to assess linguistic coverage and alignment.

At the video level, we compute diversity-oriented features inspired by the \texttt{diversity} package. These include \texttt{compression\_ratio}, \texttt{compression\_ratio\_pos}, \texttt{homogenization\_score}, \texttt{ttr\_1gram}, \texttt{ngram\_4\_diversity}, \texttt{self\_repetition\_score}, \texttt{template\_rate}, and \texttt{templates\_per\_token\_mean}. The compression metrics quantify redundancy in both raw text and POS sequences \citep{shaib-etal-2025-standardizing}, while homogenization \citep{anderson2024homogenization} and self-repetition scores \citep{salkar2022self_repitiion_score} capture the extent to which comments are overly similar. Type-token ratio \citep{herdan1960type} and n-gram diversity measure lexical richness, and template-based metrics estimate the prevalence of repeated syntactic patterns.

Together, these features provide a detailed characterization of linguistic variability. They allow us to assess whether generated comments exhibit natural variation in wording and structure or converge to repetitive or templated expressions, and they are used throughout the evaluation framework for analyzing diversity, coverage, and alignment.

\paragraph{Socio-pragmatic features.}
Socio-pragmatic features capture communicative function, affect, and social meaning in comments. These features are especially important because two comments can be semantically similar while differing in tone, emotion, humor, sarcasm, toxicity, or social framing. We extracted a set of classifier-based pragmatic features for both human and LLM comments. The feature pipeline stores both categorical outputs and continuous confidence/probability scores when the classifier provides them. For example, emotion is stored as a distribution over emotion labels, while humor, hate, sarcasm, and metaphor features are stored as a label plus a scalar score.

\begin{description}

\item[Sentiment]
Overall polarity of the comment derived using VADER sentiment analysis \citep{singh2021efficient} with rule-based thresholds. We assign one of four labels: \texttt{positive}, \texttt{negative}, \texttt{neutral}, or \texttt{mixed}. The \texttt{mixed} label captures comments that contain both positive and negative signals while remaining overall neutral \citep{gaspar2016beyond}.

\item[Emotion]
 Using HuggingFace \texttt{j-hartmann/emotion-english-distilroberta-base} model \citep{hartmann2022emotionenglish} for dominant emotional tone predicted. Each comment is assigned a primary emotion from \texttt{anger}, \texttt{disgust}, \texttt{fear}, \texttt{joy}, \texttt{neutral}, \texttt{sadness}, or \texttt{surprise} \citep{ekman2014expression}.

\item[Humor]
Binary indicator of whether the comment expresses humor, detected using \texttt{Humor-Research/humor-detection-comb-23} \citep{baranov-etal-2023-humor}. Comments are labeled as \texttt{humor} or \texttt{non\_humor}.

\item[Bias]
 Using \texttt{d4data/bias-detection-model} \citep{fang2024bias} for the detection of biased or non-biased language. The model provides a \texttt{bias\_label} and \texttt{bias\_score}, with model-native categories distinguishing biased vs.\ non-biased content. We use \texttt{Non-biased} as the default fallback category.

\item[Sarcasm]
 Using \texttt{mrm8488/t5-base-finetuned-sarcasm-twitter} \citep{romero2020t5sarcasm} for binary detection of sarcastic intent. Each comment is labeled as \texttt{sarcasm} or \texttt{not\_sarcasm}.

\item[Hate Speech]
Detection of toxic or hateful content using \texttt{facebook/roberta-hate-speech-dynabench-r4-target}. The model identifies whether the comment contains hate-related signals.

\item[Metaphor]
Figurative language detection using the token-classification model \texttt{lwachowiak/Metaphor-Detection-XLMR} \citep{wachowiak2022metaphor}. We compute the proportion of metaphor-labeled tokens and assign a binary label: \texttt{metaphor} if any token is metaphorical, otherwise \texttt{no\_metaphor}.

\item[Formality]
Degree of linguistic formality predicted using \texttt{s-nlp/roberta-base-formality-ranker} \citep{babakov2023formality}. Each comment is labeled as \texttt{formal} or \texttt{informal}, along with a corresponding score.

\item[Stereotype]
Detection of stereotypical or anti-stereotypical content using \texttt{wu981526092/Sentence-Level-Stereotype-Detector} \citep{wu2023stereotype}. Labels include \texttt{unrelated}, \texttt{stereotype}, and \texttt{anti\_stereotype} variants across gender, race, profession, and religion.

\end{description}

\subsection{Additional datasets and ablations}
\label{appendix_subsec_ablation}
\paragraph{Ablation study: Setting $S_D$.}
The $S_D$ setting is designed as an ablation of the single-model baseline to isolate the effect of decoding stochasticity. It follows the same generation setup as Setting S, where each video is paired with a single model and the prompt includes video metadata, a summary, the target number of comments, and length constraints derived from human data. Unlike aspect-conditioned settings, $S_D$ does not introduce aspect or style information. The only intervention is to vary the sampling parameters across generation calls.

For each domain dataset (\texttt{news}, \texttt{pop\_culture}, and \texttt{tech}), we use the same prompt and assignment structure as the baseline, but sample a new set of generation parameters for each $(\text{Video\_id}, \text{Model\_id}, \text{Provider}, \text{length bin})$ tuple. The generated outputs are parsed into numbered comments and stored along with metadata, length guidelines, an extraction-success flag, and the sampled \texttt{generation\_settings}. This design ensures that the process is resumable and auditable, since previously completed generations are skipped.

We vary the temperature over the set $\{0.3, 0.5, 0.7, 1.0\}$. When supported, we also sample nucleus parameters with $\texttt{top\_p} \in \{0.7, 0.8, 0.9, 0.95\}$, and for TogetherAI-style providers we optionally include $\texttt{top\_k} \in \{10, 20, 50, 100\}$. Parameter usage follows provider-specific constraints. OpenAI, DeepSeek, and xAI use both temperature and $\texttt{top\_p}$. Google uses temperature and \texttt{topP}. Microsoft/Phi uses temperature only. Anthropic uses either temperature or $\texttt{top\_p}$ when both are available, selecting one per request. Together-hosted models support all three parameters, with defaults applied when not explicitly sampled.

To improve robustness, we request a small surplus of comments in each call ($\texttt{EXTRA\_COMMENTS} = 5$) beyond the target, like previous settings. This helps mitigate formatting or parsing failures when extracting numbered outputs.

Overall, $S_D$ provides a controlled way to test whether increased diversity can be achieved through stochastic decoding alone, without introducing multiple models or aspect conditioning. This allows us to separate the contribution of sampling variability from the structural design choices explored in other settings.

\paragraph{Ablation study: newer videos and LLM as planner.}
Our primary setup derives aspects from human comments on historical videos, which may implicitly benefit LLMs through prior exposure. To evaluate generalization, we design an ablation on newer videos where such signals are not available. In this setting, aspects are not extracted from human comments but instead inferred using an LLM acting as a planner. This setup reflects a realistic deployment scenario where comments must be generated for newly published content without prior discussion data.

We consider videos published between \texttt{2025-11-01} and \texttt{2026-01-06}, which fall beyond the pretraining cutoff of the models used in our study. \textbf{We explicitly disable web search and tool use during both aspect planning and comment generation to prevent access to post-cutoff knowledge}. Videos are collected from the same set of channels as in the main study, using identical filtering criteria to ensure consistency. We select approximately 100 videos per domain and apply the same preprocessing pipeline to extract and clean comments. This results in 18,509, 15,788, and 7,771 comments for the news, pop culture, and tech domains, respectively.

\begin{tcolorbox}[
    colback=blue!10,
    colframe=blue!40!black,
    title=Prompt for aspect generation using LLM without human comments,
    label=prompt_aspect_generation,
    rounded corners
]\footnotesize

\textbf{Task:} You are an expert analyst of online comment behavior and discussion dynamics.

\vspace{4pt}
\textbf{Important Constraints:}
\begin{itemize}
    \item The video was published after your training cutoff.
    \item Do \textbf{not} use web search or assume real-world knowledge.
    \item Reason only from the provided metadata and general human behavior.
    \item Goal: simulate realistic, diverse comment discussions.
\end{itemize}

\vspace{4pt}
\textbf{Video Information:}
\begin{itemize}
    \item Title: \texttt{\{video\_title\}}
    \item Channel: \texttt{\{channel\_info\}}
    \item Summary: \texttt{\{video\_summary\}}
    \item Post Date: \texttt{\{video\_post\_date\}}
\end{itemize}

\vspace{4pt}
\textbf{Comment Target:}
\begin{itemize}
    \item Total comments (approx.): \texttt{\{approx\_num\_comments\}}
    \item Number of clusters: \texttt{\{num\_clusters\}}
    \item Length statistics: \texttt{\{length\_stats\}}
\end{itemize}

\vspace{4pt}
\textbf{Task Description:}
Infer \texttt{\{num\_clusters\}} realistic discussion clusters for first-time viewers.

\textbf{Zero-shot setting:}
\begin{itemize}
    \item No prior comments or labels
    \item No external knowledge
\end{itemize}

Clusters should:
\begin{itemize}
    \item Capture diverse perspectives and engagement styles
    \item Be imbalanced (2--3 dominant clusters)
    \item Include both shallow reactions and deeper discussions
\end{itemize}

\vspace{4pt}
\textbf{For Each Cluster, Provide:}
\begin{itemize}
    \item \textbf{cluster\_id:} short, human-readable title
    \item \textbf{cluster\_description:} main discussion themes
    \item \textbf{style\_description:} sentiment, tone, formality, stereotypes, mild toxicity (if any)
    \item \textbf{number\_of\_comments:}
    \begin{itemize}
        \item Sum across clusters must equal \texttt{\{approx\_num\_comments\}}
        \item Each cluster must have $\geq 5$ comments
    \end{itemize}
    \item \textbf{length\_guideline:} min, 25\%, 50\%, 75\%, max
\end{itemize}

\vspace{4pt}
\textbf{Global Consistency Requirements:}
\begin{itemize}
    \item Total comments must equal \texttt{\{approx\_num\_comments\}}
    \item Weighted length distributions should approximately match global statistics
    \item Do not copy global stats into every cluster
    \item Short clusters should be concise; analytical clusters longer
\end{itemize}

\vspace{4pt}
\textbf{Output Format:}
Return \textbf{only} a valid JSON array of length \texttt{\{num\_clusters\}}.  
Do not include explanations or extra text.

\end{tcolorbox}

To generate aspects in this zero-shot setting, we use \texttt{GPT-5.2} as a planner model, which is not included in the generator pool. The planner receives only video metadata, summary, and global length statistics derived from human comments. It produces a set of discussion clusters with associated descriptions, style characteristics, comment counts, and length guidelines. This design ensures that aspect inference is grounded in general patterns of human behavior rather than memorized content.

For generations, we have focused on the M, $S_A$, and $M_A$ settings, as Setting S showed limited diversity in earlier experiments. The M setting generates comments from multiple models using only video context, while $S_A$ and $M_A$ incorporate the planner-generated aspects and length guidelines. The prompt structure remains consistent with the original setup to isolate the effect of the aspect source. Box \ref{prompt_aspect_generation} presents the prompt used for the planner model.

This ablation provides a comprehensive test of whether our framework can generalize beyond historical data, and whether LLM-inferred aspects can effectively substitute for human-derived discussion structure in unseen scenarios.

\subsection{Evaluation metrics}
\label{app_evaluation_metric}
Evaluating diversity in LLM-generated social discourse requires moving beyond single-metric assessments and accounting for multiple facets of comment-set behavior. 
We compute all metrics at the video level, since a video provides a natural boundary for the comment space. Comments under the same video are grounded in the same context, which makes diversity, coverage, and alignment comparisons meaningful. We use $H_v$ to denote the set of human comments for video $v$, and $M_v$ to denote the generated comments from a given LLM setting for the same video. A comment is denoted by $c$, and $\phi(c)$ denotes its semantic embedding from \texttt{EmbeddingGemma} \citep{vera2025embeddinggemma}. Unless otherwise specified, higher values indicate better performance. For coverage and alignment, all comparisons are made with respect to the human comment set $H_v$, because our goal is to assess how well each LLM setting captures or matches the human discussion distribution.

\paragraph{Semantic Dispersion.}
Semantic dispersion measures how broadly comments within a video spread across the semantic space. We use embedding-based pairwise dissimilarity because it directly captures topical and meaning-level variation beyond surface word overlap. This is a common choice for measuring diversity in generated text \citep{cann2023using, guo2025benchmarking, zhang2025verbalized}, and it is especially suitable for comments where different wordings may express similar meanings. For a comment set $X_v = \{c_1,\dots,c_N\}$, we define semantic dispersion as

\[
\mathcal{D}_{\text{sem}}(X_v)
=
\frac{1}{N(N-1)}
\sum_{i \neq j}
\left(
1 -
\max\left(0, \cos\big(\phi(c_i), \phi(c_j)\big)\right)
\right).
\]

Here, $N$ is the number of comments in $X_v$, and $\cos(\cdot,\cdot)$ is cosine similarity. We clip negative cosine similarity to zero to avoid over-rewarding comments that are simply opposite in the embedding space rather than meaningfully diverse.

\paragraph{Linguistic Dispersion.}
Linguistic dispersion captures how varied the comments are in wording, lexical choice, and syntactic structure. Unlike semantic dispersion, which focuses on meaning, this group of metrics measures whether generated comments collapse into repeated phrases, templates, or POS patterns. We use several complementary metrics because no single linguistic statistic fully captures lexical and syntactic variety.

First, we compute n-gram diversity, following prior work on generation diversity \citep{padmakumar2024does}. Let $X_v^{\oplus}$ denote the concatenation of all comments in $X_v$, and let $\mathcal{N}_n(X_v^{\oplus})$ be the multiset of $n$-grams. We compute average n-gram diversity across $n=1$ to $4$ as

\[
\mathrm{NGD}(X_v)
=
\frac{1}{4}
\sum_{n=1}^{4}
\frac{
|\mathrm{unique}(\mathcal{N}_n(X_v^{\oplus}))|
}{
|\mathcal{N}_n(X_v^{\oplus})|
}.
\]

We also use the self-repetition score \citep{salkar2022self_repitiion_score}, which penalizes repeated expressions within the text. If $N_i$ denotes the repetition count for repeated unit $i$ across $K$ repeated units, then

\[
\mathrm{SRS}(X_v)
=
\log
\left(
\sum_{i=1}^{K} (N_i + 1)
\right).
\]

To capture redundancy at the sequence level, we compute the compression ratio over both raw text and POS-tag sequences \citep{shaib-etal-2025-standardizing}. Let $\mathrm{comp}(\cdot)$ denote a compression function. We define

\[
\mathrm{CR}(X_v)
=
\frac{
\mathrm{size}(X_v^{\oplus})
}{
\mathrm{size}(\mathrm{comp}(X_v^{\oplus}))
}.
\]

A higher compression ratio indicates more redundancy, since repetitive text is easier to compress. We also compute the same measure over POS-tag sequences to capture syntactic repetition.

Finally, we compute a homogenization score based on pairwise similarity between comments, inspired by ROUGE-style similarity measures \citep{lin2004rouge}. For a comment-level similarity function $\mathrm{sim}(\cdot,\cdot)$, we define

\[
\mathrm{Hom}(X_v)
=
\frac{1}{|X_v|(|X_v|-1)}
\sum_{\substack{c,c' \in X_v \\ c \neq c'}}
\mathrm{sim}(c,c').
\]

Together, these metrics allow us to distinguish genuinely varied language from outputs that are semantically different but linguistically repetitive, or linguistically diverse but semantically narrow.

\paragraph{Pragmatic Dispersion.}
Pragmatic dispersion measures variation in socio-pragmatic attributes such as sentiment, emotion, formality, humor, sarcasm, bias, stereotype, and metaphor. These features capture how comments communicate, not only what they discuss. We use the Simpson diversity index because it is widely used in ecology and linguistics to measure categorical evenness \citep{simpson1949measurement}. It can be interpreted as the probability that two randomly selected comments differ in their pragmatic category.

For a pragmatic feature $r$ with category set $\mathcal{Y}_r$, let $p_y$ be the proportion of comments in category $y$. The Simpson diversity for feature $r$ is

\[
\mathcal{D}_{\text{prag}}^{(r)}(X_v)
=
1 -
\sum_{y \in \mathcal{Y}_r} p_y^2.
\]

When multiple pragmatic features are used, we average across the feature set $\mathcal{R}$:

\[
\mathcal{D}_{\text{prag}}(X_v)
=
\frac{1}{|\mathcal{R}|}
\sum_{r \in \mathcal{R}}
\mathcal{D}_{\text{prag}}^{(r)}(X_v).
\]

This gives a compact measure of whether a comment set contains a range of communicative styles and social signals, rather than concentrating on a few dominant categories.

\paragraph{Semantic Coverage.}
Semantic coverage asks how much of the human discussion space is captured by an LLM setting. Unlike dispersion, which only measures variation within generated comments, coverage explicitly compares generated comments against the human distribution. Our primary metric is weighted manifold recall, adapted from manifold precision and recall \citep{kynkaanniemi2019manifold_precision}. The rationale is that human comments often form clusters corresponding to topical themes or discussion modes, and a good generated set should cover not only frequent themes but also less dominant ones.

We first cluster the human comments $H_v$ in the embedding space using K-means. Let $\mathcal{K}(H_v)$ be the set of human clusters, and let $w_k$ be the number of human comments in cluster $k$. Let $\mathcal{K}_k$ denote the region of semantic space represented by cluster $k$. The weighted semantic coverage of $M_v$ with respect to $H_v$ is

\[
\mathcal{C}_{\text{sem}}^{\text{manifold}}(M_v \mid H_v)
=
\frac{
\sum_{k \in \mathcal{K}(H_v)}
\mathbf{1}[\mathcal{K}_k \cap M_v \neq \emptyset] \cdot w_k
}{
\sum_{k \in \mathcal{K}(H_v)} w_k
}.
\]

Here, $\mathbf{1}[\cdot]$ is an indicator function. A cluster is considered covered if at least one generated comment falls into that cluster region. The cluster weights ensure that coverage reflects the empirical importance of each human discussion region.

We also compute MaxSim recall as a softer nearest-neighbor measure. For each human comment, we find the most similar generated comment in the same video and average these similarities:

\[
\mathcal{C}_{\text{sem}}^{\text{MaxSim}}(M_v \mid H_v)
=
\frac{1}{|H_v|}
\sum_{h \in H_v}
\max_{m \in M_v}
\cos\big(\phi(h), \phi(m)\big).
\]

Finally, we compute thresholded recall. For each video, we choose a video-specific threshold $\tau_v$ using an elbow method over the human-to-generated similarity distribution. This avoids imposing a single global threshold across videos with different semantic breadth. The metric is

\[
\mathcal{C}_{\text{sem}}^{\tau}(M_v \mid H_v)
=
\frac{1}{|H_v|}
\sum_{h \in H_v}
\mathbf{1}
\left[
\max_{m \in M_v}
\cos\big(\phi(h), \phi(m)\big)
\geq
\tau_v
\right].
\]

Together, these coverage metrics measure whether generated comments recover the semantic regions occupied by human comments, both at the cluster level and at the individual comment level.

\paragraph{Linguistic Coverage.}
Linguistic coverage measures whether generated comments reproduce the linguistic patterns observed in human comments. We focus on POS patterns and n-gram patterns because they capture complementary aspects of language use: POS reflects syntactic structure, while n-grams reflect lexical and phrase-level conventions. Our primary linguistic coverage score is the average of POS coverage and n-gram coverage.

Let $\mathcal{P}(H_v)$ and $\mathcal{P}(M_v)$ denote the sets of POS patterns observed in human and generated comments, respectively. Let $f_H(g)$ be the frequency of pattern $g$ in human comments. POS coverage is

\[
\mathcal{C}_{\text{ling}}^{\text{POS}}(M_v \mid H_v)
=
\frac{
\sum_{g \in \mathcal{P}(H_v)}
\mathbf{1}[g \in \mathcal{P}(M_v)] \cdot f_H(g)
}{
\sum_{g \in \mathcal{P}(H_v)} f_H(g)
}.
\]

Similarly, let $\mathcal{G}(H_v)$ and $\mathcal{G}(M_v)$ denote the sets of n-grams observed in human and generated comments. N-gram coverage is

\[
\mathcal{C}_{\text{ling}}^{\text{ngram}}(M_v \mid H_v)
=
\frac{
\sum_{g \in \mathcal{G}(H_v)}
\mathbf{1}[g \in \mathcal{G}(M_v)] \cdot f_H(g)
}{
\sum_{g \in \mathcal{G}(H_v)} f_H(g)
}.
\]

We combine the two as

\[
\mathcal{C}_{\text{ling}}(M_v \mid H_v)
=
\frac{1}{2}
\left(
\mathcal{C}_{\text{ling}}^{\text{POS}}(M_v \mid H_v)
+
\mathcal{C}_{\text{ling}}^{\text{ngram}}(M_v \mid H_v)
\right).
\]

This metric gives more weight to frequent human patterns, so it measures whether generated comments cover the linguistic structures that are actually common in human discussion.

\paragraph{Pragmatic Coverage.}
Pragmatic coverage measures whether generated comments capture the socio-pragmatic clusters present in human comments. We use a procedure analogous to semantic manifold recall, but instead of clustering embeddings, we cluster human comments using pragmatic feature vectors. Since these features are categorical, we use K-Modes clustering. This choice is appropriate because it respects categorical feature structure and avoids imposing continuous distance assumptions.

Let $\mathcal{Z}(H_v)$ be the set of pragmatic clusters obtained from human comments, and let $w_z$ be the number of human comments in pragmatic cluster $z$. Let $\mathcal{Z}_z$ denote the region or cluster assignment corresponding to $z$. Pragmatic coverage is defined as

\[
\mathcal{C}_{\text{prag}}(M_v \mid H_v)
=
\frac{
\sum_{z \in \mathcal{Z}(H_v)}
\mathbf{1}[\mathcal{Z}_z \cap M_v \neq \emptyset] \cdot w_z
}{
\sum_{z \in \mathcal{Z}(H_v)} w_z
}.
\]

This metric asks whether generated comments cover the human distribution of communicative styles, affective patterns, and social signals, rather than only matching topical content.

\paragraph{Semantic Alignment.}
Semantic alignment measures how closely the generated semantic distribution matches the human semantic distribution. We use \textbf{Maximum Mean Discrepancy (MMD) with an RBF kernel} as the primary metric because it is non-parametric, theoretically grounded, and does not require covariance estimation. This is important because the embedding space is high-dimensional, with 768 dimensions, while each video has a moderate number of comments. In this setting, metrics such as JSD are not appropriate without additional discretization or density estimation.

Let $h,h' \sim H_v$ be human comments and $m,m' \sim M_v$ be generated comments. Let $k(\cdot,\cdot)$ denote the RBF kernel applied to embeddings. The squared MMD is

\[
\begin{aligned}
\mathrm{MMD}^2(H_v,M_v)
&=
\mathbb{E}_{h,h' \sim H_v}
\left[
k\big(\phi(h),\phi(h')\big)
\right]
-
2\mathbb{E}_{h \sim H_v,\,m \sim M_v}
\left[
k\big(\phi(h),\phi(m)\big)
\right] \\
&\quad+
\mathbb{E}_{m,m' \sim M_v}
\left[
k\big(\phi(m),\phi(m')\big)
\right].
\end{aligned}
\]

Lower values indicate stronger semantic alignment between human and generated comments. 

\paragraph{Linguistic Alignment.}
Linguistic alignment measures whether generated comments match the distribution of human linguistic features. We focus on three compact feature distributions: POS categories, top frequent n-grams, and word-length distribution. These feature spaces are small enough to support stable distributional comparison, so we use Jensen-Shannon divergence (JSD).

Let $\mathrm{JSD}_{\text{POS}}$, $\mathrm{JSD}_{\text{ngram}}$, and $\mathrm{JSD}_{\text{length}}$ denote the Jensen-Shannon divergence between human and generated distributions for POS, n-gram, and length features, respectively. We define

\[
\mathcal{A}_{\text{ling}}(M_v \mid H_v)
=
\frac{1}{3}
\left(
\mathrm{JSD}_{\text{POS}}
+
\mathrm{JSD}_{\text{ngram}}
+
\mathrm{JSD}_{\text{length}}
\right).
\]

For each feature type, JSD is computed as

\[
\mathrm{JSD}(P \parallel Q)
=
H\left(\frac{P+Q}{2}\right)
-
\frac{1}{2}H(P)
-
\frac{1}{2}H(Q),
\]

where $P$ and $Q$ are the human and generated feature distributions, and $H(\cdot)$ is Shannon entropy. Lower values indicate better linguistic alignment.

\paragraph{Pragmatic Alignment.}
Pragmatic alignment measures how closely the generated comments match human comments across socio-pragmatic feature distributions. Since each pragmatic feature is represented as a categorical distribution, JSD is a natural choice: it has a closed form, is symmetric, and does not require binning or density estimation.

Let $\mathcal{R}$ be the set of pragmatic features. For each feature $r \in \mathcal{R}$, let $P_r^{H_v}$ and $P_r^{M_v}$ denote the human and generated categorical distributions. The JSD for feature $r$ is

\[
\mathrm{JSD}_r(P_r^{H_v} \parallel P_r^{M_v})
=
H\left(
\frac{P_r^{H_v}+P_r^{M_v}}{2}
\right)
-
\frac{1}{2}H(P_r^{H_v})
-
\frac{1}{2}H(P_r^{M_v}).
\]

We then average across all pragmatic dimensions:

\[
\mathcal{A}_{\text{prag}}(M_v \mid H_v)
=
\frac{1}{|\mathcal{R}|}
\sum_{r \in \mathcal{R}}
\mathrm{JSD}_r(P_r^{H_v} \parallel P_r^{M_v}).
\]

Lower values indicate that generated comments more closely match the human distribution of sentiment, emotion, style, and other socio-pragmatic signals. Together with semantic and linguistic alignment, this provides a comprehensive view of whether generated comments match human comments not only in topic, but also in form and communicative function.

\section{RQ1 detailed findings}

\subsection{Empirical validation of modeling assumptions}
\label{subsec_empirical_validation}

In this section, we empirically validate Assumption \ref{ass:separation}, which posits lower cross-model similarity and higher within-model similarity. This assumption forms the basis of our theoretical motivation for diversity gains in multi-LLM and aspect-conditioned generation settings.

To examine this, we conduct a proof-of-concept experiment using all $K$ models listed in Table \ref{tab:model_list}. Each model is prompted with identical, content-neutral instructions to generate text. We evaluate this behavior in two complementary settings: (i) our task-specific setting of comment generation conditioned on video context, and (ii) a generic text generation setting.

For the task-specific setup, we sample 10 videos from each domain and instruct each model to generate 10 diverse comments per video under controlled length constraints. For the generic setting, we use 50 topics from the \texttt{agentlans/wikipedia-paragraphs}\footnote{\url{https://huggingface.co/datasets/agentlans/wikipedia-paragraphs}} dataset and prompt each model to generate 5 paragraphs (each consisting of approximately 10 sentences) per topic. This dual setup allows us to assess whether the observed behavior generalizes across both structured and open-ended generation tasks. We quantify similarity using cosine similarity over text embeddings. As shown in Figure~\ref{fig_appendix_proof_concept_model}, within-model similarity is consistently higher than cross-model similarity, and the difference is statistically significant (Mann--Whitney U test, $p < 0.001$) across both settings. This result supports our assumption that different models tend to produce distinct distributions even under identical prompts.
We also observe that overall similarity scores are lower in the comment generation setting compared to the generic paragraph generation setting. This suggests that the comment space is inherently more diverse, which further motivates the use of multiple models to better capture this diversity.

To further validate this observation, we analyze similarity at the provider level, as shown in Figure~\ref{fig_appendix_proof_concept_provider}. The results indicate that cross-provider similarity is significantly lower than within-provider similarity, reinforcing the idea that diversity increases when combining models from different providers. This trend is also consistent with our large-scale experiments (see Discussion, Point 2).
Finally, we compare our findings with results reported in \citep{jiangartificial}, which uses the \texttt{InfinityChat} dataset to analyze cross-model similarity. The observed trends align closely, providing additional external validation for our assumptions.

Overall, this proof-of-concept study provides strong empirical support for our modeling assumptions and justifies the design of our multi-LLM and aspect-based generation framework.

\begin{figure}[h]
    \centering
    \begin{subfigure}[t]{0.48\textwidth}
        \centering
        \includegraphics[width=\linewidth]{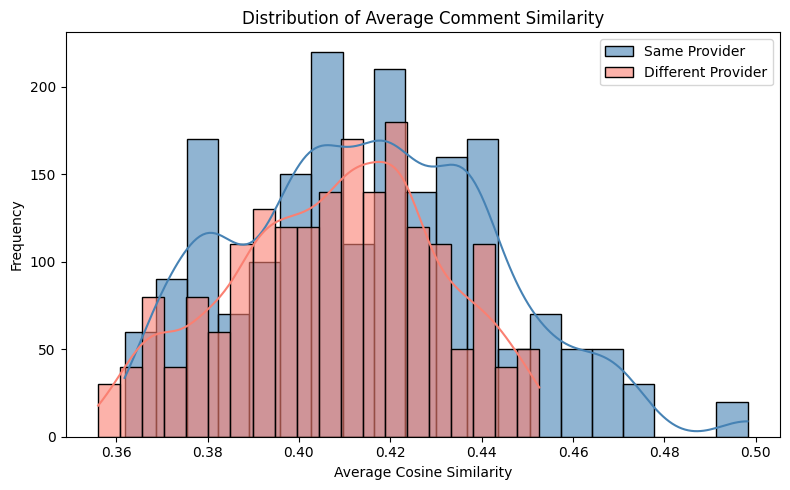}
    \end{subfigure}
    \hfill
    \begin{subfigure}[t]{0.48\textwidth}
        \centering
        \includegraphics[width=\linewidth]{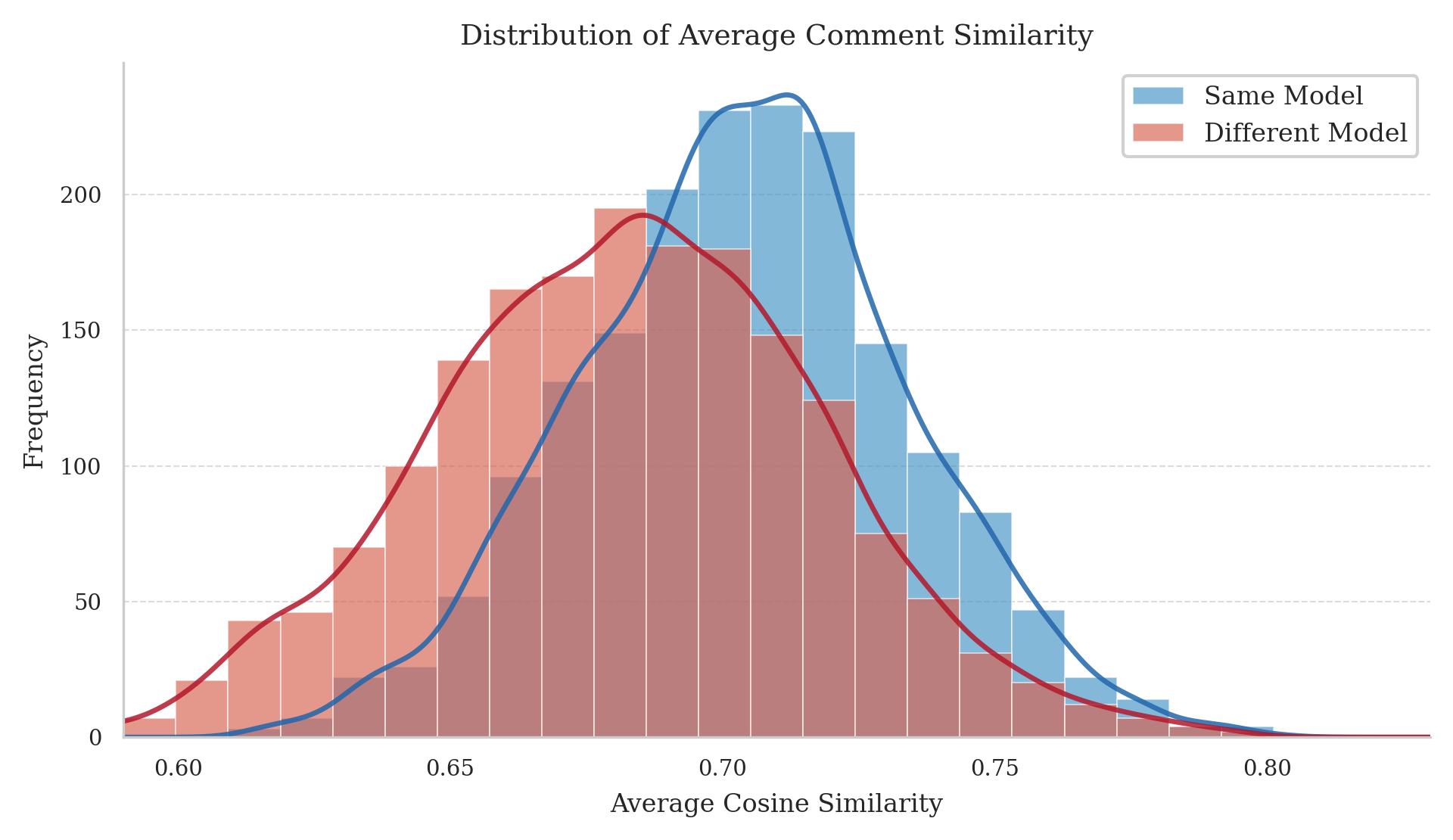}
    \end{subfigure}

    \caption{Distribution of text similarity in the proof-of-concept analysis. \textbf{Left:} Pairwise similarity of comments generated for the same video, comparing within-model and cross-model outputs. \textbf{Right:} Pairwise similarity of paragraphs generated on the same topic, comparing within-model and cross-model outputs.}
    \label{fig_appendix_proof_concept_model}
\end{figure}

\begin{figure}[h]
    \centering
    \begin{subfigure}[t]{0.48\textwidth}
        \centering
        \includegraphics[width=\linewidth]{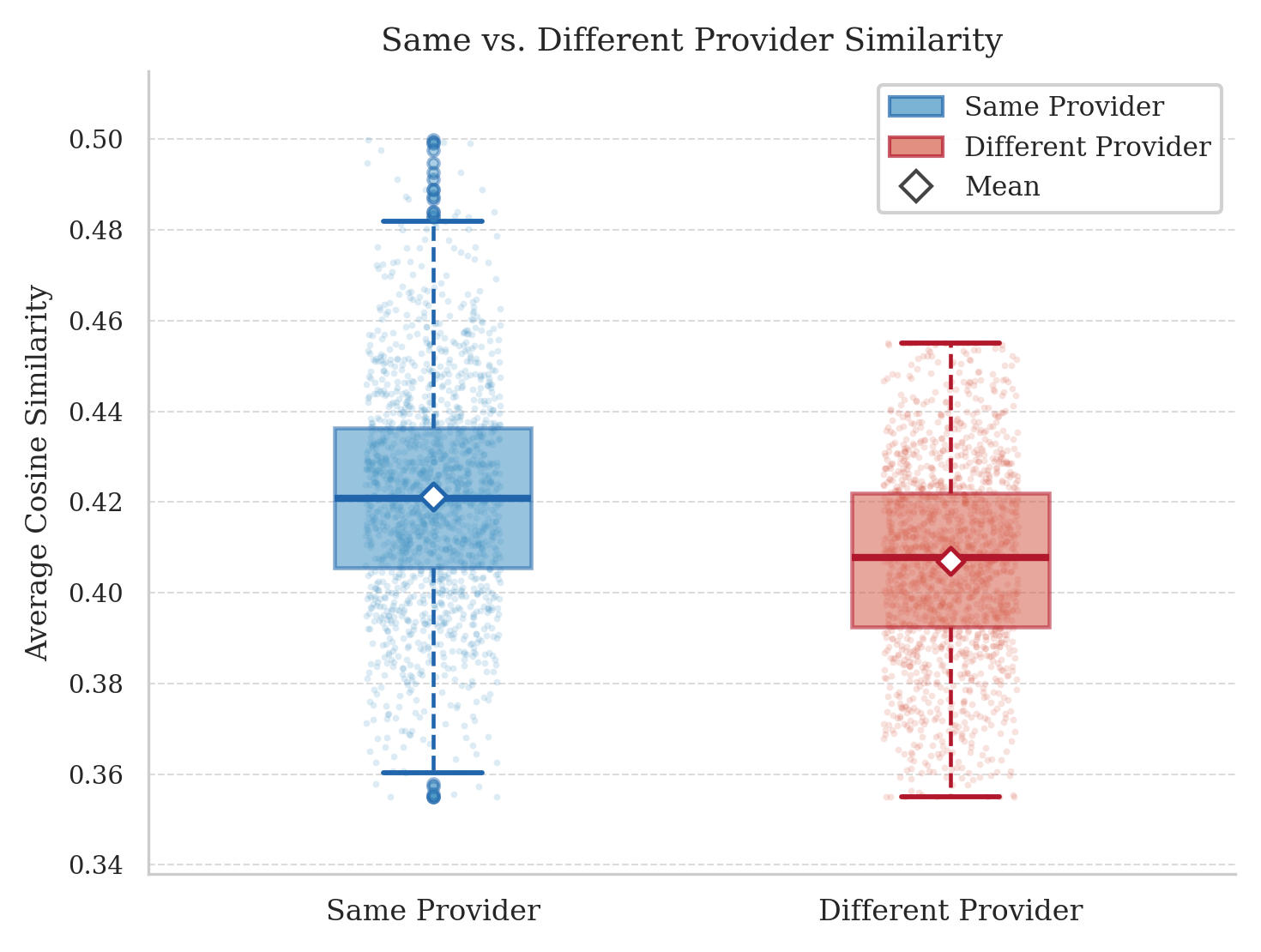}
    \end{subfigure}
    \hfill
    \begin{subfigure}[t]{0.48\textwidth}
        \centering
        \includegraphics[width=\linewidth]{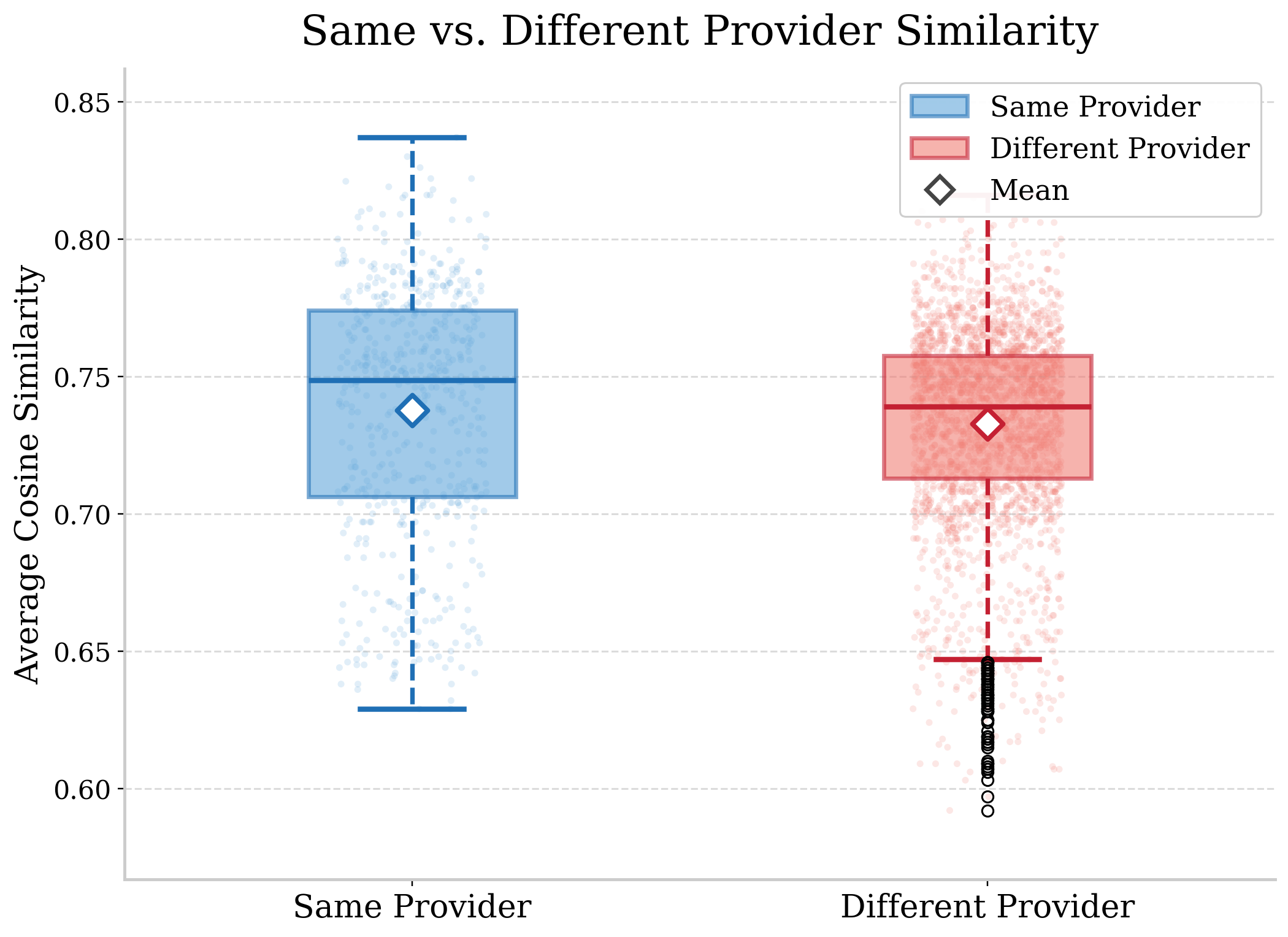}
    \end{subfigure}

\caption{Distribution of provider-level similarity in the proof-of-concept analysis. \textbf{Left:} Similarity between models grouped by provider, computed from comment generation on the same videos (derived from Figure~\ref{fig_appendix_proof_concept_model}). \textbf{Right:} Provider-level similarity (same vs.\ different) based on model similarity reported in~\citep{jiangartificial} (Table 7-11).}
    \label{fig_appendix_proof_concept_provider}
\end{figure}

\subsection{Detailed result analysis}
\label{appendix_sec_detailed_rq1}
In this subsection, we provide a detailed analysis of the generated comment distributions across the three feature spaces considered in our study: semantic, linguistic, and pragmatic. For each feature space, we evaluate the generated comments along the three complementary evaluation axes introduced in Section~\ref{app_evaluation_metric}: dispersion, coverage, and alignment. We further analyze the results separately for each domain (\textit{news}, \textit{pop\_culture}, and \textit{tech}) in addition to reporting aggregate results across all domains. This fine-grained breakdown allows us to examine not only the overall effectiveness of each generation strategy, but also how their behavior varies across different conversational environments and feature representations.

All metrics are computed independently for each video and then aggregated across videos to report the mean and standard deviation. To assess whether differences between two methods are statistically significant, we perform two-tailed paired $t$-tests on the per-video metric values, using a significance threshold of $p < 0.05$. In the result tables, ordering relationships between methods are reported only when the corresponding differences are statistically significant; methods connected by ``$\sim$'' indicate that no statistically significant difference was observed. This evaluation protocol enables a more reliable comparison of generation strategies by accounting for variability across videos rather than relying solely on aggregate scores. We report the results from our primary settings ($S, M, S_A, M_A$ and the ablation $S_D$ as well since they are computed on the same datasets).

\subsubsection{Semantic features.}

Across all three semantic evaluation axes, aspect conditioning emerges as the single most consequential factor in narrowing the gap between synthetic and human comment diversity (Tables~\ref{tab:semantic_dispersion}, \ref{tab:semantic_coverage}, and \ref{tab:semantic_alignment}). For semantic dispersion (Table~\ref{tab:semantic_dispersion}), human comments consistently achieve the highest scores ($0.65$ across all domains), while the $M_A$ setting comes closest at $0.62$, followed by $S_A$ at $0.61$. The gap narrows further in the \textit{pop\_culture} domain, where $M_A$ nearly matches human-level diversity ($0.63$ vs.\ $0.64$). In contrast, single-model generation without conditioning (S) produces the lowest dispersion scores, and the decoding-focused variant $S_D$ offers little improvement over S. This suggests that decoding specialization alone is insufficient to substantially widen the semantic spread of generated comments. Notably, the ordering $H > M_A > S_A > M > S_D > S$ holds consistently across all three domains, indicating that the hierarchy reflects a stable property of the generation strategies rather than a domain-specific artifact.

\begin{table*}[h]
\caption{Semantic dispersion results across datasets. Higher values indicate greater semantic diversity.}
\label{tab:semantic_dispersion}
\centering
\scriptsize
\setlength{\tabcolsep}{3pt}
    \resizebox{\columnwidth}{!}{
\begin{tabular}{lccccccp{4.2cm}}
\toprule
\textbf{Dataset} & \textbf{H} & \textbf{S} & \textbf{$S_D$} & \textbf{M} & \textbf{$S_A$} & \textbf{$M_A$} & \textbf{Comment} \\
\midrule
News & $0.65 \pm 0.03$ & $0.56 \pm 0.04$ & $0.57 \pm 0.04$ & $0.59 \pm 0.03$ & $0.59 \pm 0.03$ & $0.60 \pm 0.03$ & H $>$ $M_A$ $>$ $S_A$ $>$ M $>$ $S_D$ $>$ S \\
Pop & $0.64 \pm 0.03$ & $0.56 \pm 0.04$ & $0.56 \pm 0.04$ & $0.59 \pm 0.03$ & $0.62 \pm 0.03$ & $0.63 \pm 0.03$ & H $>$ $M_A$ $>$ $S_A$ $>$ M $>$ $S_D$ $>$ S \\
Tech & $0.64 \pm 0.03$ & $0.56 \pm 0.04$ & $0.56 \pm 0.04$ & $0.58 \pm 0.03$ & $0.60 \pm 0.03$ & $0.61 \pm 0.03$ & H $>$ $M_A$ $>$ $S_A$ $>$ M $>$ $S_D$ $>$ S \\
All & $0.65 \pm 0.03$ & $0.56 \pm 0.04$ & $0.57 \pm 0.04$ & $0.59 \pm 0.03$ & $0.61 \pm 0.03$ & $0.62 \pm 0.03$ & H $>$ $M_A$ $>$ $S_A$ $>$ M $>$ $S_D$ $>$ S \\
\bottomrule
\end{tabular}
}
\end{table*}

The semantic coverage and alignment metrics further reinforce this pattern while revealing an important distinction between the contributions of aspect conditioning and multi-model diversity. As shown in Table~\ref{tab:semantic_coverage}, the aspect-conditioned settings ($S_A$ and $M_A$) achieve manifold recall scores ranging from approximately $0.90$ to $0.97$ across domains, substantially outperforming the non-conditioned settings, which range from $0.53$ to $0.79$. Importantly, the gap between conditioned and non-conditioned settings is far larger than the relatively small differences between $S_A$ and $M_A$, suggesting that aspect conditioning is the primary driver of semantic coverage gains, while multi-model diversity contributes a smaller secondary improvement.

\begin{table*}[h]
\caption{Semantic coverage results across datasets. Higher values indicate better coverage of the human semantic space.}
\label{tab:semantic_coverage}
\centering
\scriptsize
\setlength{\tabcolsep}{4pt}
    \resizebox{\columnwidth}{!}{
\begin{tabular}{llcccccp{4cm}}
\toprule
\textbf{Dataset} & \textbf{Metric} & \textbf{S} & \textbf{$S_D$} & \textbf{M} & \textbf{$S_A$} & \textbf{$M_A$} & \textbf{Comment} \\
\midrule
News & Manifold recall $\uparrow$ & $0.69 \pm 0.33$ & $0.53 \pm 0.31$ & $0.76 \pm 0.54$ & $0.90 \pm 0.18$ & $0.90 \pm 0.17$ & $M_A \sim S_A > M > S \sim S_D$ \\
News & Semantic recall (MaxSim) $\uparrow$ & $0.54 \pm 0.04$ & $0.54 \pm 0.04$ & $0.55 \pm 0.04$ & $0.59 \pm 0.04$ & $0.59 \pm 0.04$ & $S_A \sim M_A > M > S_D > S$ \\
News & Semantic recall @ $\tau$ $\uparrow$ & $0.26 \pm 0.11$ & $0.27 \pm 0.11$ & $0.30 \pm 0.10$ & $0.44 \pm 0.14$ & $0.44 \pm 0.12$ & $S_A \sim M_A > M > S_D > S$ \\
\midrule
Pop\_culture & Manifold recall $\uparrow$ & $0.72 \pm 0.31$ & $0.73 \pm 0.30$ & $0.75 \pm 0.29$ & $0.93 \pm 0.14$ & $0.93 \pm 0.14$ & $M_A \sim S_A > M \sim S_D \sim S$ \\
Pop\_culture & Semantic recall (MaxSim) $\uparrow$ & $0.56 \pm 0.04$ & $0.57 \pm 0.04$ & $0.58 \pm 0.04$ & $0.62 \pm 0.04$ & $0.62 \pm 0.04$ & $S_A \sim M_A > M > S_D > S$ \\
Pop\_culture & Semantic recall @ $\tau$ $\uparrow$ & $0.21 \pm 0.10$ & $0.24 \pm 0.10$ & $0.25 \pm 0.09$ & $0.39 \pm 0.14$ & $0.37 \pm 0.12$ & $S_A > M_A > M > S_D > S$ \\
\midrule
Tech & Manifold recall $\uparrow$ & $0.78 \pm 0.28$ & $0.78 \pm 0.29$ & $0.79 \pm 0.28$ & $0.96 \pm 0.11$ & $0.97 \pm 0.09$ & $M_A \sim S_A > M \sim S \sim S_D$ \\
Tech & Semantic recall (MaxSim) $\uparrow$ & $0.57 \pm 0.03$ & $0.57 \pm 0.03$ & $0.58 \pm 0.03$ & $0.62 \pm 0.03$ & $0.62 \pm 0.03$ & $M_A \sim S_A > M > S_D > S$ \\
Tech & Semantic recall @ $\tau$ $\uparrow$ & $0.22 \pm 0.09$ & $0.24 \pm 0.10$ & $0.27 \pm 0.09$ & $0.38 \pm 0.12$ & $0.38 \pm 0.10$ & $M_A \sim S_A > M > S_D > S$ \\
\midrule
All & Manifold recall $\uparrow$ & $0.73 \pm 0.31$ & $0.76 \pm 0.30$ & $0.78 \pm 0.29$ & $0.92 \pm 0.16$ & $0.93 \pm 0.15$ & $M_A \sim S_A > M \sim S_D > S$ \\
All & Semantic recall (MaxSim) $\uparrow$ & $0.56 \pm 0.04$ & $0.56 \pm 0.04$ & $0.56 \pm 0.04$ & $0.60 \pm 0.04$ & $0.60 \pm 0.04$ & $M_A \sim S_A > M > S_D > S$ \\
All & Semantic recall @ $\tau$ $\uparrow$ & $0.23 \pm 0.10$ & $0.25 \pm 0.11$ & $0.29 \pm 0.10$ & $0.38 \pm 0.14$ & $0.38 \pm 0.12$ & $S_A \sim M_A > M > S_D > S$ \\
\bottomrule
\end{tabular}
}
\end{table*}

A similar trend appears in semantic alignment (Table~\ref{tab:semantic_alignment}), where lower values indicate better alignment with the human semantic distribution. Across all datasets, the consistent ordering $M_A < S_A < M < S_D < S$ is observed, with both $M_A$ and $S_A$ achieving near-identical alignment scores of approximately $0.02$. These values are roughly three times lower than the baseline S setting, demonstrating substantially tighter distributional alignment with human comments. Taken together, the semantic results indicate that aspect conditioning is essential for achieving both broad semantic coverage and strong alignment with human comment distributions, while multi-model generation provides an additional but comparatively smaller complementary gain.

\begin{table*}[h]
\caption{Semantic alignment results across datasets. Lower values indicate better alignment with the human semantic distribution.}
\label{tab:semantic_alignment}
\centering
\small
\setlength{\tabcolsep}{5pt}
    \resizebox{\columnwidth}{!}{
\begin{tabular}{lcccccp{4.2cm}}
\toprule
\textbf{Dataset} & \textbf{S} & \textbf{$S_D$} & \textbf{M} & \textbf{$S_A$} & \textbf{$M_A$} & \textbf{Comment} \\
\midrule
News & $0.06 \pm 0.02$ & $0.06 \pm 0.02$ & $0.05 \pm 0.01$ & $0.02 \pm 0.01$ & $0.02 \pm 0.01$ & $M_A < S_A < M < S_D < S$ \\
Pop\_culture & $0.05 \pm 0.02$ & $0.05 \pm 0.02$ & $0.04 \pm 0.02$ & $0.02 \pm 0.01$ & $0.02 \pm 0.01$ & $M_A < S_A < M < S_D \sim S$ \\
Tech & $0.05 \pm 0.02$ & $0.04 \pm 0.02$ & $0.04 \pm 0.01$ & $0.02 \pm 0.01$ & $0.02 \pm 0.01$ & $M_A < S_A < M < S_D \sim S$ \\
All & $0.05 \pm 0.02$ & $0.05 \pm 0.02$ & $0.04 \pm 0.01$ & $0.02 \pm 0.01$ & $0.02 \pm 0.01$ & $M_A < S_A < M < S_D < S$ \\
\bottomrule
\end{tabular}
}
\end{table*}

\subsubsection{Linguistic features}

The linguistic dispersion results present a more differentiated picture than the semantic analysis, with the quality--diversity tension appearing most clearly in this feature space (Tables~\ref{tab:linguistic_dispersion}, \ref{tab:linguistic_coverage}, and \ref{tab:linguistic_alignment}). Across nearly every linguistic metric and domain, $S_D$ emerges as the weakest setting, consistently performing worse than even the unconditioned single-model baseline S on measures such as n-gram diversity, TTR, and self-repetition. In particular, self-repetition values for $S_D$ range from $0.84$ to $1.00$, substantially higher than the human baseline of $0.06$ to $0.08$. This pattern suggests that domain-focused prompting encourages models to converge toward a narrower stylistic register composed of topically relevant yet highly repetitive language.

In contrast, $M_A$ consistently produces the closest approximation to human-level linguistic dispersion across most metrics. For example, its compression ratio and POS compression ratio values range from $2.37$ to $2.44$ and $4.03$ to $4.09$, respectively, compared to the corresponding human ranges of $2.12$ to $2.20$ and $3.81$ to $3.96$. The $S_A$ setting is generally close behind and is often statistically tied with $M_A$, indicating that aspect conditioning contributes substantially to linguistic diversity even without combining multiple models.

\begin{table*}[h]
\caption{Linguistic dispersion results across datasets. Arrows indicate metric directionality: higher is better for $\uparrow$ and lower is better for $\downarrow$.}
\label{tab:linguistic_dispersion}
\centering
\scriptsize
\setlength{\tabcolsep}{3pt}
    \resizebox{\columnwidth}{!}{
\begin{tabular}{llccccccp{3.5cm}}
\toprule
\textbf{Dataset} & \textbf{Metric} & \textbf{H} & \textbf{S} & \textbf{$S_D$} & \textbf{M} & \textbf{$S_A$} & \textbf{$M_A$} & \textbf{Comment} \\
\midrule
News & n-gram diversity $\uparrow$ & $3.35 \pm 0.15$ & $3.08 \pm 0.22$ & $2.89 \pm 0.25$ & $3.18 \pm 0.13$ & $3.26 \pm 0.22$ & $3.25 \pm 0.18$ & H $>$ $S_A \sim M_A$ $>$ M $>$ S $>$ $S_D$ \\
News & TTR $\uparrow$ & $0.47 \pm 0.09$ & $0.39 \pm 0.08$ & $0.34 \pm 0.07$ & $0.41 \pm 0.07$ & $0.46 \pm 0.10$ & $0.47 \pm 0.09$ & $M_A$ $>$ H $>$ $S_A$ $>$ M $>$ S $>$ $S_D$ \\
News & Self-repetition $\downarrow$ & $0.06 \pm 0.07$ & $0.52 \pm 0.34$ & $0.84 \pm 0.41$ & $0.28 \pm 0.16$ & $0.24 \pm 0.23$ & $0.26 \pm 0.23$ & H $<$ $S_A$ $<$ $M_A$ $<$ M $<$ S $<$ $S_D$ \\
News & Compression ratio $\downarrow$ & $2.12 \pm 0.17$ & $2.72 \pm 0.30$ & $2.96 \pm 0.36$ & $2.61 \pm 0.18$ & $2.40 \pm 0.29$ & $2.37 \pm 0.28$ & H $<$ $M_A$ $<$ $S_A$ $<$ M $<$ S $<$ $S_D$ \\
News & POS comp. ratio $\downarrow$ & $3.81 \pm 0.37$ & $4.31 \pm 0.43$ & $4.59 \pm 0.48$ & $4.16 \pm 0.34$ & $4.02 \pm 0.42$ & $4.03 \pm 0.37$ & H $<$ $S_A \sim M_A$ $<$ M $<$ S $<$ $S_D$ \\
\midrule
Pop\_culture & n-gram diversity $\uparrow$ & $3.34 \pm 0.14$ & $3.02 \pm 0.21$ & $2.81 \pm 0.25$ & $3.10 \pm 0.11$ & $3.23 \pm 0.20$ & $3.20 \pm 0.17$ & H $>$ $S_A$ $>$ $M_A$ $>$ M $>$ S $>$ $S_D$ \\
Pop\_culture & TTR $\uparrow$ & $0.48 \pm 0.09$ & $0.38 \pm 0.08$ & $0.32 \pm 0.07$ & $0.39 \pm 0.06$ & $0.45 \pm 0.09$ & $0.46 \pm 0.08$ & H $>$ $M_A$ $>$ $S_A$ $>$ M $>$ S $>$ $S_D$ \\
Pop\_culture & Self-repetition $\downarrow$ & $0.08 \pm 0.11$ & $0.52 \pm 0.35$ & $0.87 \pm 0.49$ & $0.34 \pm 0.18$ & $0.23 \pm 0.22$ & $0.28 \pm 0.19$ & H $<$ $S_A$ $<$ $M_A$ $<$ M $<$ S $<$ $S_D$ \\
Pop\_culture & Compression ratio $\downarrow$ & $2.10 \pm 0.17$ & $2.72 \pm 0.28$ & $2.98 \pm 0.36$ & $2.64 \pm 0.15$ & $2.39 \pm 0.26$ & $2.38 \pm 0.29$ & H $<$ $M_A \sim S_A$ $<$ M $<$ S $<$ $S_D$ \\
Pop\_culture & POS comp. ratio $\downarrow$ & $3.74 \pm 0.36$ & $4.34 \pm 0.45$ & $4.69 \pm 0.56$ & $4.23 \pm 0.32$ & $4.02 \pm 0.38$ & $4.05 \pm 0.37$ & H $<$ $S_A$ $<$ $M_A$ $<$ M $<$ S $<$ $S_D$ \\
\midrule
Tech & n-gram diversity $\uparrow$ & $3.26 \pm 0.12$ & $2.98 \pm 0.19$ & $2.77 \pm 0.26$ & $2.98 \pm 0.11$ & $3.15 \pm 0.20$ & $3.14 \pm 0.15$ & H $>$ $S_A \sim M_A$ $>$ M $\sim$ S $>$ $S_D$ \\
Tech & TTR $\uparrow$ & $0.41 \pm 0.07$ & $0.34 \pm 0.06$ & $0.30 \pm 0.06$ & $0.33 \pm 0.04$ & $0.41 \pm 0.08$ & $0.42 \pm 0.07$ & $M_A$ $>$ H $\sim$ $S_A$ $>$ S $>$ M $>$ $S_D$ \\
Tech & Self-repetition $\downarrow$ & $0.08 \pm 0.08$ & $0.62 \pm 0.37$ & $1.00 \pm 0.48$ & $0.50 \pm 0.20$ & $0.30 \pm 0.25$ & $0.32 \pm 0.18$ & H $<$ $S_A$ $<$ $M_A$ $<$ M $<$ S $<$ $S_D$ \\
Tech & Compression ratio $\downarrow$ & $2.20 \pm 0.13$ & $2.79 \pm 0.25$ & $3.05 \pm 0.39$ & $2.81 \pm 0.14$ & $2.48 \pm 0.25$ & $2.44 \pm 0.19$ & H $<$ $M_A$ $<$ $S_A$ $<$ S $<$ M $<$ $S_D$ \\
Tech & POS comp. ratio $\downarrow$ & $3.96 \pm 0.25$ & $4.48 \pm 0.33$ & $4.81 \pm 0.49$ & $4.48 \pm 0.21$ & $4.19 \pm 0.32$ & $4.17 \pm 0.27$ & H $<$ $M_A \sim S_A$ $<$ M $\sim$ S $<$ $S_D$ \\
\midrule
All & n-gram diversity $\uparrow$ & $3.33 \pm 0.14$ & $3.02 \pm 0.21$ & $2.82 \pm 0.26$ & $3.12 \pm 0.15$ & $3.22 \pm 0.20$ & $3.22 \pm 0.17$ & H $>$ $S_A \sim M_A$ $>$ M $>$ S $>$ $S_D$ \\
All & TTR $\uparrow$ & $0.46 \pm 0.09$ & $0.37 \pm 0.08$ & $0.32 \pm 0.07$ & $0.39 \pm 0.07$ & $0.44 \pm 0.09$ & $0.45 \pm 0.08$ & H $<$ $M_A$ $>$ $S_A$ $>$ M $>$ S $>$ $S_D$ \\
All & Self-repetition $\downarrow$ & $0.07 \pm 0.09$ & $0.55 \pm 0.36$ & $0.90 \pm 0.46$ & $0.35 \pm 0.20$ & $0.25 \pm 0.23$ & $0.26 \pm 0.20$ & H $<$ $S_A$ $<$ $M_A$ $<$ M $<$ S $<$ $S_D$ \\
All & Compression ratio $\downarrow$ & $2.13 \pm 0.17$ & $2.74 \pm 0.28$ & $3.00 \pm 0.37$ & $2.66 \pm 0.19$ & $2.43 \pm 0.26$ & $2.39 \pm 0.25$ & H $<$ $M_A$ $<$ $S_A$ $<$ M $<$ S $<$ $S_D$ \\
All & POS comp. ratio $\downarrow$ & $3.82 \pm 0.36$ & $4.38 \pm 0.41$ & $4.70 \pm 0.52$ & $4.25 \pm 0.34$ & $4.09 \pm 0.38$ & $4.09 \pm 0.34$ & H $<$ $M_A \sim S_A$ $<$ M $<$ S $<$ $S_D$ \\
\bottomrule
\end{tabular}
}
\end{table*}

The linguistic coverage and alignment results further support these observations. As shown in Table~\ref{tab:linguistic_coverage}, coverage scores are tightly clustered within the $0.91$ to $0.93$ range across all settings, suggesting that most generation strategies are able to capture the dominant lexical and POS-level structures present in human comments. The alignment metric in Table~\ref{tab:linguistic_alignment}, however, reveals a clearer separation between settings. Specifically, $M_A$ achieves the lowest divergence values ($0.02$ across all domains), followed closely by $S_A$ ($0.02$ to $0.03$), while S exhibits the highest divergence ($0.04$ to $0.05$). This indicates that aspect-conditioned generation more closely matches the underlying linguistic distribution of human comments.

\begin{table*}[h]
\caption{Linguistic coverage results across datasets. Higher values indicate better coverage of the human linguistic feature space.}
\label{tab:linguistic_coverage}
\centering
\small
\setlength{\tabcolsep}{5pt}
    \resizebox{\columnwidth}{!}{
\begin{tabular}{lcccccp{4.2cm}}
\toprule
\textbf{Dataset} & \textbf{S} & \textbf{$S_D$} & \textbf{M} & \textbf{$S_A$} & \textbf{$M_A$} & \textbf{Comment} \\
\midrule
News & $0.91 \pm 0.01$ & $0.91 \pm 0.01$ & $0.92 \pm 0.01$ & $0.92 \pm 0.01$ & $0.92 \pm 0.01$ & $M_A$ $>$ $S_A$ $>$ M $>$ $S_D$ $>$ S \\
Pop\_culture & $0.92 \pm 0.01$ & $0.92 \pm 0.01$ & $0.92 \pm 0.01$ & $0.93 \pm 0.01$ & $0.93 \pm 0.01$ & $M_A$ $>$ $S_A$ $>$ M $>$ $S_D$ $>$ S \\
Tech & $0.92 \pm 0.01$ & $0.92 \pm 0.01$ & $0.93 \pm 0.01$ & $0.93 \pm 0.01$ & $0.93 \pm 0.01$ & $M_A \sim$ M $>$ $S_A$ $>$ $S_D$ $>$ S \\
All & $0.92 \pm 0.01$ & $0.92 \pm 0.01$ & $0.92 \pm 0.01$ & $0.92 \pm 0.01$ & $0.93 \pm 0.01$ & $M_A$ $>$ $S_A$ $>$ M $>$ $S_D$ $>$ S \\
\bottomrule
\end{tabular}
}
\end{table*}

The TTR metric requires more careful interpretation because it is sensitive to comment length distributions. For example, in the \textit{tech} domain, $M_A$ slightly exceeds the human baseline ($0.42$ vs.\ $0.41$). However, this difference likely reflects variations in average comment length rather than genuinely greater vocabulary richness. Overall, the linguistic results reinforce the importance of aspect conditioning while also highlighting the limitations of domain-focused prompting as a mechanism for generating stylistically diverse comments.

\begin{table*}[h]
\caption{Linguistic alignment results across datasets. Lower values indicate better alignment with the human linguistic distribution.}
\label{tab:linguistic_alignment}
\centering
\small
\setlength{\tabcolsep}{5pt}
    \resizebox{\columnwidth}{!}{
\begin{tabular}{lcccccp{4.2cm}}
\toprule
\textbf{Dataset} & \textbf{S} & \textbf{$S_D$} & \textbf{M} & \textbf{$S_A$} & \textbf{$M_A$} & \textbf{Comment} \\
\midrule
News & $0.04 \pm 0.02$ & $0.04 \pm 0.02$ & $0.03 \pm 0.01$ & $0.03 \pm 0.02$ & $0.02 \pm 0.01$ & $M_A$ $<$ $S_A$ $<$ M $<$ $S_D$ $<$ S \\
Pop\_culture & $0.05 \pm 0.02$ & $0.04 \pm 0.02$ & $0.04 \pm 0.02$ & $0.03 \pm 0.02$ & $0.02 \pm 0.02$ & $M_A$ $<$ $S_A$ $<$ M $<$ $S_D$ $<$ S \\
Tech & $0.04 \pm 0.02$ & $0.03 \pm 0.02$ & $0.03 \pm 0.01$ & $0.02 \pm 0.01$ & $0.02 \pm 0.01$ & $M_A$ $<$ $S_A$ $<$ M $<$ $S_D \sim$ S \\
All & $0.04 \pm 0.02$ & $0.04 \pm 0.02$ & $0.03 \pm 0.01$ & $0.03 \pm 0.02$ & $0.02 \pm 0.01$ & $M_A$ $<$ $S_A$ $<$ M $<$ $S_D$ $<$ S \\
\bottomrule
\end{tabular}
}
\end{table*}

\subsubsection{Pragmatic features}

The pragmatic results reveal the most notable inversion across the entire evaluation framework (Tables~\ref{tab:pragmatic_dispersion}, \ref{tab:pragmatic_coverage}, and \ref{tab:pragmatic_alignment}). Unlike the semantic and linguistic feature spaces, the multi-model setting M achieves the highest pragmatic dispersion scores ($0.87$ across all domains), even slightly exceeding the human baseline of approximately $0.86$. This pattern suggests that combining outputs from multiple models without aspect conditioning naturally produces a broader range of pragmatic signals, including sentiment, stance, emotional tone, and discourse roles, than what is typically observed within a single human discussion thread.

In contrast, the aspect-conditioned settings $S_A$ and $M_A$ achieve pragmatic dispersion scores in the $0.85$ to $0.86$ range, closely tracking the behavior of S and $S_D$. This indicates that explicitly constraining generation through predefined aspects may slightly reduce the raw pragmatic variability of generated comments, even though aspect conditioning substantially improves diversity in the semantic and linguistic spaces.

\begin{table*}[h]
\caption{Pragmatic dispersion results across datasets. Higher values indicate greater diversity in pragmatic attributes.}
\label{tab:pragmatic_dispersion}
\centering
\small
\setlength{\tabcolsep}{5pt}
    \resizebox{\columnwidth}{!}{
\begin{tabular}{lccccccp{4.2cm}}
\toprule
\textbf{Dataset} & \textbf{H} & \textbf{S} & \textbf{$S_D$} & \textbf{M} & \textbf{$S_A$} & \textbf{$M_A$} & \textbf{Comment} \\
\midrule
News & $0.87 \pm 0.00$ & $0.85 \pm 0.01$ & $0.85 \pm 0.01$ & $0.87 \pm 0.00$ & $0.85 \pm 0.01$ & $0.86 \pm 0.01$ & M $>$ H $>$ $M_A$ $>$ $S_A$ $>$ S $>$ $S_D$ \\
Pop\_culture & $0.86 \pm 0.01$ & $0.85 \pm 0.01$ & $0.85 \pm 0.01$ & $0.86 \pm 0.00$ & $0.85 \pm 0.01$ & $0.85 \pm 0.01$ & M $>$ H $>$ $M_A$ $>$ $S_A$ $>$ S $>$ $S_D$ \\
Tech & $0.86 \pm 0.00$ & $0.80 \pm 0.03$ & $0.85 \pm 0.01$ & $0.87 \pm 0.00$ & $0.85 \pm 0.01$ & $0.85 \pm 0.01$ & M $>$ H $>$ $M_A$ $>$ $S_A$ $>$ $S_D$ $>$ S \\
All & $0.86 \pm 0.01$ & $0.83 \pm 0.03$ & $0.85 \pm 0.01$ & $0.87 \pm 0.00$ & $0.85 \pm 0.01$ & $0.85 \pm 0.01$ & M $>$ H $>$ $M_A$ $>$ $S_A$ $>$ $S_D$ $>$ S \\
\bottomrule
\end{tabular}%
}
\end{table*}

However, the pragmatic coverage and alignment results reveal a more nuanced pattern. As shown in Table~\ref{tab:pragmatic_coverage}, $M_A$ consistently achieves the highest pragmatic coverage across all domains, reaching values up to $0.90$ in the \textit{tech} and \textit{pop\_culture} domains. Interestingly, the unconstrained multi-model setting M performs nearly as strongly, substantially outperforming both S and $S_D$, and in the \textit{tech} domain achieving coverage identical to $M_A$ ($0.90$). In contrast, the aspect-conditioned single-model setting $S_A$ no longer consistently outperforms M, and in some domains exhibits noticeably lower coverage.

A similar trend appears in the pragmatic alignment results shown in Table~\ref{tab:pragmatic_alignment}. Across domains, both $M_A$ and M achieve the lowest divergence scores, typically ranging between $0.04$ and $0.06$, while S and $S_D$ consistently produce the weakest alignment. Unlike the semantic and linguistic spaces, aspect conditioning alone does not consistently improve pragmatic alignment over unconstrained multi-model generation. These findings suggest that combining diverse models naturally captures a broad and human-like range of pragmatic behaviors, whereas aspect conditioning contributes comparatively less additional benefit in the pragmatic feature space.

\begin{table*}[h]
\caption{Pragmatic coverage results across datasets. Higher values indicate broader coverage of the pragmatic characteristics observed in human discussions, including sentiment, stance, emotional tone, and discourse behavior.}
\label{tab:pragmatic_coverage}
\centering
\small
\setlength{\tabcolsep}{5pt}
\resizebox{\columnwidth}{!}{
\begin{tabular}{lcccccp{4.2cm}}
\toprule
\textbf{Dataset} & \textbf{S} & \textbf{$S_D$} & \textbf{M} & \textbf{$S_A$} & \textbf{$M_A$} & \textbf{Comment} \\
\midrule
News & $0.82 \pm 0.10$ & $0.83 \pm 0.10$ & $0.86 \pm 0.07$ & $0.83 \pm 0.09$ & $0.87 \pm 0.07$ & $M_A > M > S_A > S_D > S$ \\
Pop\_culture & $0.82 \pm 0.11$ & $0.82 \pm 0.11$ & $0.88 \pm 0.06$ & $0.85 \pm 0.08$ & $0.89 \pm 0.06$ & $M_A > M > S_A > S \sim S_D$ \\
Tech & $0.84 \pm 0.10$ & $0.85 \pm 0.09$ & $0.90 \pm 0.06$ & $0.82 \pm 0.09$ & $0.90 \pm 0.05$ & $M_A \sim M > S_D > S > S_A$ \\
All & $0.83 \pm 0.10$ & $0.83 \pm 0.10$ & $0.88 \pm 0.07$ & $0.83 \pm 0.09$ & $0.89 \pm 0.06$ & $M_A > M > S_A \sim S_D \sim S$ \\
\bottomrule
\end{tabular}%
}
\end{table*}

One interesting exception appears in the \textit{tech} domain, where the unconstrained multi-model setting M achieves pragmatic coverage comparable to $M_A$ ($0.90$ vs.\ $0.90$), while the aspect-conditioned setting $S_A$ drops to $0.82$. This suggests that, for highly technical discussions, unconstrained multi-model generation may naturally capture a broader range of pragmatic behaviors than aspect-conditioned prompting, potentially because technical conversations already exhibit narrower communicative patterns.

Overall, the pragmatic results suggest a clear trade-off between unconstrained pragmatic variability and human-aligned pragmatic structure. Multi-model generation without aspect conditioning consistently achieves strong pragmatic coverage and dispersion, whereas aspect-conditioned generation primarily improves alignment with the pragmatic profile of human discussions.

\begin{table*}[h]
\caption{Pragmatic alignment results across datasets. Lower values indicate closer alignment with the pragmatic distribution of human discussions, including sentiment balance, discourse roles, and interaction styles.}
\label{tab:pragmatic_alignment}
\centering
\small
\setlength{\tabcolsep}{5pt}
\resizebox{\columnwidth}{!}{
\begin{tabular}{lcccccp{4.2cm}}
\toprule
\textbf{Dataset} & \textbf{S} & \textbf{$S_D$} & \textbf{M} & \textbf{$S_A$} & \textbf{$M_A$} & \textbf{Comment} \\
\midrule
News & $0.09 \pm 0.04$ & $0.08 \pm 0.03$ & $0.07 \pm 0.03$ & $0.08 \pm 0.03$ & $0.06 \pm 0.03$ & $M_A < M < S_A \sim S_D < S$ \\
Pop\_culture & $0.07 \pm 0.03$ & $0.07 \pm 0.03$ & $0.05 \pm 0.02$ & $0.06 \pm 0.02$ & $0.05 \pm 0.02$ & $M_A \sim M < S_A < S_D \sim S$ \\
Tech & $0.06 \pm 0.02$ & $0.05 \pm 0.02$ & $0.04 \pm 0.01$ & $0.05 \pm 0.02$ & $0.04 \pm 0.01$ & $M_A \sim M < S_A \sim S_D < S$ \\
All & $0.07 \pm 0.03$ & $0.07 \pm 0.03$ & $0.05 \pm 0.03$ & $0.07 \pm 0.03$ & $0.05 \pm 0.02$ & $M_A \sim M < S_A \sim S_D \sim S$ \\
\bottomrule
\end{tabular}%
}
\end{table*}

\subsection{Domain-Specific results}
Examining results across the \textit{news}, \textit{pop\_culture}, and \textit{tech} domains reveals that the relative ordering of generation settings is remarkably stable across domains. As shown in Figure~\ref{fig:domain_semantic_dispersion}, the hierarchy $H > M_A > S_A > M > S_D > S$ for semantic dispersion is preserved consistently across all three domains, with similar stability observed for the corresponding semantic alignment orderings. This suggests that the primary findings of the paper are not artifacts of a specific topical context, but instead reflect broader properties of the underlying generation strategies.

At the same time, meaningful domain-level variation emerges in the magnitude of differences between settings rather than in their relative ordering. The semantic dispersion gap between human comments and $M_A$ is smallest in the \textit{pop\_culture} domain ($0.64$ vs.\ $0.63$), indicating that LLM-generated comments are most capable of approximating human semantic breadth in this setting. One likely explanation is that pop culture discussions naturally span a wide range of informal reactions, emotional expressions, humor, and stylistic variation that are particularly well captured through aspect conditioning. In contrast, the \textit{tech} domain consistently exhibits the lowest absolute semantic dispersion scores across all settings, which aligns with the intuition that technical discussions are inherently narrower in both vocabulary and topical scope.

\begin{figure}[h]
    \centering
    \includegraphics[width=\linewidth]{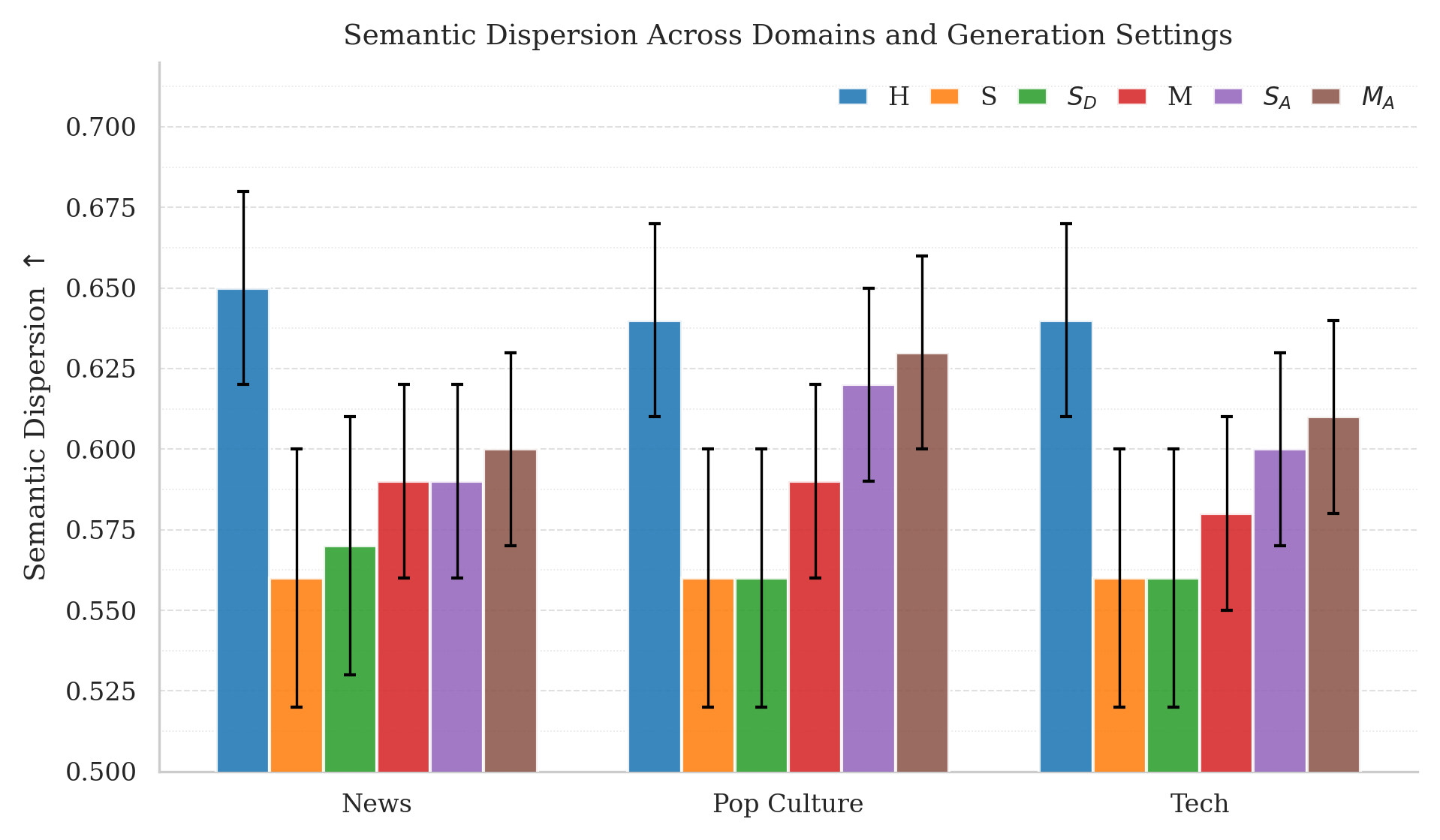}
    \caption{Semantic dispersion across domains and generation settings. The relative ordering of settings remains highly consistent across the \textit{news}, \textit{pop\_culture}, and \textit{tech} domains, with $M_A$ consistently producing the closest approximation to human semantic diversity. The smallest gap between human and synthetic comments is observed in the \textit{pop\_culture} domain.}
    \label{fig:domain_semantic_dispersion}
\end{figure}

Domain-specific effects are most pronounced in the pragmatic feature space. In the \textit{tech} domain, for example, $S_D$ unexpectedly outperforms M on pragmatic coverage ($0.85$ vs.\ $0.82$), reversing the more common ordering observed elsewhere. This suggests that domain-focused prompting may be especially effective for structured and information-dense conversations where pragmatic behaviors such as questioning, troubleshooting, and criticism are highly predictable from topical context. Similarly, self-repetition scores are noticeably higher across all synthetic settings in the \textit{tech} domain compared to \textit{news} and \textit{pop\_culture}. For instance, S reaches a self-repetition score of $0.62$ in \textit{tech} compared to $0.52$ in \textit{news}, indicating that technical content induces more formulaic generation patterns.

\begin{figure}[h]
    \centering
    \includegraphics[width=\linewidth]{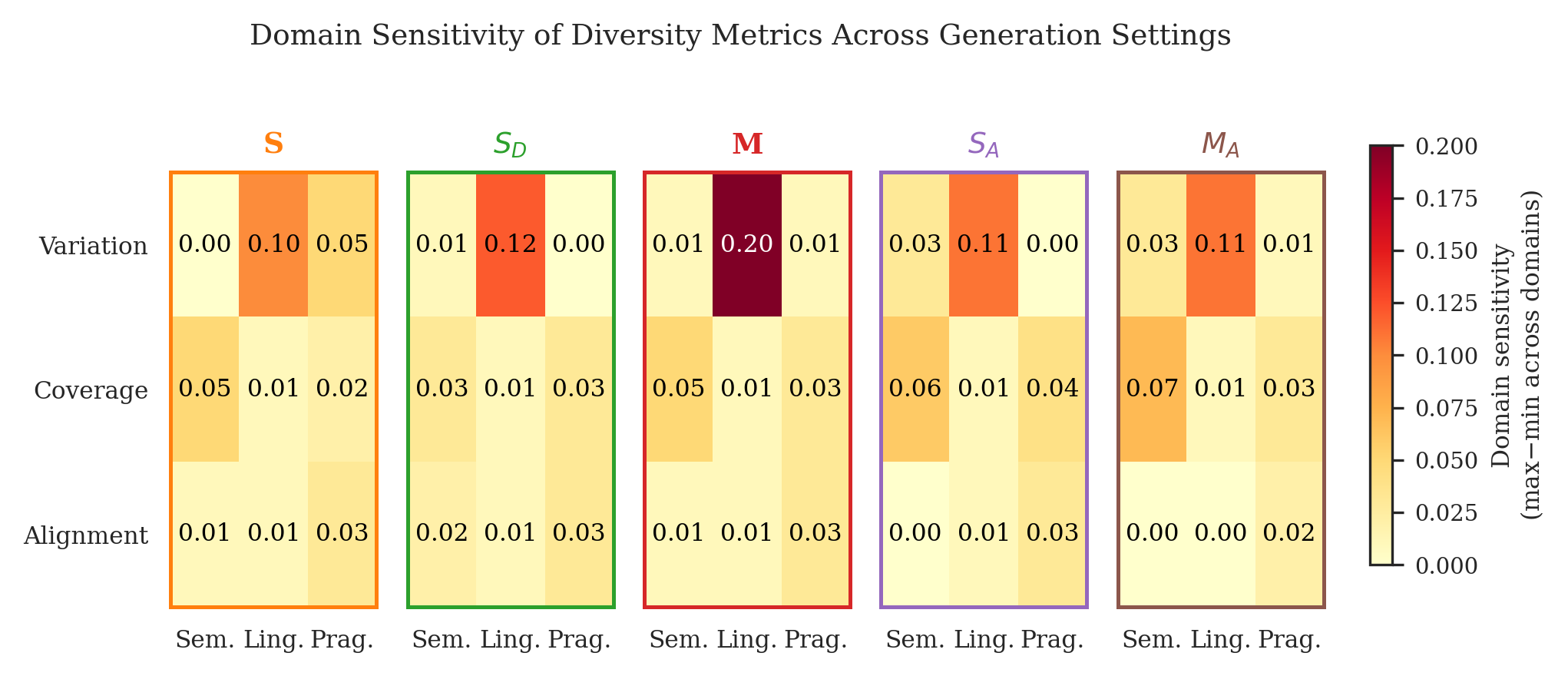}
    \caption{Domain-sensitivity heatmap across metrics and generation settings. Each cell represents the range (maximum - minimum) of a metric across domains for a given setting, where darker colors indicate greater sensitivity to domain variation. Pragmatic metrics exhibit the strongest domain sensitivity overall, particularly for the S and M settings, while aspect-conditioned settings ($S_A$ and $M_A$) remain substantially more stable across domains.}
    \label{fig:domain_sensitivity_heatmap}
\end{figure}

By contrast, the \textit{pop\_culture} domain demonstrates the largest semantic coverage gains from aspect conditioning, with both $S_A$ and $M_A$ achieving manifold recall scores of approximately $0.93$. This suggests that the highly heterogeneous nature of pop culture discussions particularly benefits from explicit aspect prescription. Overall, while domain identity influences the absolute performance levels and occasionally alters individual metric relationships, the dominant finding that aspect conditioning is the primary driver of diversity remains highly consistent across all domains.

Figure~\ref{fig:domain_sensitivity_heatmap} further summarizes these domain effects through a domain-sensitivity heatmap, where each cell represents the range (maximum minus minimum) of a metric across domains for a given setting. Darker cells correspond to higher domain sensitivity. Pragmatic alignment and pragmatic dispersion exhibit the highest sensitivity overall, particularly for the S and M settings. In contrast, the aspect-conditioned settings ($S_A$ and $M_A$) show consistently lighter patterns across most metrics, indicating that aspect conditioning improves not only diversity but also robustness across domains. Additionally, semantic coverage for $M_A$ appears notably more stable than for $S_A$, suggesting that the multi-model component contributes additional robustness on top of the gains introduced by aspect conditioning.

\subsection{Ablation analysis of decoding diversity $S_D$}

We further analyze the ablation setting $S_D$ to examine whether diversity can be improved by changing only the decoding strategy of a single model. This setting uses the same single-model design as S, but introduces diversity-oriented sampling parameters during generation. The goal is to separate the effect of stochastic decoding from the stronger structural interventions used in M, $S_A$, and $M_A$. The detailed metric-level results are reported in the preceding appendix tables.

Overall, $S_D$ provides only limited and inconsistent gains over the plain single-model setting S. In the semantic space, $S_D$ slightly improves semantic dispersion compared to S in the aggregate results, increasing from 0.56 to 0.57. It also improves semantic recall at threshold from 0.23 to 0.25. However, these gains remain small, and $S_D$ is still clearly behind M, $S_A$, and $M_A$ across semantic dispersion, coverage, and alignment. This suggests that decoding-level variation can introduce some local semantic variation, but it does not substantially expand the range of generated comments in the way model pluralism or aspect conditioning does. Semantic alignment shows a similar pattern: $S_D$ is close to S, while M, $S_A$, and especially $M_A$ consistently produce distributions closer to human comments.

The linguistic results make the limitation of $S_D$ even clearer. Across aggregate results, $S_D$ performs worse than S on several linguistic dispersion metrics. It has lower n-gram diversity, lower type-token ratio, higher self-repetition, and higher compression ratios. In other words, sampling variation may increase surface randomness, but it does not necessarily produce richer or more human-like lexical and syntactic variation. Instead, the outputs often remain repetitive or compressible, indicating that decoding alone cannot resolve the homogenization tendency of a single model.

The pragmatic results show a slightly more favorable but still limited picture. $S_D$ sometimes improves over S, particularly in pragmatic coverage and alignment, and it performs better than S on pragmatic dispersion in the tech domain. However, these improvements are small and do not approach the gains obtained by aspect-conditioned settings. Both $S_A$ and $M_A$ consistently provide stronger pragmatic coverage and alignment, suggesting that explicit aspect or style guidance is more effective than relying on sampling noise to induce communicative diversity.

Taken together, the $S_D$ ablation serves as a useful negative or weak baseline. It shows that simply making a single model more stochastic is not sufficient to recover human-like comment diversity. The main gains come from introducing structured sources of variation, either through multiple models or, more importantly, through explicit aspect conditioning. This supports our central design choice: diversity in synthetic comment generation is better achieved by modeling the structure of human discussion than by only adjusting decoding parameters.

\subsection{Ablation newer videos result analysis}

\begin{table*}[h]
\centering
\footnotesize
\setlength{\tabcolsep}{4.5pt}
\renewcommand{\arraystretch}{1}
\caption{
Summary of the newer-video ablation results across domains. 
We report representative metrics from the semantic and linguistic evaluation axes. 
For dispersion metrics, higher values are better; for alignment metrics, lower values are better. 
Human results are included where applicable.
}
\label{tab:newer_video_ablation_summary}
\begin{tabular}{llccccc}
\toprule
\textbf{Feature} & \textbf{Axis/Metric} & \textbf{Dataset} & \textbf{H} & \textbf{M} & \(\boldsymbol{S_A}\) & \(\boldsymbol{M_A}\)  \\
\midrule

\multirow{3}{*}{Semantic}
& \multirow{3}{*}{Dispersion \(\uparrow\)}
& News-new        & \(0.65 \pm 0.02\) & \(0.58 \pm 0.03\) & \(0.61 \pm 0.03\) & \(0.62 \pm 0.02\) \\
& 
& Pop-culture-new & \(0.64 \pm 0.03\) & \(0.60 \pm 0.03\) & \(0.63 \pm 0.03\) & \(0.64 \pm 0.02\) \\
& 
& Tech-new        & \(0.63 \pm 0.04\) & \(0.57 \pm 0.04\) & \(0.59 \pm 0.04\) & \(0.60 \pm 0.04\) \\

\midrule

\multirow{3}{*}{Semantic}
& \multirow{3}{*}{Manifold recall \(\uparrow\)}
& News-new        & -- & \(0.77 \pm 0.36\) & \(0.90 \pm 0.24\) & \(0.94 \pm 0.17\) \\
& 
& Pop-culture-new & -- & \(0.78 \pm 0.33\) & \(0.84 \pm 0.28\) & \(0.87 \pm 0.25\) \\
& 
& Tech-new        & -- & \(0.91 \pm 0.20\) & \(0.92 \pm 0.17\) & \(0.92 \pm 0.16\) \\

\midrule

\multirow{3}{*}{Semantic}
& \multirow{3}{*}{Alignment \(\downarrow\)}
& News-new        & -- & \(0.06 \pm 0.02\) & \(0.03 \pm 0.01\) & \(0.03 \pm 0.01\) \\
& 
& Pop-culture-new & -- & \(0.05 \pm 0.02\) & \(0.03 \pm 0.02\) & \(0.03 \pm 0.01\) \\
& 
& Tech-new        & -- & \(0.05 \pm 0.01\) & \(0.03 \pm 0.01\) & \(0.03 \pm 0.01\) \\

\midrule

\multirow{3}{*}{Linguistic}
& \multirow{3}{*}{N-gram diversity \(\uparrow\)}
& News-new        & \(3.26 \pm 0.13\) & \(3.06 \pm 0.12\) & \(3.22 \pm 0.15\) & \(3.32 \pm 0.09\) \\
& 
& Pop-culture-new & \(3.34 \pm 0.17\) & \(3.09 \pm 0.14\) & \(3.29 \pm 0.17\) & \(3.34 \pm 0.12\) \\
& 
& Tech-new        & \(3.32 \pm 0.16\) & \(3.08 \pm 0.13\) & \(3.25 \pm 0.13\) & \(3.30 \pm 0.10\) \\

\midrule

\multirow{3}{*}{Linguistic}
& \multirow{3}{*}{Coverage \(\uparrow\)}
& News-new        & -- & \(0.92 \pm 0.01\) & \(0.92 \pm 0.01\) & \(0.92 \pm 0.01\) \\
& 
& Pop-culture-new & -- & \(0.92 \pm 0.01\) & \(0.92 \pm 0.01\) & \(0.92 \pm 0.01\) \\
& 
& Tech-new        & -- & \(0.92 \pm 0.01\) & \(0.92 \pm 0.01\) & \(0.92 \pm 0.01\) \\

\midrule

\multirow{3}{*}{Linguistic}
& \multirow{3}{*}{Alignment \(\downarrow\)}
& News-new        & -- & \(0.02 \pm 0.01\) & \(0.03 \pm 0.01\) & \(0.02 \pm 0.01\) \\
& 
& Pop-culture-new & -- & \(0.03 \pm 0.01\) & \(0.03 \pm 0.02\) & \(0.03 \pm 0.01\) \\
& 
& Tech-new        & -- & \(0.03 \pm 0.01\) & \(0.03 \pm 0.01\) & \(0.02 \pm 0.01\) \\

\bottomrule
\end{tabular}
\end{table*}

We further evaluate our framework on newer videos where human-derived aspects are not available during generation. In this ablation, we use an LLM planner to infer possible discussion aspects directly from the video context. Since S consistently under performs in the original setting, we focus on M, $S_A$, and $M_A$. This setup allows us to test whether the main findings hold under a more realistic temporal shift, where the model must generate comments for newer content without relying on aspect information extracted from human comment threads.

The main takeaway from (Table \ref{tab:newer_video_ablation_summary}) is that the core ordering from the primary dataset is largely preserved. Across all three newer datasets, semantic dispersion follows the same hierarchy: H $>$ $M_A$ $>$ $S_A$ $>$ M. The strongest result appears in pop culture, where $M_A$ essentially reaches human semantic dispersion, with $0.64 \pm 0.02$ for $M_A$ compared to $0.64 \pm 0.03$ for H. This is an important robustness finding because it shows that the benefit of aspect conditioning and multi-model generation is not limited to the original historical benchmark. Instead, the same pattern transfers to newer videos where aspects are inferred rather than extracted from human comments.

Semantic coverage shows a similar trend. In news and pop culture, $M_A$ and $S_A$ remain clearly stronger than plain multi-model generation. In tech, the gap is smaller because M already achieves high manifold recall, suggesting that technical videos may contain more constrained discussion spaces where generic multi-model generation can already cover many dominant themes. Semantic alignment is also consistent with the main results: $M_A$ and $S_A$ remain closer to the human distribution than M across domains. The absolute alignment scores are slightly weaker in the newer-video ablation, moving from around $0.02$ in the primary dataset to around $0.03$ for the aspect-conditioned settings, but the relative ordering remains stable. Overall, these results suggest that aspect conditioning is the more stable contributor to semantic quality, while multi-model generation provides an additional but smaller gain.

The linguistic results are also supportive, although they require a more cautious interpretation. For n-gram diversity, $M_A$ performs strongly on newer videos, matching or exceeding human values in all three domains. In news, $M_A$ reaches $3.32$ compared to $3.26$ for H; in pop culture, both $M_A$ and H reach $3.34$; and in tech, $M_A$ reaches $3.30$ compared to $3.32$ for H. This strengthens the claim that the combined setting can approximate human-like surface-level variation even for newer content. However, linguistic coverage is nearly saturated across M, $S_A$, and $M_A$, with all methods around $0.92$. Therefore, this metric should not be overinterpreted as evidence of meaningful separation among the methods in the newer-video setting. Linguistic alignment also generally favors $M_A$, but the gaps are small, and in news M is tied with $M_A$.

Overall, this ablation should be interpreted as a robustness check rather than a separate new result. It confirms that the main conclusions are stable under temporal shift: aspect-conditioned generation, especially when combined with multiple models, better preserves semantic diversity, improves coverage, and remains closer to the human distribution. At the same time, the ablation shows that some linguistic metrics become less discriminative once all methods reach high coverage. This reinforces the value of our multi-axis evaluation framework, where semantic, linguistic, and alignment-based metrics provide complementary views of generation quality.

\subsection{Overall comment quality evaluation}
\label{appendix_subsec_overall_comment_quality}

We evaluate the overall quality of generated comment sets through human judgment, complemented with LLM-based evaluation for validation and scalability. Our goal is to assess whether generated comments not only capture diversity and structure but also resemble realistic online discussions in a holistic sense. To this end, we design the evaluation at the level of comment sets rather than individual comments, since properties such as diversity and representativeness emerge at the group level.

\paragraph{Human evaluation setup.}
We conduct a human evaluation with two annotators on 120 sampled items drawn from four settings: Human (H), $S$, $M$, and $M_A$. Each item corresponds to a single video paired with 10 comments. We choose 10 comments per item to balance two competing goals: preserving enough variation to capture discussion diversity, while keeping the set small enough for reliable human judgment. This choice is motivated by cognitive constraints in processing multiple information units \cite{miller1956magical}. The sampled items cover different domains and settings to ensure broad coverage.

Annotators are asked to provide holistic ratings on a 5-point Likert scale for the entire set of comments. The items are shuffled and presented without revealing their source to avoid bias. We consider the following evaluation rubrics:

\begin{enumerate}
\item \textbf{Diversity:} How diverse are the comments in terms of viewpoints and expressions \cite{khashabi-etal-2022-genie}.
\item \textbf{Context:} How well the comments reflect a shared context or conversation \cite{van2021human}.
\item \textbf{Naturalness:} How natural the comments read as human language \cite{van2021human}.
\item \textbf{Authenticity:} How authentic the comments feel as real online responses \cite{van2021human}.
\item \textbf{Representativeness:} How well the comments reflect typical online discourse \cite{han-etal-2022-measuring}.
\item \textbf{Appropriateness:} How appropriate the comments are with respect to the video \cite{van2021human}.
\item \textbf{Overall quality:} The overall perceived quality of the comment set \cite{khashabi-etal-2022-genie, van2021human}.
\end{enumerate}

\begin{table}[h]
\caption{Inter-annotator agreement across evaluation metrics using Krippendorff's $\alpha$ and weighted Cohen's $\kappa$. Agreement levels follow standard interpretation thresholds.}
\label{tab:iaa_metrics}
\centering
\small
\setlength{\tabcolsep}{6pt}
\begin{tabular}{lccc}
\toprule
\textbf{Metric} & \textbf{Krippendorff's $\alpha$} & \textbf{Weighted Cohen's $\kappa$} & \textbf{Agreement Level} \\
\midrule
\textbf{Naturalness}           & 0.601 & 0.591 & Moderate \\
\textbf{Authenticity}          & 0.463 & 0.471 & Moderate \\
\textbf{Overall Quality}       & 0.447 & 0.446 & Moderate \\
\textbf{Representativeness}    & 0.413 & 0.360 & Fair \\
\textbf{Diversity}             & 0.369 & 0.397 & Fair \\
\textbf{Appropriateness}       & 0.298 & 0.312 & Fair \\
\textbf{Context}               & 0.144 & 0.179 & Slight \\
\bottomrule
\end{tabular}

\end{table}

\paragraph{Agreement and validity.}
Inter-annotator agreement is measured using Krippendorff’s $\alpha$ and weighted Cohen’s $\kappa$. Agreement ranges from slight to moderate across metrics, with five of seven rubrics achieving moderate agreement. Naturalness, authenticity, and overall quality show the strongest agreement, indicating that annotators consistently recognize linguistic fluency and realism. Diversity and representativeness show fair agreement, which is expected given their more subjective nature.

Context and appropriateness exhibit lower agreement, with $\kappa$ around 0.2. This reflects the inherent ambiguity in judging how well comments align with a video, especially when annotators may interpret the video context differently or focus on different aspects of the discussion. Rather than indicating unreliability, this lower agreement highlights the nuanced and subjective nature of discourse-level evaluation, which aligns with observations in prior work on human judgments of language quality.

\begin{table}[h]
\caption{Human evaluation scores (mean $\pm$ std) across different settings. S: single LLM, M: multi-LLM, $M_A$: aspect-guided multi-LLM, H: human comments. $\sim$ means the difference between two settings is not statistically significant. }
\label{tab:human_eval_scores}
\centering
\small
\setlength{\tabcolsep}{5pt}
\begin{tabular}{lccccc}
\toprule
\textbf{Metric} & \textbf{S} & \textbf{M} & \textbf{$M_A$} & \textbf{H} & \textbf{Comment} \\
\midrule

\textbf{Diversity} 
& $3.47 \pm 0.52$ & $3.58 \pm 0.63$ & $3.98 \pm 0.70$ & \textbf{$4.12 \pm 0.72$} 
& H $\sim$ $M_A$ $>$ M $\sim$ S \\

\textbf{Context} 
& $3.90 \pm 0.52$ & $4.05 \pm 0.51$ & $4.03 \pm 0.47$ & \textbf{$4.05 \pm 0.56$} 
& H $\sim$ $M_A$ $\sim$ M $\sim$ S \\

\textbf{Naturalness} 
& $2.74 \pm 0.73$ & $3.12 \pm 0.57$ & $3.60 \pm 0.79$ & \textbf{$4.28 \pm 0.65$} 
& H $>$ $M_A$ $>$ M $>$ S \\

\textbf{Authenticity} 
& $3.07 \pm 0.56$ & $3.35 \pm 0.53$ & $3.60 \pm 0.71$ & \textbf{$3.93 \pm 0.58$} 
& H $\sim$ $M_A$ $\sim$ M $\sim$ S \\

\textbf{Representativeness} 
& $3.17 \pm 0.57$ & $3.20 \pm 0.41$ & $3.57 \pm 0.58$ & \textbf{$3.70 \pm 0.64$} 
& H $\sim$ $M_A$ $>$ M $\sim$ S \\

\textbf{Appropriateness} 
& $3.81 \pm 0.62$ & \textbf{$3.97 \pm 0.60$} & $3.93 \pm 0.64$ & $3.77 \pm 0.58$ 
& H $\sim$ $M_A$ $\sim$ M $\sim$ S \\

\textbf{Overall Quality} 
& $3.16 \pm 0.54$ & $3.43 \pm 0.43$ & $3.83 \pm 0.67$ & \textbf{$3.97 \pm 0.64$} 
& H $\sim$ $M_A$ $>$ M $>$ S \\

\bottomrule
\end{tabular}

\end{table}

\paragraph{Human evaluation results.}
The results reveal clear and consistent trends across settings. Human comments achieve the highest scores overall, with $M_A$ closely matching them on several dimensions. In terms of diversity, $M_A$ approaches human performance and significantly outperforms S and M, suggesting that aspect conditioning effectively captures varied discussion patterns. Context scores remain similar across all settings, indicating that even simpler models can maintain topical coherence when conditioned on the same video.
Naturalness shows the largest gap, with human comments clearly outperforming all generated settings. However, $M_A$ substantially improves over S and M, demonstrating that structured generation leads to more fluent outputs. Authenticity and appropriateness show relatively small differences across settings, suggesting that even baseline models produce plausible responses at a surface level.

Overall quality follows a similar pattern, where $M_A$ again approaches human performance and consistently outperforms S and M. Statistical testing using two-tailed t-tests confirms that many of these differences are significant, particularly for diversity, naturalness, and overall quality. These findings indicate that combining multiple models with aspect conditioning leads to meaningful improvements in perceived comment quality.

\begin{table}[h]
\caption{Comparison between human and LLM judgments across evaluation metrics. Spearman correlation ($\rho$) is reported at both the setting level and sample level.}
\label{tab:human_llm_alignment}
\centering
\small
\setlength{\tabcolsep}{5pt}
\begin{tabular}{lccccc}
\toprule
\textbf{Metric} & \textbf{Human Judgment} & \textbf{LLM Judgment} & \textbf{Setting-level $\rho$} & \textbf{Sample-level $\rho$} \\
\midrule

\textbf{Diversity} 
& H $\sim$ $M_A$ $>$ M $\sim$ S 
& H $\sim$ $M_A$ $>$ M $\sim$ S 
& 1.00 & 0.16 \\

\textbf{Context} 
& H $\sim$ M $\sim$ $M_A$ $\sim$ S 
& S $\sim$ M $>$ $M_A$ $\sim$ H 
& -0.63 & 0.04 \\

\textbf{Naturalness} 
& H $>$ $M_A$ $>$ M $>$ S 
& H $\sim$ $M_A$ $\sim$ M $\sim$ S 
& 1.00 & 0.07 \\

\textbf{Authenticity} 
& H $\sim$ $M_A$ $\sim$ M $\sim$ S 
& H $>$ $M_A$ $\sim$ M $>$ S 
& 1.00 & 0.21 \\

\textbf{Representativeness} 
& H $\sim$ $M_A$ $>$ M $\sim$ S 
& H $>$ $M_A$ $>$ M $\sim$ S 
& 1.00 & 0.43 \\

\textbf{Appropriateness} 
& M $\sim$ $M_A$ $\sim$ S $\sim$ H 
& S $\sim$ M $>$ $M_A$ $\sim$ H 
& 0.63 & 0.05 \\

\textbf{Overall Quality} 
& H $\sim$ $M_A$ $>$ M $>$ S 
& H $\sim$ $M_A$ $\sim$ M $>$ S 
& 1.00 & 0.20 \\

\bottomrule
\end{tabular}

\end{table}

\paragraph{LLM-based validation and large-scale evaluation.}
To validate the human annotations and extend evaluation to the full dataset, we employ an LLM-as-a-judge approach using three models from different providers (GPT-5.2, Gemini-3, and Claude-4.5-Sonnet). We first evaluate the same sampled items used in human annotation and compare LLM judgments with human scores. While sample-level correlations are weak, setting-level correlations are consistently high across most metrics. This means that although LLMs may differ in absolute scoring, they preserve the relative ranking between settings. This provides confidence that LLM-based evaluation can be used to study large-scale trends. We use the prompt in Box \ref{prompt_llm_as_judge} with temperature = 0 for reproducibility.

\begin{tcolorbox}[
    colback=blue!10,
    colframe=blue!40!black,
    title=Prompt for LLM as judge comment evaluation,
    label=prompt_llm_as_judge,
    rounded corners
]\footnotesize

\textbf{Task:} You are an expert evaluator of online comment sections.

Your task is to assess the overall quality of a set of comments posted in response to a video. Base your judgment \textbf{only} on the provided video information and comments.

\vspace{5pt}
\textbf{Important Instructions:}
\begin{itemize}
    \item Use a Likert scale from \textbf{1 to 5} for each category:
    \begin{itemize}
        \item 1 = very low \quad 3 = moderate \quad 5 = very high
    \end{itemize}
    \item Do \textbf{not} provide explanations, only numeric scores.
    \item Do \textbf{not} add any extra text, headings, or commentary.
    \item Follow the output format exactly.
\end{itemize}

\vspace{5pt}
\textbf{Video Information:} \\
\texttt{\{video\_information\}}

\vspace{5pt}
\textbf{Comments:} \\
\texttt{\{comments\}}

\vspace{5pt}
\textbf{Evaluation Questions:}

Evaluate the comments as a whole based on the following criteria:

\begin{enumerate}
    \item \textbf{Diversity:} To what extent do the comments reflect a variety of viewpoints or stances toward the video?
    
    \item \textbf{Context:} How much do the comments feel like they are responding to the same shared context or conversation about the video?
    
    \item \textbf{Naturalness:} How natural do the comments feel in terms of language and tone?
    
    \item \textbf{Authenticity:} How authentic do these comments feel as responses people would actually post online?
    
    \item \textbf{Representativeness:} To what extent does this comment section feel representative of how people typically respond online to similar videos?
    
    \item \textbf{Appropriateness:} How appropriate are the comments with respect to the video content?
    
    \item \textbf{Overall Quality:} Overall, how would you rate the quality of this comment section?
\end{enumerate}

\vspace{5pt}
\textbf{Output Format (Strict):}
\begin{itemize}
    \item diversity: \texttt{<score>}
    \item context: \texttt{<score>}
    \item naturalness: \texttt{<score>}
    \item authenticity: \texttt{<score>}
    \item representativeness: \texttt{<score>}
    \item appropriateness: \texttt{<score>}
    \item overall\_quality: \texttt{<score>}
\end{itemize}

\end{tcolorbox}

We then apply LLM-based evaluation to the full dataset. The aggregated results (Figure \ref{fig_llm_judgement_score}) follow the same patterns observed in human evaluation, with $M_A$ consistently performing close to human comments and outperforming S and M across most metrics. This consistency across human and automated evaluation supports the validity of our assumptions and demonstrates that our framework captures meaningful aspects of comment quality at scale.

Overall, this evaluation provides a comprehensive assessment of comment quality. It combines human judgment for depth and reliability with LLM-based evaluation for scalability, and shows that structured multi-model generation can produce comments that are both diverse and realistic.

\begin{figure}[h]
    \centering
    \includegraphics[width=\linewidth]{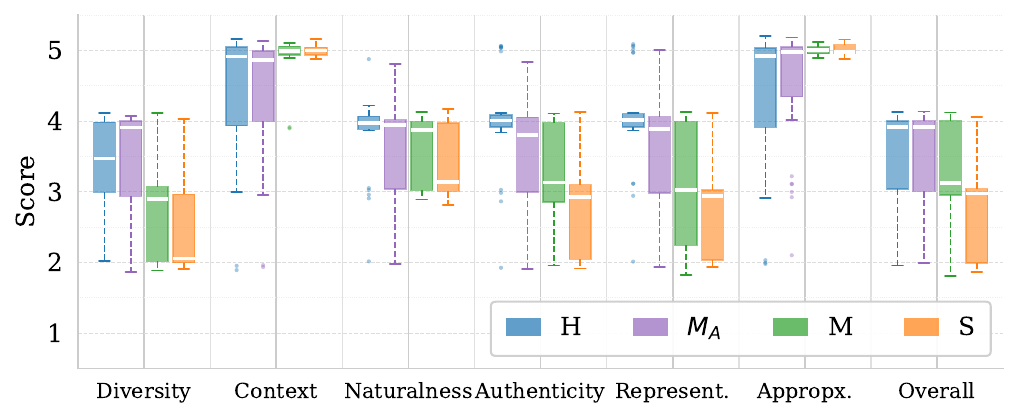}
    \caption{Overall comment quality results on whole dataset using LLM-as-judge.}
    \label{fig_llm_judgement_score}
\end{figure}

\section{RQ2 detailed findings}

Prior works have shown that high-quality synthetic datasets \citep{benallal2024cosmopedia, finephrase} can improve model performance. Despite being noisier, social media data is expanding rapidly and increasingly contributing to the internet data ecosystem \citep{longpre2025bridging, villalobos2024position}, making its quality and impact crucial.
Thus, \textbf{RQ2} evaluates whether the LLM-generated comments are useful as downstream data, beyond matching human comment diversity. Since full-scale pre-training is infeasible at our scale \citep{hoffmann2022training, li2024datacomp}, we use standard data curation and selection pipelines as a proxy to assess whether LLM-generated comments would be retained as training data relative to human counterparts. In addition, we evaluate their effectiveness in fine-tuning and instruction-tuning settings, where such data is more likely to be directly utilized in practice.

\subsection{Pre-training data utility}
\label{appendix_pretraining_data}

\paragraph{Representative document construction.}
For pre-training-style quality evaluation, individual YouTube comments are often too short to resemble standard pre-training documents. To address this, we convert comment threads into representative document-like samples while preserving the RQ1 source metadata. The resulting document table retains video, domain, setting, provider or model, and comment identifiers, allowing each constructed document to be traced back to its origin.

For human comments, we group comments by \texttt{Video\_id} and sort them chronologically using \texttt{Comment\_posting\_date}. We then construct documents based on comment length. Short comments with fewer than 200 words are greedily accumulated into a buffer until the combined length approaches 200 words. Medium-length comments between 200 and 400 words are kept as standalone documents, which aligns with the typical 256--512 token range used in document pre-selection pipelines \citep{su2025nemotron}. Long comments above 400 words are truncated to their first 400 words to avoid outliers. At the end of each video, any remaining short-comment buffer is merged into the nearest previous document; if no such document exists, the buffer itself is retained as a document. This procedure preserves the temporal flow of discussions while producing text units that are suitable for document-level evaluation.

For LLM-generated comments, we adopt a length-aware packing strategy that reflects how generated outputs are structured across models and settings. Comments are grouped by \texttt{Video\_id} and \texttt{model\_id}, and for aspect-conditioned settings also by \texttt{cluster\_id}. Within each group, comments are sorted by length in descending order. Long comments above 400 words are truncated, medium-length comments are kept as standalone documents, and shorter comments are greedily combined. For non-aspect settings, we use a target length of 200 words for packed documents to match the human construction. For aspect-conditioned settings, we use a smaller target of 40 words, since comments are already organized into finer-grained clusters and tend to be shorter and more homogeneous. 

Overall, this construction strategy ensures that both human and generated texts are transformed into comparable document units. It balances fidelity to the original comment structure with the need for sufficiently long inputs for downstream quality evaluation.

\paragraph{Pre-Training Data Quality Filtering.}
After constructing representative documents, we apply a staged filtering pipeline that approximates common pre-training data selection procedures. The filters are ordered to prioritize inexpensive lexical and metadata-based checks before more computationally intensive semantic operations. This design allows us to scale efficiently while still capturing multiple dimensions of data quality.

\begin{enumerate}

\item \textbf{Exact substring deduplication.} 
We first perform substring-based deduplication to identify near-duplicate documents. For each document, we extract hashed fixed-length token substrings using 40-token windows and compute overlap using Jaccard similarity. A document is flagged if its overlap with a previously retained document exceeds 0.5, with at least one shared substring. In practice, this step does not remove documents in our setting, as the generation pipeline already introduces sufficient diversity. We retain it to remain consistent with standard pre-training pipelines.

\item \textbf{Low-readability filtering.}
We compute a readability profile using standard metrics. We use Flesch reading ease \citep{kincaid1975derivation} as the primary filter and flag documents that fall in the bottom 1 percentile of the active document set following \citep{gohari2025gneissweb}. This step removes poorly formed or unintelligible text while preserving the natural variability of online comments.

\item \textbf{Tokenization anomaly filtering.}
We identify documents with unusual tokenization patterns by computing statistics such as character length, byte length, token count, tokens per character, and tokens per byte. The main filter uses tokens per character and flags documents that lie more than three standard deviations away from the mean of the active document set following \citep{gohari2025gneissweb}. This helps remove corrupted or irregular text that may arise from encoding issues or malformed content.

\item \textbf{Perplexity-based filtering.}
We compute document perplexity using a small language model, \texttt{GPT-2} in our implementation, and flag documents whose perplexity deviates by more than two standard deviations from the mean following the mid-selection criteria in \citep{anknerperplexed}. This step captures text that is either highly predictable or highly irregular, both of which are undesirable for pre-training data.

\item \textbf{Semantic deduplication.}
Finally, we perform semantic deduplication  \citep{abbas2023semdedup} using embedding-based similarity. Documents are embedded and compared using nearest-neighbor search, and a document is flagged if it is a later occurrence of another document with cosine similarity greater than 0.95. This step removes redundant content that may not be captured by lexical overlap alone.

\end{enumerate}

At each stage, we track retention relative to the initial document set across \texttt{setting} and \texttt{domain}. This analysis is important for RQ2, as it reveals whether human and different LLM-generated data sources are affected differently by quality filtering.

\paragraph{Educational value and document quality evaluation.}
We evaluate document quality using complementary pretrained models that capture different aspects of pre-training utility. Rather than treating these scores as strict filters, we use them as descriptive signals to compare human and LLM-generated documents. This allows us to assess not only whether generated data resembles human text, but also whether it meets the quality standards expected in pre-training pipelines.

\begin{itemize}

\item \textbf{General document quality.}
We use NVIDIA's \texttt{quality-classifier-deberta} \citep{su2025nemotron} to assess overall document quality. The model assigns each document to a category such as low, medium, or high quality and provides a probability distribution over these categories. In our document table, we store both the predicted label (\texttt{doc\_quality}) and the associated probability vector (\texttt{doc\_quality\_prob}). This provides a coarse but informative signal of how well a document aligns with general quality expectations.

\item \textbf{Educational value.}
We measure educational usefulness using NVIDIA's \texttt{nemocurator-fineweb-nemotron-4-edu-classifier}, which produces a continuous score on an approximate 0 to 5 scale. We retain the raw score (\texttt{educational\_score}) and also map it to a binary label for interpretability. Documents with scores of at least 2.5 are labeled as high quality, while the rest are labeled as low quality. This threshold provides a simple way to distinguish content that may be useful for learning or knowledge transfer.

\end{itemize}

Together, these two signals capture complementary views of document quality. The general quality classifier reflects overall readability and coherence, while the educational score captures potential usefulness for downstream learning tasks. We compare their distributions across human comments and each LLM generation setting, and we also analyze how these distributions change after the staged filtering pipeline. This allows us to examine whether generated documents not only resemble human comments, but also exhibit properties that would make them suitable for inclusion in pre-training corpora.

As shown in Figure \ref{fig_rq2_pretraining_results}, our results reveal a key tension between \textbf{diversity and conventional notions of data quality} in social media contexts. While multi-LLM and aspect-conditioned generation improve diversity and lead to higher retention under standard pre-training curation pipelines, this does not directly translate to higher document quality as defined by existing filters. In fact, these filters tend to favor structured, knowledge-rich text, whereas authentic social media discourse, human or synthetic, often appears less “high-quality” under such criteria. This highlights 
a fundamental mismatch between current data curation standards and the nature of social media data, suggesting 
the need for \textbf{revisiting quality definitions when leveraging diverse, discourse-driven data} for future language model training.

\subsection{Fine-tuning data utility}
\label{appendix_finetuning_data}

\paragraph{Data preparation.}
To evaluate whether generated comments are useful as supervised fine-tuning data, we formulate RQ2 as a downstream multi-class emotion classification task based on Ekman’s six emotions plus neutral \citep{ekman2014expression}. This task is well aligned with the expressive nature of social media comments. We use the labels provided by \citep{hartmann2022emotionenglish} as ground truth, which are also used in our socio-pragmatic feature extraction pipeline.

We first construct human train, validation, and test splits using a 70/15/15 ratio, resulting in 37,941 training, 8,131 validation, and 8,131 test samples. The label distribution is naturally imbalanced, with neutral (51.91\%), joy (18.20\%), and surprise (15.29\%) dominating, followed by sadness (4.55\%), anger (3.98\%), disgust (3.80\%), and fear (2.26\%). 

For each LLM setting (S, M, $S_A$, $M_A$), we construct aligned datasets by sampling one generated comment for each human comment. The sampled comment matches the same \texttt{Video\_id}, the same label when label-aware sampling is enabled, and a similar length within a fixed tolerance when applicable. Each generated comment is used at most once to avoid duplication. After sampling, we balance the combined human and LLM data by label to ensure comparable class distributions across settings.

We also consider an augmentation setting (H\_aug), where we specifically target under-represented classes. In this setup, we replace a portion of over-represented classes (neutral, joy, surprise) with generated examples so that each under-represented class (sadness, disgust, anger, fear) accounts for at most 10\% of the total dataset. The overall dataset size remains fixed. After augmentation, the label distribution becomes more balanced, with each minority class close to 10\% and reduced dominance of neutral and joy.

In addition, we perform a controlled ablation by constructing mixed training and validation sets. For each class label, we retain a fraction of human examples and replace the remainder with generated examples from a given setting. We vary the human fraction over $\{10, 25, 50\}$ while preserving the original human class distribution. The held-out human test set remains fixed across all experiments to ensure consistent evaluation on real-world data.

\paragraph{Model training.}
We fine-tune a \texttt{bert-base-uncased} model for multi-class text classification \citep{sun2019fine}. Training is performed on an A100 40GB GPU using standard hyperparameters: maximum sequence length of 512 tokens, training batch size of 64, evaluation batch size of 128, learning rate of $2\times10^{-5}$, weight decay of 0.01, and mixed-precision training enabled. The model is trained for 3 epochs using categorical cross-entropy loss. These settings follow common practice and ensure that performance differences are primarily driven by the data rather than model or optimization choices.

\begin{table}[h]
\centering
\caption{Full per-class classification report for 7-class emotion fine-tuning across all training settings. P = precision, R = recall, F1 = per-class F1. Mac.F1 = macro F1. All models evaluated on the held-out human test split.}
\label{tab:emotion_full_report}
\scriptsize
\setlength{\tabcolsep}{3pt}
    \resizebox{\columnwidth}{!}{
\begin{tabular}{l ccc ccc ccc ccc ccc ccc ccc c}
\toprule
\textbf{Setting} & \multicolumn{3}{c}{\textbf{Anger}} & \multicolumn{3}{c}{\textbf{Disgust}} & \multicolumn{3}{c}{\textbf{Fear}} & \multicolumn{3}{c}{\textbf{Joy}} & \multicolumn{3}{c}{\textbf{Neutral}} & \multicolumn{3}{c}{\textbf{Sadness}} & \multicolumn{3}{c}{\textbf{Surprise}} & \textbf{Mac.F1} \\
\cmidrule(lr){2-4} \cmidrule(lr){5-7} \cmidrule(lr){8-10} \cmidrule(lr){11-13} \cmidrule(lr){14-16} \cmidrule(lr){17-19} \cmidrule(lr){20-22}
& P & R & F1 & P & R & F1 & P & R & F1 & P & R & F1 & P & R & F1 & P & R & F1 & P & R & F1 & \\
\midrule
H & 0.55 & 0.56 & 0.55 & 0.57 & 0.53 & 0.55 & 0.73 & 0.34 & 0.46 & 0.77 & 0.80 & 0.79 & 0.84 & 0.86 & 0.85 & 0.57 & 0.56 & 0.56 & 0.72 & 0.71 & 0.71 & 0.64 \\
S & 0.44 & 0.56 & 0.49 & 0.45 & 0.60 & 0.52 & 0.68 & 0.46 & 0.55 & 0.79 & 0.73 & 0.76 & 0.82 & 0.83 & 0.82 & 0.62 & 0.57 & 0.59 & 0.68 & 0.64 & 0.66 & 0.63 \\
M & 0.42 & 0.64 & 0.51 & 0.51 & 0.43 & 0.47 & 0.73 & 0.39 & 0.51 & 0.78 & 0.74 & 0.76 & 0.82 & 0.84 & 0.83 & 0.61 & 0.54 & 0.58 & 0.71 & 0.67 & 0.69 & 0.62 \\
$S_A$ & 0.53 & 0.55 & 0.54 & 0.54 & 0.52 & 0.53 & 0.75 & 0.41 & 0.54 & 0.74 & 0.81 & 0.77 & 0.84 & 0.83 & 0.84 & 0.63 & 0.53 & 0.57 & 0.68 & 0.71 & 0.70 & 0.64 \\
$M_A$ & 0.57 & 0.56 & 0.57 & 0.56 & 0.50 & 0.53 & 0.81 & 0.43 & 0.56 & 0.75 & 0.77 & 0.76 & 0.82 & 0.85 & 0.83 & 0.62 & 0.49 & 0.54 & 0.68 & 0.71 & 0.69 & 0.64 \\
$H_\mathrm{aug}$ & 0.48 & 0.65 & 0.55 & 0.51 & 0.61 & 0.55 & 0.57 & 0.65 & 0.61 & 0.77 & 0.77 & 0.77 & 0.84 & 0.83 & 0.83 & 0.56 & 0.67 & 0.61 & 0.76 & 0.62 & 0.69 & 0.66 \\
\bottomrule
\end{tabular}
}
\end{table}

\paragraph{Results and analysis.}
We summarize the main findings in Section 4.2 of the paper. Table~\ref{tab:emotion_full_report} provides detailed classification reports, and Figure~\ref{fig_finetuning_heatmap} visualizes  F1 score heatmap across settings. The results reveal consistent patterns across classes and settings.

Fear is a structurally difficult class across all configurations. Precision is relatively high, ranging from 0.57 to 0.81, but recall remains low, typically between 0.34 and 0.46 for both human-only and purely synthetic training. This indicates that models learn a conservative decision boundary for fear and fail to capture a large portion of true instances. The H\_aug setting substantially improves recall to 0.65, although precision drops to 0.57. This reflects a trade-off enabled by exposing the model to a broader range of fear-related expressions.

Disgust follows a similar but less pronounced pattern, with recall consistently trailing precision by 7 to 13 points across most settings. In contrast, joy and neutral achieve strong and balanced performance, with both precision and recall above 0.73. This reflects their higher frequency and clearer linguistic signals in the training data.

\begin{figure}
    \centering
    \includegraphics[width=0.5\linewidth]{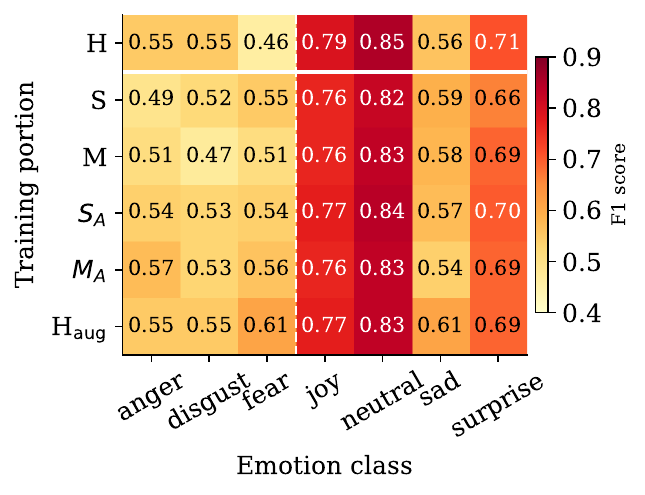}
    \caption{Heatmap of per-class F1 score while using different training portion in emotion classification task (to show utility as finetuning data).}
    \label{fig_finetuning_heatmap}
\end{figure}

Across synthetic settings, aspect-conditioned approaches ($S_A$, $M_A$) improve performance for several classes compared to non-conditioned settings (S, M). For example, in the anger class, $S_A$ maintains similar precision while improving recall, and $M_A$ achieves the highest F1 score among purely synthetic settings at 0.57, matching human-only training. This suggests that aspect conditioning increases the diversity of affective expressions and provides better coverage of class-specific patterns.

The sadness class shows a different behavior. The H\_aug setting achieves the highest recall at 0.67 but with reduced precision at 0.56, indicating that while generated examples improve coverage, they may also introduce some ambiguity that leads to false positives. Human-only training produces more balanced precision and recall for this class.

Overall, this evaluation provides a comprehensive view of fine-tuning utility. It shows that generated comments can serve as effective training data, especially when combined with aspect conditioning or targeted augmentation. At the same time, it highlights limitations in capturing rare or nuanced emotional expressions and the importance of carefully balancing diversity and label fidelity.

\subsection{Instruction-tuning data utility}
\label{appendix_instruction_data}

\paragraph{Data preparation.}
To evaluate whether generated comments are useful as instruction-style data, we design a political stance classification task. The goal is to classify a comment, given the corresponding video context, into one of three labels: \textit{conservative}, \textit{liberal}, or \textit{neutral}. This task requires contextual reasoning beyond standard text classification, since the stance of a comment often depends on the framing and subject of the video. We focus on the news domain, where stance is more likely to be expressed and polarized. In particular, we collect comments from three major news channels, Fox News, MSNBC, and CNBC, which are known to exhibit differing political orientations \citep{hosseinmardi2025unpacking, mays2020social, martin2014bias}.

We consider comments with a minimum length of 10 tokens from both human data and the $M_A$ setting, as the latter shows strong performance in both diversity and utility. To obtain labels, we use an LLM-as-judge approach with three annotators from different providers: GPT-5.2, Gemini-3, and Claude-4.5-Sonnet. We use the prompt in Box \ref{prompt_political_stance} for this task with temperature=0 for deterministic behavior.
Also, the final label is determined by majority voting following the standard practices in LLM-as-judge literature \citep{tan2024large, li2023synthetic}. To ensure meaningful evaluation, we select only videos that contain sufficient examples from each class. This results in a dataset of 2,003 comments from 21 videos. We split this into 878 test samples and use the remaining comments from both human and LLM sources as candidate few-shot examples.

\begin{tcolorbox}[
    colback=blue!10,
    colframe=blue!40!black,
    title=\textit{Prompt for Political Stance Detection},
    label=prompt_political_stance,
    rounded corners
]\footnotesize

\textbf{Task:} You are a political content analyst.

Your task is to classify the political orientation expressed in a single
online comment, based \textbf{only} on the provided video information and comment text.

\vspace{4pt}
\textbf{Important Rules:}
\begin{itemize}
    \item Choose exactly \textbf{one} label from the allowed list.
    \item Do \textbf{not} explain your reasoning.
    \item Do \textbf{not} add any extra text.
    \item If the comment does not clearly express a political stance, label it as \texttt{neutral}.
    \item Base the judgment on tone, framing, ideology, and implied alignment.
\end{itemize}

\vspace{4pt}
\textbf{Allowed Labels:}
\begin{itemize}
    \item \texttt{liberal}
    \item \texttt{conservative}
    \item \texttt{neutral}
\end{itemize}

\vspace{4pt}
\textbf{Video Information:} \\
\texttt{\{video\_information\}}

\vspace{4pt}
\textbf{Comment:} \\
\texttt{\{comment\_text\}}

\vspace{4pt}
\textbf{Output Format (Strict):} \\
\texttt{label: <liberal | conservative | neutral>}

\end{tcolorbox}

\paragraph{Task formulation and prompt design.}
We evaluate the utility of instruction under both zero-shot and few-shot prompting. In the zero-shot setting, the prompt includes task instructions, video information, the target comment, and the allowed label set. The model is instructed to output exactly one label without explanation. This setting tests whether the model can infer stance using only contextual understanding. We use the same prompt in Box \ref{prompt_political_stance} for zero-shot setting.

In the few-shot setting, we augment the prompt with available labeled examples from the same video. For each target comment, we sample up to two examples per class from comments with the same \texttt{Video\_id}. This ensures that the examples are contextually aligned with the target. The prompt, therefore, includes video information, a small set of labeled comments, and the target comment (see the prompt in Box \ref{prompt_for_fewshot_political_stance}). Using the same video examples is important because stance interpretation is often context-dependent. 

\begin{tcolorbox}[
    colback=blue!10,
    colframe=blue!40!black,
    title=\textit{Prompt for Few-Shot Political Stance Detection},
    label=prompt_for_fewshot_political_stance,
    rounded corners
]\footnotesize

\textbf{Task:} You are a political content analyst.

Your task is to classify the political orientation expressed in an online
comment, based \textbf{only} on the provided video information and examples given.

\vspace{4pt}
\textbf{Important Rules:}
\begin{itemize}
    \item Choose exactly \textbf{one} label from the allowed list.
    \item Do \textbf{not} explain your reasoning.
    \item Do \textbf{not} add any extra text.
    \item If the comment does not clearly express a political stance, label it as \texttt{neutral}.
    \item  Base the judgment on tone, framing, ideology, and implied alignment.
\end{itemize}

\vspace{4pt}
\textbf{Allowed Labels:}
\begin{itemize}
    \item \texttt{liberal}
    \item \texttt{conservative}
    \item \texttt{neutral}
\end{itemize}

\vspace{4pt}
\textbf{Video Information:} \\
\texttt{\{video\_information\}}

\vspace{4pt}
\textbf{Labeled Examples:} \\
\texttt{\{examples\_text\}}

\vspace{4pt}
\textbf{Comment to Classify:} \\
\texttt{\{comment\_text\}}

\vspace{4pt}
\textbf{Output Format (Strict):} \\
\texttt{label: <conservative | neutral | liberal>}

\end{tcolorbox}

We evaluate instruction-tuned models of varying sizes, including \texttt{Qwen2.5-0.5B-Instruct}, \texttt{Mistral-7B-Instruct-v0.2}, and \texttt{Gemma-3-12B-it}. These models are obtained from the official Hugging Face repositories and loaded in 8-bit precision to enable efficient inference. All experiments are run on an A100 40GB GPU.  We use deterministic decoding (temperature=0, $top\_p$=1.0,
        $top\_k=0$) with greedy inference, limiting the output to a small number of tokens (5) to enforce label-only responses. This setup ensures that performance differences reflect the usefulness of the provided examples rather than sampling variability.

\begin{figure}[h]
    \centering
    \includegraphics[width=\linewidth]{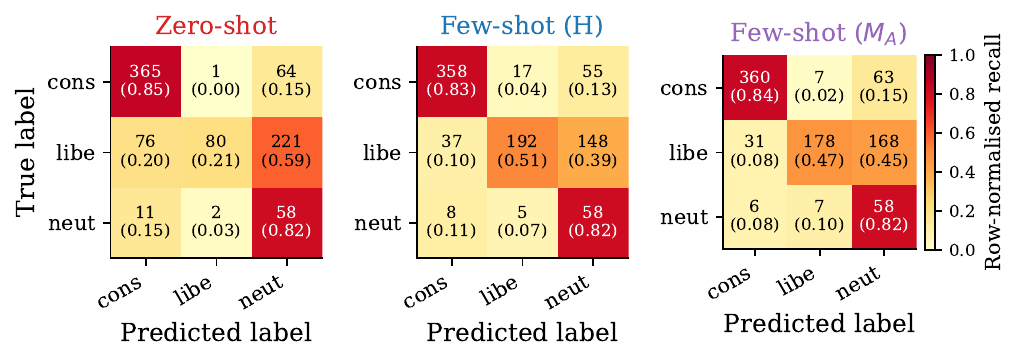}
    \caption{Confusion matrix in different settings using \texttt{Mistral-7B-Instruct-v0.2} model.  Providing few-shot examples improve the performance of separating neutral class from liberal class and we observe the similar effect for both using human comments and LLM comments as examples.}
    \label{fig_mistral_instruction_confusion_matrix}
\end{figure}

\paragraph{Results and analysis.}
We present the main findings in the paper and provide detailed results in Table~\ref{tab:stance_full_report} and Figure~\ref{fig_mistral_instruction_confusion_matrix}. The results show a consistent asymmetry across models and settings. Conservative and liberal stances are detected with relatively high recall, while the neutral class exhibits lower precision. This indicates that models often default to neutral predictions when uncertain.

\begin{table}[h]
\centering
\caption{Full per-class classification report for stance detection under few-shot prompting with human (H) and LLM ($M_A$) examples across three backbone models. Support counts: conservative 430, liberal 377, neutral 71 (total 878).}
\label{tab:stance_full_report}
\small
\setlength{\tabcolsep}{4pt}
\begin{tabular}{ll ccc ccc ccc cc}
\toprule
& & \multicolumn{3}{c}{\textbf{Conservative}} & \multicolumn{3}{c}{\textbf{Liberal}} & \multicolumn{3}{c}{\textbf{Neutral}} & \multicolumn{2}{c}{\textbf{Overall}} \\
\cmidrule(lr){3-5}\cmidrule(lr){6-8}\cmidrule(lr){9-11}\cmidrule(lr){12-13}
\textbf{Model} & \textbf{Condition} & P & R & F1 & P & R & F1 & P & R & F1 & Acc & Mac.F1 \\
\midrule
  \multirow{2}{*}{Qwen-0.5B} & Few-shot (H) & 0.48 & 0.36 & 0.41 & 0.41 & 0.59 & 0.49 & 0.20 & 0.06 & 0.09 & 0.43 & 0.33 \\
   & Few-shot ($M_A$) & 0.51 & 0.19 & 0.28 & 0.42 & 0.72 & 0.53 & 0.22 & 0.23 & 0.22 & 0.42 & 0.34 \\
\midrule
  \multirow{2}{*}{Mistral-7B} & Few-shot (H) & 0.89 & 0.83 & 0.86 & 0.90 & 0.51 & 0.65 & 0.22 & 0.82 & 0.35 & 0.69 & 0.62 \\
   & Few-shot ($M_A$) & 0.91 & 0.84 & 0.87 & 0.93 & 0.47 & 0.63 & 0.20 & 0.82 & 0.32 & 0.68 & 0.61 \\
\midrule
  \multirow{2}{*}{Gemma-12B} & Few-shot (H) & 0.74 & 0.95 & 0.83 & 0.96 & 0.54 & 0.69 & 0.41 & 0.66 & 0.51 & 0.75 & 0.68 \\
   & Few-shot ($M_A$) & 0.70 & 0.95 & 0.80 & 0.97 & 0.44 & 0.61 & 0.42 & 0.73 & 0.53 & 0.71 & 0.65 \\
\bottomrule
\end{tabular}
\end{table}

For \texttt{Gemma-3-12B-it}, conservative recall reaches 0.95 and liberal recall 0.54 when using human examples, while neutral recall improves to 0.66 at the cost of lower precision. \texttt{Mistral-7B-Instruct} shows a similar pattern, with neutral recall as high as 0.82 but precision dropping to around 0.20, suggesting a strong bias toward neutral predictions under uncertainty. \texttt{Qwen2.5-0.5B-Instruct} performs significantly worse, with macro F1 near 0.33 to 0.34 and near-random confusion between conservative and liberal classes, highlighting the limitations of smaller models in leveraging in-context examples.

Comparing human and LLM-generated few-shot examples, the differences in macro F1 are relatively small for larger models, but the class-wise behavior reveals important differences. Human examples consistently improve liberal recall, while LLM-generated examples tend to increase neutral recall at the expense of liberal recall. This suggests that synthetic examples, while competitive in overall performance, encode a slightly different distribution of stance signals that leads to more conservative predictions.

Overall, this evaluation provides a complementary view of data utility. It shows that generated comments can serve as effective instruction-style examples, particularly for medium and large models, while also highlighting subtle distributional differences that influence model behavior in context-dependent tasks.



\end{document}